\documentclass[a4paper,11pt]{article}
\usepackage[margin=1in]{geometry}
\usepackage[utf8]{inputenc}
\usepackage{dsfont}
\usepackage{natbib}
\usepackage{adjustbox}

\usepackage{microtype}
\usepackage{graphicx}
\usepackage{subcaption}
\usepackage{booktabs} %
\usepackage{pgfplots}
\pgfplotsset{compat=1.18}
\usepackage{subcaption}
\usepackage{tikz}
\usepackage{pgfplots}
\pgfplotsset{compat=1.18}
\usepackage{subcaption}
\usepackage{enumitem}

\pgfplotsset{table/search path={Experiments Data/Verify_bound}}

\usepackage{hyperref}

\usepackage{mathtools}

\usepackage[utf8]{inputenc} %
\usepackage[T1]{fontenc}    %
\usepackage{hyperref}       %
\usepackage{url,soul}            %
\usepackage{booktabs}       %
\usepackage{amsfonts}       %
\usepackage{nicefrac}       %
\usepackage{microtype}      %
\usepackage{xcolor}         %
\usepackage{amsmath, amssymb, amsthm}
\usepackage{algorithm}
\usepackage{algorithmic}
\usepackage{graphicx}
\usepackage{dsfont}
\usepackage{bm}
\usepackage{csquotes}
\usepackage{multirow}

\newcommand{\Ltrain}{\mathcal{L}_\text{T}}
\newcommand{\Lval}{\mathcal{L}_\text{V}}

\usepackage[capitalize,noabbrev]{cleveref}

\theoremstyle{plain}
\newtheorem{theorem}{Theorem}[section]
\newtheorem{proposition}[theorem]{Proposition}
\newtheorem{lemma}[theorem]{Lemma}

\newtheorem{fact}[theorem]{Fact}
\theoremstyle{definition}

\newtheorem{assumption}[theorem]{Assumption}
\theoremstyle{remark}
\newtheorem{remark}[theorem]{Remark}

\title{{\bf Less is More: Convergence Benefits of Fewer Data Weight Updates over Longer Horizon}}
\author{
  Rudrajit Das\textsuperscript{1}, Neel Patel\textsuperscript{2}, Meisam Razaviyayn\textsuperscript{1,3}, and Vahab Mirrokni\textsuperscript{1} \\
  \textsuperscript{1}Google Research, \textsuperscript{2}EPFL, 
  \textsuperscript{3}University of Southern California \\
  \texttt{\{dasrudrajit, razaviyayn, mirrokni\}@google.com, neel.patel@epfl.ch}
}

\date{}

\begin{document}

\maketitle 

\vspace{0.3cm}

\begin{abstract}
\noindent Data mixing—the strategic reweighting of training domains—is a critical component in training robust machine learning models. This problem is naturally formulated as a bilevel optimization task, where the outer loop optimizes domain weights to minimize validation loss, and the inner loop optimizes model parameters to minimize the weighted training loss. Classical bilevel optimization relies on hypergradients, which theoretically require the inner optimization to reach convergence. However, due to computational constraints, state-of-the-art methods use a finite, often small, number of inner update steps before updating the weights. The theoretical implications of this approximation are not well understood. In this work, we rigorously analyze the convergence behavior of data mixing with a finite number of inner steps $T$. We prove that the ``greedy'' practical approach of using $T=1$ can fail even in a simple quadratic example. Under a fixed parameter update budget $N$ and assuming the per-domain losses are strongly convex, we show that the optimal $T$ scales as $\Theta(\log N)$ (resp., $\Theta(\sqrt{N \log N})$) for the data mixing problem with access to full (resp., stochastic) gradients. We complement our theoretical results with proof-of-concept experiments. 
\end{abstract}

\section{Introduction}
The performance of modern machine learning models depends heavily on the composition of their (pre-)training data \cite{sambasivan2021everyone, whang2023data}. In particular, as modern foundation models are trained on mixtures of diverse domains (ranging from high-quality code and textbooks to noisy web scrapes) determining the optimal contribution of each domain, known as \textit{data mixing}, has become a crucial challenge \cite{brown2020language, touvron2023llama, raffel2020exploring, gao2020pile}. Historically, data mixtures were determined via manual heuristics. For instance, GPT-3 and PaLM utilized manual upsampling of \enquote{high-quality} sources \cite{chowdhery2023palm,brown2020language}. More recently, this has been formalized into data mixing laws \cite{ye2024data,liu2025regmix}.

The data mixing problem is naturally cast as a \ul{bilevel optimization} problem \cite{sinha2017review}. The outer loop seeks to find a weighting vector $\bm{w}$ for the training domains that minimizes a validation loss $\mathcal{L}_\text{V}$ (representing the downstream %
performance measured using a high-quality reference set), while the inner loop minimizes the weighted training loss $\mathcal{L}_\text{T}$ with respect to the model parameters $\bm{\theta}$ 
\cite{franceschi2018bilevel, fan2023doge}.

While mathematically elegant, the application of bilevel optimization to data mixing for LLMs faces a severe computational bottleneck: calculating the exact gradient with respect to $\bm{w}$---the \textit{hypergradient}---requires differentiating through the entire training trajectory of $\bm{\theta}$ (which evolves as a function of $\bm{w}$) {until convergence}, but this is intractable for large models. To make this feasible, %
several online mixing methods such as DoGE \cite{fan2023doge} and PIKE \cite{li2025pike}, as well as earlier meta-learning approaches \citep{ren2018learning, shu2019meta}, adopt a ``greedy" heuristic. They perform a single step or very few steps of model parameter updates before updating the data weights, %
prioritizing computational efficiency over hypergradient accuracy. We refer to the number of parameter updates per weight update as the \textit{lookahead horizon}, and denote it by $T$.

This reliance on greedy approximations ignites a critical theoretical controversy centered on \textit{short-horizon bias} \citep{wu2018understanding}. Seminal works in meta-learning suggest that truncating the inner optimization loop to a negligible horizon introduces systematic errors, often leading the meta-learner to adopt myopic strategies—such as artificially suppressing learning rates—to minimize immediate volatility rather than maximizing long-term generalization \citep{shaban2019truncated}. Despite the empirical adoption of greedy mixing methods, the theoretical implications of this bias in the simplex-constrained domain of data mixing remain underexplored. Does the prevalent $T=\Theta(1)$ heuristic %
find the optimal mixture, or does it succumb to truncation bias?

In this work, we rigorously analyze the convergence behavior of data mixing %
as a function of the lookahead horizon $T$. We challenge the prevailing ``greedy'' heuristic, demonstrating that the strategy of frequent weight updates with short horizons is systematically biased. We establish a ``\textit{less is more}'' principle: under a fixed total budget of parameter updates, it is provably better to update data weights \textit{less frequently} but with a \textit{longer lookahead horizon}.%
\\
\\
Our main \textbf{contributions} are:
\\
\noindent  \ul{``Greedy'' Approach Fails}: We theoretically show that the greedy approach of using $T=1$ in data mixing is not merely noisy but systematically biased. In a simple quadratic example, it can cause data weights to converge to sub-optimal values, failing to identify the correct mixture (\Cref{lin:reg-example}). 
\\
\\
\noindent\ul{Optimal Horizon Scaling}: We derive the optimal scaling laws for the lookahead horizon $T$ (i.e., number of parameter updates per weight update) under a fixed parameter update budget $N = K T$, where $K$ is the total number of weight updates. Specifically, in the strongly convex case, we prove that the optimal $T$ scales as $\Theta(\log N)$ and $\Theta(\sqrt{N \log N})$ in the deterministic and stochastic settings, respectively, contrasting sharply with the $\Theta(1)$ horizon commonly used in practice (\Cref{convg-results}). {We corroborate our theoretical findings with proof-of-concept experiments which show that moderate values of $T$ (sublinear in $N$) lead to the best performance in practice (\Cref{sec:experiment}).}

\section{Related Work}
\label{sec:related_work}

\textbf{Static Heuristics and Scaling Laws.}
Historically, data mixtures were determined via manual heuristics, such as upsampling high-quality domains such as Wikipedia \citep{chowdhery2023palm}. Recent work has formalized this into data mixing laws \citep{ye2024data}, which fit power-law regression models on small-scale training runs to predict the loss of larger models. Methods like UniMax \citep{chung2024unimax} propose maximizing diversity under epoching constraints. While principled, these methods are static; they freeze the mixture weights prior to training and cannot adapt to the model's evolving curriculum needs \citep{bengio2009curriculum}.
\\
\\
\textbf{Offline Proxy-Based Optimization.}
To introduce data-driven adaptability without the cost of online updates, several methods utilize proxy models \citep{ren2018learning, shu2019meta}. DoReMi \citep{xie2023doremi} formulates data mixing as a distributionally robust optimization (DRO) problem \citep{sagawa2020distributionally, kuhn2025distributionally}, training a small proxy model to minimize worst-case excess loss relative to a reference model. RegMix \citep{liu2024regmix} treats mixture selection as a regression problem, fitting high-degree polynomials to map mixture weights to validation loss via extensive random sampling. While effective, these methods incur significant pre-computation overhead and suffer from a gap between proxy and target model dynamics.
\\
\\
\textbf{Online Bilevel Optimization.} The most adaptive paradigm integrates weight optimization directly into the training loop. Early ``learning to weight'' approaches used auxiliary networks to predict sample weights \citep{jiang2018mentornet, shu2019meta}, but these struggle to scale to LLMs. Current state-of-the-art methods focus on efficient domain-level reweighting. DoGE \citep{fan2023doge} estimates domain weights by computing a ``generalization gain'' via the dot product of validation and training gradients. PIKE \citep{li2025pike} extends this to multi-task learning by minimizing gradient conflicts. TIKMIX \citep{wang2025tikmix} and TANDEM \citep{wang2025tandem} represent the newest wave of dynamic mixing. Crucially, to maintain efficiency, these methods %
rely on the ``greedy'' approach ($T=1$), updating weights based on immediate gradient feedback.
\\
\\
\textbf{Theoretical Foundations: Truncation Bias vs. Unrolling.}
The theoretical validity of truncating the inner optimization loop is a subject of intense scrutiny. \citet{wu2018understanding} identified that short-horizon unrolling introduces systematic bias, while \citet{shaban2019truncated} proved that exact convergence requires the horizon to scale logarithmically with precision. Conversely, single-loop methods \citep{ji2021bilevel} argue convergence is possible under strict time-scale separation, though this %
condition is rarely met in standard pre-training.
\\
\\
Alternative approaches include Implicit Differentiation \citep{lorraine2020optimizing, blondel2022efficient}, which approximates the inverse Hessian via Neumann series \citep{grazzi2020iteration}. However, recent analysis on the ``curse of unrolling'' \citep{scieur2022curse} suggests that unrolled differentiation can suffer from numerical instability over long horizons. Our work unifies these perspectives for data mixing, providing a rigorous analysis of the ``sweet spot'' for $T$ that minimizes both optimization error and truncation bias.

\section{Notation and Preliminaries}
We denote vectors and matrices by bold font symbols. 
For a natural number $n \in \mathbb{N}$, we denote the set $\{1,\ldots,n\}$ by $[n]$. For a vector $\bm{v}$, we denote its $i^\text{th}$ coordinate by $v^{(i)}$ and its $\ell_p$ norm by $\|\bm{v}\|_p$. A vector $\bm{v} = [v^{(1)}, \ldots, v^{(n)}]$ is said to belong to the $n$-dimensional simplex $\bm{\Delta}^n$ if $\sum_{i=1}^n v^{(i)} = 1$ and $v^{(i)} \geq 0$ $\forall$ $i \in [n]$. For a matrix $\bm{M}$, we denote its operator norm by $\|\bm{M}\|_\text{op}$. A function $f(\bm{x}): \mathbb{R}^d \xrightarrow{} \mathbb{R}$ is said to be $G$-Lipschitz if $\sup_{\bm{x} \in \mathbb{R}^d}\|\nabla f(\bm{x})\|_2 \leq G$, $L$-smooth if $\|\nabla f(\bm{y}) - \nabla f(\bm{x})\|_2 \leq L \|\bm{y} - \bm{x}\|_2$ for all $\bm{x}, \bm{y} \in \mathbb{R}^d$, and $\mu$-strongly-convex if $f(\bm{y}) \geq f(\bm{x}) + \langle \nabla f(\bm{x}), \bm{y} - \bm{x} \rangle + \frac{\mu}{2}\|\bm{y} - \bm{x}\|_2^2$ for all $\bm{x}, \bm{y} \in \mathbb{R}^d$.

\section{Problem Formulation}
\label{sec:prob-form}
In pre-training large foundation models, we are typically given access to $m$ diverse training domains (e.g., Wikipedia, GitHub, CommonCrawl), each associated with a loss function $\ell_i$ ($i \in [m]$). The goal is not simply to minimize the average loss on these domains, but to minimize the loss on a %
\textbf{validation set} (or ``reference'' distribution) $\mathcal{L}_\text{V}$, which serves as a proxy for downstream capability. To that end, a weighted loss over the training domains can be minimized, where the weights should be chosen so as to minimize the validation loss as much as possible.

This creates a hierarchical dependency: we want to find the mixture weights $\bm{w} \in \bm{\Delta}^m$ that produces a model $\bm{\theta}^*(\bm{w})$ which performs the best on $\mathcal{L}_V$. This is the canonical \textbf{bilevel optimization} formulation \citep{franceschi2018bilevel, grazzi2020iteration}:
\begin{equation}
\label{eq:def}
\begin{split}
    \min_{\bm{w} \in \bm{\Delta}^m}\;  & F(\bm{w}) := \mathcal{L}_\text{V}(\bm{\theta}^*(\bm{w})) \\
     \text{s.t.} \quad &\bm{\theta}^*(\bm{w}) \in \arg\min_{\bm{\theta}} \mathcal{L}_\text{T}(\bm{\theta}, \bm{w}) := \sum_{i=1}^m w^{(i)} \ell_i(\bm{\theta}).
\end{split}
\end{equation}
{Note that for $F(\bm{w})$ to be well-defined, we need to uniquely specify $\bm{\theta}^*(\bm{w})$ as one particular optimum of $\mathcal{L}_\text{T}(\bm{\theta}, \bm{w})$ if it has multiple optima. For example, this could be %
the optimum closest to the initialization. %
}

\noindent \textbf{The hypergradient challenge.} To optimize $\bm{w}$ via first-order methods, we need the \textit{hypergradient} $\nabla F(\bm{w})$. Using the Implicit Function Theorem (IFT), the $j^\text{th}$ ($j \in [m]$) coordinate of this is given by:
\begin{equation}
    \label{eq:ift_hypergradient}
    \nabla F(\bm{w})^{(j)} = - \Big \langle \nabla_{\bm{\theta}} \mathcal{L}_\text{V} \big(\bm{\theta}^{\ast}\big),  \nabla_{\bm{\theta}}^2 \mathcal{L}_\text{T} \big(\bm{\theta}^{\ast}, \bm{w}\big)^{-1} \nabla \ell_j\big(\bm{\theta}^{\ast}\big) \Big \rangle,
\end{equation}
where for brevity $\bm{\theta}^{\ast}$ denotes $\bm{\theta}^{\ast}(\bm{w})$.\footnote{See \Cref{lem2} for a quick proof of this result.}
Unfortunately, computing the exact quantity in Eq.~\ref{eq:ift_hypergradient} is intractable for two reasons: (1) for every single update of $\bm{w}$, it requires training $\bm{\theta}$ to full convergence ($\bm{\theta}^*$), and (2) it requires inverting the massive Hessian matrix at $\bm{\theta}^*$. 
So we focus on a more practical \textit{unrolled algorithm}, where the inner optimization w.r.t $\bm{\theta}$ is approximated by a few steps, say $T$, of gradient descent, which is then used to approximate the hypergradient.\footnote{Note that $T \to \infty$ would lead to convergence to $\bm{\theta}^{\ast}$.} We formally describe this {unrolled algorithm} in \Cref{alg:1-main}. It can be viewed as a coupled system where each ``round'' consists of $T$ inner parameter updates followed by one outer weight update. In \Cref{alg:1-main}, $\bm{\theta}_{k,t}$ denotes the parameter at step $t$ of round $k$, while $\bm{w}_k$ denotes the weight vector in round $k$. The inner parameter update rule
\begin{equation}
    \label{eq:inner_update}
    \bm{\theta}_{k, t+1} = \bm{\theta}_{k, t} - \eta \nabla_\theta \mathcal{L}_\text{T}(\bm{\theta}_{k, t}, \bm{w}_k),
\end{equation}
is standard gradient descent with constant step-size $\eta$. After $T$ parameter update steps, we compute the gradient of the validation loss at $\bm{\theta}_{k,T}$ with respect to the current weights $\bm{w}_k$ (i.e., $\frac{\partial \mathcal{L}_\text{V}(\bm{\theta}_{k,T})}{\partial \bm{w}_k}$) to update them. Specifically, for updating the weights, we perform mirror-descent \cite{bubeck2015convex} with the negative entropy function as the Bregman divergence and constant step-size $\alpha$. 

\begin{algorithm}[t]
	\caption{Data Mixing}
	\label{alg:1-main}
	\begin{algorithmic}
	    \STATE {\bfseries Input:} Initialization for model parameters $\bm{\theta}_0$, initialization for mixture weights $\bm{w}_0$, parameter update step-size $\eta$, weight update step-size $\alpha$, number of rounds $K$, number of parameter updates per round / \enquote{horizon} $T$. 
        \vspace{0.1 cm}
        \STATE Set $\bm{\theta}_{0, 0} = \bm{\theta}_0$.
        \vspace{0.1 cm}
        \FOR{$k \in \{0,\ldots,K-1\}$}
        \vspace{0.1 cm}
        \FOR{$t \in \{0,\ldots,T-1\}$}
        \vspace{0.1 cm}
        \STATE Update $\bm{\theta}_{k, t+1} = \bm{\theta}_{k, t} - \eta \nabla_\theta \mathcal{L}_\text{T}(\bm{\theta}_{k, t}, \bm{w}_k)$. 
        \\
        // \texttt{Inner parameter update}
        \vspace{0.1 cm}
        \ENDFOR
        \vspace{0.1 cm}
        \STATE For $j \in [m]$, update $w_{k+1}^{(j)} = \frac{w_{k}^{(j)} \exp\big(-\alpha \bar{g}^{(j)}_{k, T}\big)}{\sum_{s=1}^m w_{k}^{(s)} \exp\big(-\alpha \bar{g}^{(s)}_{k, T}\big)}$, with  $\bar{g}^{(j)}_{k, T} := \frac{\partial \mathcal{L}_\text{V}(\bm{\theta}_{k,T})}{\partial {w}^{(j)}_k}$. 
        \\
        // \texttt{Outer weight update}
        \vspace{0.1 cm}
        \STATE Set $\bm{\theta}_{k+1, 0} = \bm{\theta}_{k, T}$.
        \vspace{0.1 cm}
        \ENDFOR 
        \vspace{0.1 cm}
		\STATE \textbf{Return} $\bm{\theta}_{K, 0}$ and $\bm{w}_K$.
	\end{algorithmic}
\end{algorithm}

Note that if $T \to \infty$, then we will get the exact hypergradient in \cref{eq:ift_hypergradient}. Unfortunately, as $T$ increases, keeping track of $\frac{\partial \mathcal{L}_\text{V}(\bm{\theta}_{k,T})}{\partial \bm{w}_k}$ exactly becomes challenging. To see this, note that the $j^\text{th}$ coordinate of $\frac{\partial \mathcal{L}_\text{V}(\bm{\theta}_{k,T})}{\partial \bm{w}_k}$, i.e., 
\begin{equation}
    \label{eq:22-dec30}
    \frac{\partial \mathcal{L}_\text{V}(\bm{\theta}_{k,T})}{\partial w_k^{(j)}} = \Bigg\langle \nabla_{\bm{\theta}} \mathcal{L}_\text{V}(\bm{\theta}_{k,T}), \frac{\partial \bm{\theta}_{k, T}}{\partial w_k^{(j)}}\Bigg \rangle,
\end{equation}
using the chain rule. Now, $\frac{\partial \bm{\theta}_{k, T}}{\partial w_k^{(j)}}$ is not straightforward to compute. The following proposition describes how this quantity can be computed recursively.  
\begin{proposition}
\label{prop:recursion}
For $t \ge 0$, $\frac{\partial \bm{\theta}_{k, t}}{\partial w_k^{(j)}}$ evolves as:
\begin{equation}
    \label{eq:recursion}
    \frac{\partial \bm{\theta}_{k, t}}{\partial w_k^{(j)}} = \Big(\textbf{\textup{I}} - \eta \nabla^2_{\bm{\theta}} \mathcal{L}_\text{T}\big(\bm{\theta}_{k,t-1}, \bm{w}_k\big)\Big)\frac{\partial \bm{\theta}_{k, t-1}}{\partial w_k^{(j)}} - \eta \nabla \ell_j\big(\bm{\theta}_{k, t-1}\big),
\end{equation}
with initial condition $\frac{\partial \bm{\theta}_{k, 0}}{\partial w_k^{(j)}} = \vec{0}$. 
\end{proposition}
The proof of \Cref{prop:recursion} is in \Cref{pf-prop:recursion}. This recursion reveals that the influence of a weight $w_k^{(j)}$ accumulates over the trajectory, decayed by a matrix depending on the Hessian of the weighted training loss over the trajectory. Thus, the exact computation of $\frac{\partial \mathcal{L}_\text{V}(\bm{\theta}_{k,T})}{\partial w_k^{(j)}}$ becomes infeasible as $T$ increases. 

As a result, an extreme version of Alg. \ref{alg:1-main} is used in many practical algorithms. For example,  DoGE \cite{fan2023doge} and PIKE \cite{li2025pike} propose a ``greedy'' approach by setting $T=1$. However, using $T=1$ or a small value of $T$ might yield a poor approximation of the hypergradient, which might be detrimental to convergence. Indeed, in \Cref{lin:reg-example}, we show that $T=1$ leads to suboptimal performance in a simple example involving quadratic losses. 
\\
\\
\textbf{Our main focus.} Motivated by this conundrum, we seek to rigorously analyze the impact of the horizon $T$ on convergence. 
More importantly, given a fixed parameter update budget $N = K T$, we wish to quantify the optimal value of $T$ as a function of $N$ that results in the best convergence bound. We present this analysis in \Cref{convg-results} for a more general and practical version of \Cref{alg:1-main}.
\\
\\
\textbf{Our key insight.} We show that $T = \Theta(1)$ (here, $\Theta(\cdot)$ is with respect to the total parameter update budget $N$) is suboptimal and a larger value of $T$ leads to the best convergence bound; see \Cref{tab:summary} for a summary of our results. In other words, we show that making fewer weight updates with a good approximation of the hypergradient (obtained with a large $T$) leads to better performance than making more weight updates with a poor approximation of the hypergradient (obtained with $T = \Theta(1)$), i.e., \textit{less is more}.

\begin{table}[!htb]
\caption{Optimal horizon $T$ (leading to best convergence bound) as a function of the total parameter update budget $N$,  when the per-domain losses are strongly convex. In the deterministic (resp., stochastic) case, we have access to full (resp., stochastic) gradients. 
}
\label{tab:summary}
\begin{center}
\begin{small}
\begin{sc}
\begin{tabular}{lc}
\toprule
Setting & Optimal Horizon $T$ \\
\midrule
Deterministic (Sec. \ref{convg-res-sec}) & $\Theta\big(\log N\big)$ \\
Stochastic (Sec. \ref{stoc-main}) & $\Theta\big(\sqrt{N \log N}\big)$ \\
\bottomrule
\end{tabular}
\end{sc}
\end{small}
\end{center}
\end{table}

{It is worth clarifying here that we are not pitching \Cref{alg:1-main} as a novel algorithm; rather, our main contribution is theoretically deriving the optimal value of $T$ in it.}

\section{Motivating Example and Initial Insights}
\label{lin:reg-example}
Here we analyze a 1-dimensional quadratic setting where the \enquote{greedy} approach of using $T=1$ fails. 

\noindent\textbf{Quadratic example.} Consider a scalar parameter $\theta$ and weight $w \in [0,1]$ controlling two training domains:
\begin{align*}
    \nonumber
    \Ltrain(\theta, w) & = w\left(\frac{\theta^2}{2}\right) + (1-w)\left(\frac{(\theta-1)^2}{2}\right), \text{ } 
    \Lval(\theta) = \frac{\theta^2}{2}.
\end{align*}
Here, the validation domain is identical to the first training domain. Thus, the optimal weight is clearly $w^* = 1$, which yields $\theta^*=0$. 
For this example, the theorem below shows that the greedy approach of setting $T=1$ completely fails when starting far away from the optimum $\theta^*=0$, while a larger value of $T$ does not suffer from this problem.
\begin{theorem}[Informal: Failure of Greedy Approach]
\label{informal-quad-thm}
Suppose we run \Cref{alg:1-main} starting from $w_0 = 0.5$ and $\theta_0 = -R$, with $R>0$ being sufficiently large. Let the total parameter update budget be $N$ ($= K T$).
\begin{itemize}
    \item \textbf{Case 1 ($T=1$):} The final weight iterate $w_K$ converges to $0$, while the optimal weight is $1$.
    \item \textbf{Case 2 ($T \gg 1$):} If $T$ scales as $\Theta(\log R)$, the final weight iterate $w_K$ can be made to converge to a value arbitrarily close to the optimal value $1$.
\end{itemize}
\end{theorem}
The formal statement of \Cref{informal-quad-thm} and its proof are relegated to \Cref{lin:reg-example-app}. This result shows that $T=1$ is problematic even in a simple quadratic example, and converging to the optimal weight requires making use of a larger horizon $T$. %
At a high level, in the $T=1$ case, the algorithm updates weights based on immediate but \enquote{misleading} validation loss improvement. To see this, notice that at $\theta_0 = -R$, the gradient of the second domain's loss is steeper and locally reduces validation loss faster, incorrectly signaling that $w$ should decrease. In contrast, when using a larger value of $T$, the algorithm does not suffer from this short-term bias and receives the correct signal that $w$ should increase to align with the validation loss.

Note that under a constant budget $N = K T$, as $T$ increases, the number of %
weight updates $K$ decreases. Given that we showed $T \gg 1$ is better than $T = 1$, we also showed that making fewer weight updates with a relatively precise hypergradient is better than greedily making more weight updates with an erroneous hypergradient, i.e., \textit{less is more}!

With these insights in mind, we will move on to \Cref{convg-results} which presents a much more general analysis of this phenomenon.

\section{Main Results: Less is More}
\label{convg-results}
To formally understand the role of the horizon~$T$, this section provides convergence bounds for a more general and practical version of \Cref{alg:1-main}, assuming the per-domain losses are strongly-convex and smooth. 

\subsection{Practical Version of Algorithm~\ref{alg:1-main} Using Approximate Hessian}
\label{sec:pract-alg}
Let us first describe the computationally-practical version of \Cref{alg:1-main} that we analyze. Recall that to update the weights, we need to compute $\frac{\partial \mathcal{L}_\text{V}(\bm{\theta}_{k,T})}{\partial {w}_k^{(j)}}$ ($j \in [m]$) which by the chain rule is equal to $\big\langle \nabla_{\bm{\theta}} \mathcal{L}_\text{V}(\bm{\theta}_{k,T}), \frac{\partial \bm{\theta}_{k, T}}{\partial w_k^{(j)}}\big \rangle$ (see Eq. \eqref{eq:22-dec30}). As discussed,  $\frac{\partial \bm{\theta}_{k, T}}{\partial w_k^{(j)}}$ can be computed recursively with the following recursive update rule from \Cref{prop:recursion}:
\begin{flalign*}
    \frac{\partial \bm{\theta}_{k, t}}{\partial w_k^{(j)}} = \Big(\textbf{\textup{I}} - \eta \bm{H}_{k, t-1} \Big)\frac{\partial \bm{\theta}_{k, t-1}}{\partial w_k^{(j)}} - \eta \nabla \ell_j\big(\bm{\theta}_{k, t-1}\big),
\end{flalign*}
where $\bm{H}_{k, t-1} := \nabla^2_{\bm{\theta}} \mathcal{L}_\text{T}\big(\bm{\theta}_{k,t-1}, \bm{w}_k\big)$ and
initial condition $\frac{\partial \bm{\theta}_{k, 0}}{\partial w_k^{(j)}} = \vec{0}$. Since the computation of the Hessian $\bm{H}_{k, t}$ at every step is costly, %
we consider using a \textbf{single approximate Hessian} $\bm{H}_{k}$ for the entire round $k$. Specifically, we assume access to a \textbf{Hessian approximator} $\bm{\mathcal{H}}$ that takes as input $\bm{w}_k$ and returns an approximate Hessian $\bm{H}_{k} = \bm{\mathcal{H}}(\bm{w}_k)$.\footnote{We discuss a practical choice of $\bm{H}_{k}$ in \Cref{app:approx}.}

Under this scheme, let the approximation to $\Big\{\frac{\partial \bm{\theta}_{k, t}}{\partial w_k^{(j)}}\Big\}_t$ be denoted by $\big\{-\eta \bm{u}_{k,t}^{(j)}\big\}_t$, which has the following simpler recursive update:
\begin{equation}
    \label{eq:7}
    \bm{u}_{k, t}^{(j)} = \big(\textbf{\textup{I}} - \eta \bm{H}_{k}\big) \bm{u}_{k,t-1}^{(j)}  + \nabla \ell_j\big(\bm{\theta}_{k,t-1}\big),
\end{equation}
with $\bm{u}_{k,0}^{(j)} = \vec{0}$. %
Unfolding the recursion in Eq. \eqref{eq:7}, we get:
\begin{equation}
    \label{eq:24-dec30}
    \bm{u}_{k, T}^{(j)} = \sum_{i=0}^{T-1} \big(\textbf{\textup{I}} - \eta \bm{H}_{k}\big)^{T-1-i} \nabla \ell_j\big(\bm{\theta}_{k,i}\big).
\end{equation}
Let us denote the approximation for $\frac{\partial \mathcal{L}_\text{V}(\bm{\theta}_{k,T})}{\partial w_k^{(j)}}$ that we will obtain by approximating $\frac{\partial \bm{\theta}_{k, T}}{\partial w_k^{(j)}}$ with $-\eta \bm{u}_{k,T}^{(j)}$ by ${g}^{(j)}_{k,T}$; so, $\bm{g}_{k, T} = \Big[{g}^{(1)}_{k,T}, \ldots, {g}^{(m)}_{k,T}\Big]^\top$ is our approximate hypergradient. 
Plugging in \cref{eq:24-dec30} into \cref{eq:22-dec30} and dropping the subscript $\bm{\theta}$ in $\nabla_{\bm{\theta}} \mathcal{L}_\text{V}(\bm{\theta}_{k,T})$ henceforth, we get:
\begin{equation}
    \label{eq:26-jan1}
    {g}^{(j)}_{k,T} = -\eta \sum_{i=0}^{T-1} \Big\langle \nabla \mathcal{L}_\text{V}(\bm{\theta}_{k,T}), \big(\textbf{\textup{I}} - \eta \bm{H}_{k}\big)^{T-1-i} \nabla \ell_j\big(\bm{\theta}_{k,i}\big)\Big \rangle.
\end{equation}
The corresponding update rule for $\bm{w}_k$ is:
\begin{equation}
    w_{k+1}^{(j)} = \frac{w_{k}^{(j)} \exp\big(-\alpha {g}^{(j)}_{k,T}\big)}{\sum_{s=1}^m w_{k}^{(s)} \exp\big(-\alpha {g}^{(s)}_{k,T}\big)}.
\end{equation}
We concretely present this practical version in \Cref{alg:practical}.

\begin{algorithm}[!htb]
	\caption{Data Mixing with Approximate Hessian}
	\label{alg:practical}
	\begin{algorithmic}[1]
	    \STATE {\bfseries Input:} Initialization for model parameters $\bm{\theta}_0$, initialization for mixture weights $\bm{w}_0$, parameter update step-size $\eta$, weight update step-size $\alpha$, number of rounds $K$, number of parameter updates per round / \enquote{horizon} $T$, and Hessian approximator $\bm{\mathcal{H}}$ for weighted training loss. 
        \vspace{0.1 cm}
        \STATE Set $\bm{\theta}_{0, 0} = \bm{\theta}_0$.
        \vspace{0.1 cm}
        \FOR{$k \in \{0,\ldots,K-1\}$}
        \vspace{0.1 cm}
        \STATE Obtain Hessian approximation $\bm{H}_k = \bm{\mathcal{H}}(\bm{w}_k)$. %
        \vspace{0.1 cm}
        \STATE Set $\bm{u}_{k, 0}^{(j)} = \vec{\bm{0}}$ for $j \in [m]$.
        \vspace{0.1 cm}
        \FOR{$t \in \{0,\ldots,T-1\}$}
        \vspace{0.1 cm}
        \STATE Update $\bm{\theta}_{k, t+1} = \bm{\theta}_{k, t} - \eta \nabla_\theta \mathcal{L}_\text{T}(\bm{\theta}_{k, t}, \bm{w}_k)$. \\
        // \texttt{Inner parameter update}
        \vspace{0.1 cm}
        \FOR{$j \in \{1,\ldots,m\}$}
        \vspace{0.1 cm}
        \STATE Update $\bm{u}_{k, t+1}^{(j)} = \big(\textbf{\textup{I}} - \eta \bm{H}_{k}\big) \bm{u}_{k,t}^{(j)}  + \nabla \ell_j\big(\bm{\theta}_{k,t}\big)$
        \vspace{0.1 cm}
        \ENDFOR
        \vspace{0.1 cm}
        \ENDFOR
        \vspace{0.1 cm}
        \STATE For $j \in [m]$, update $w_{k+1}^{(j)} = \frac{w_{k}^{(j)} \exp\big(-\alpha {g}^{(j)}_{k, T}\big)}{\sum_{s=1}^m w_{k}^{(s)} \exp\big(-\alpha {g}^{(s)}_{k, T}\big)}$, where ${g}^{(j)}_{k, T} := -\eta  \Big\langle \nabla \mathcal{L}_\text{V}(\bm{\theta}_{k,T}), \bm{u}_{k,T}^{(j)}\Big \rangle$. 
        \\
        // \texttt{Outer weight update}
        \vspace{0.1 cm}
        \STATE Set $\bm{\theta}_{k+1, 0} = \bm{\theta}_{k, T}$.
        \vspace{0.1 cm}
        \ENDFOR 
        \vspace{0.1 cm}
		\STATE \textbf{Return} $\{\bm{\theta}_{k, 0}\}_{k=0}^{K}$ and $\{\bm{w}_k\}_{k=0}^K$.
	\end{algorithmic}
\end{algorithm}

\subsection{Convergence Result for Algorithm~\ref{alg:practical}}
\label{convg-res-sec}
Now we will provide a convergence result for %
\Cref{alg:practical} under the following assumptions.

\begin{assumption}[\textbf{Domain losses}]
    \label{asmp-1}
    Each $\ell_j(\bm{\theta})$ ($j \in [m]$) is $G$-Lipschitz, $L$-smooth, and $\mu$-strongly-convex. 
\end{assumption}

{Note that when the $\ell_j$'s are strongly convex, $\bm{\theta}^{*}(\bm{w})$ 
is unique and $F(\bm{w})$ in Eq. \eqref{eq:def} is unambiguously defined.}

\begin{assumption}[\textbf{Validation loss}]
    \label{asmp-2}
    $\mathcal{L}_\text{V}(\bm{\theta})$ is $G_\text{V}$-Lipschitz and $L_\text{V}$-smooth. %
\end{assumption}

\begin{assumption}[\textbf{Hessian approximation}]
    \label{asmp-3}
    For $\bm{w} \in \bm{\Delta}^m$, suppose $\bm{H}^{\ast}(\bm{w}) := \nabla_{\bm{\theta}}^2 \mathcal{L}_\text{T} \big(\bm{\theta}^{\ast}(\bm{w}), \bm{w}\big)$ is the Hessian at the optimum of the weighted training loss with weight $\bm{w}$, and $\bm{\mathcal{H}}(\bm{w})$ is the approximate Hessian returned by our Hessian approximator (recall that we denoted this by $\bm{H}_k$ for round $k$ in \Cref{sec:pract-alg}). Then, $\bm{\mathcal{H}}(\bm{w})$ is PSD with $\widehat{\mu} \textbf{\textup{I}} \preceq \bm{\mathcal{H}}(\bm{w}) \preceq \widehat{L} \textbf{\textup{I}}$ and $\|\bm{\mathcal{H}}(\bm{w}) - \bm{H}^{\ast}(\bm{w})\|_\textup{op} \leq \delta$ $\forall$ $\bm{w} \in \bm{\Delta}^m$.
\end{assumption}
We are now ready to present our result for \Cref{alg:practical}. 

\begin{theorem} \label{thm:main_convergence}
Suppose Assumptions \ref{asmp-1}, \ref{asmp-2} and \ref{asmp-3} hold,  $F(\bm{w}) = \mathcal{L}_\textup{V}\big(\bm{\theta}^{*}(\bm{w})\big)$ is convex (in $\bm{w}$), and $F^{*} := \min_{\bm{w} \in \bm{\Delta}^m} F(\bm{w})$. Let $N = K T$ denote the total parameter update budget. Suppose our initialization is $\bm{w}_0 = \frac{1}{m} \mathds{1}_m$. Then, $\exists$ $\alpha$ such that for any $\eta < \min\Big(\frac{1}{\widehat{L}}, \frac{\mu}{L^2}\Big)$ and $T \geq \big\lceil\frac{\log 4}{\eta \mu}\big\rceil$, we have the following guarantee for \Cref{alg:practical}:
\begin{flalign}
    \label{eq:bound}
    & F(\overline{\bm{w}}_K) - F^{*} \le \underbrace{\frac{c_1 \sqrt{T}}{\sqrt{N}}}_\textup{(I)} + \underbrace{\left(1 - \frac{\eta \mu}{2}\right)^{T-1} (c_2 \eta T + c_3)}_{\textup{(II)}}  + \underbrace{c_4 \delta}_\textup{(III)},
\end{flalign}
where $\overline{\bm{w}}_K := \frac{1}{K} \sum_{k=0}^{K-1} \bm{w}_k$, and $\{c_i\}_{i=1}^4$ are constants independent of $T$ and $N$. 
\\
In \cref{eq:bound}, term \textup{(I)} is the \textbf{optimization error}, term \textup{(II)} is the \textbf{finite horizon bias}, while term \textup{(III)} is the irreducible \textbf{Hessian approximation error}.
\end{theorem}
The detailed version and proof of \Cref{thm:main_convergence} are presented in \Cref{app:pf-main_convergence}. Terms (I) and (II) in Eq. \eqref{eq:bound} are reducible by increasing $N$ and choosing $T$ appropriately (as a function of $N$). However, term (III) is independent of $N$ and $T$, and cannot be reduced. This is the Hessian approximation cost, which depends on the quality of the approximator.
\\
\\
\textbf{Reducible error.} The reducible error ((I) + (II)) in \cref{eq:bound} is associated with a tradeoff with respect to the value of $T$. Specifically, the optimization error (I) increases with $T$ (or as as $K=N/T$ decreases), while the horizon bias (II) decays with $T$ asymptotically. As we discuss and show in the detailed version of \Cref{thm:main_convergence} (\Cref{thm2}), the reducible error is minimized by choosing 
\begin{equation}
    T = \Theta(\log N).\footnote{Note that the $\Theta(\cdot)$ used here and subsequently is w.r.t. $N$. Also, $\eta$ is chosen to be a constant independent of $N$ for this result.}
\end{equation}

\begin{remark}
    \label{rmk-det}
    The reducible error is minimized by setting $T = \Theta(\log N)$. \textbf{This establishes that $T = \Theta(1)$ is suboptimal} and making fewer weight updates with the horizon growing \textit{logarithmically} with the total parameter update budget is better, i.e., \textbf{less is more}!
\end{remark}

\noindent \textbf{Tightness.} Since $F(\bm{w})$ is assumed to be convex in \Cref{thm:main_convergence}, we have a lower bound of $\Omega\big(\frac{1}{\sqrt{N}}\big)$ in the worst case \cite{nesterov2013introductory}. Note that by choosing $T = \Theta(\log N)$, the reducible error in ((I) + (II)) in Eq. \eqref{eq:bound} becomes $\widetilde{\mathcal{O}}\big(\frac{1}{\sqrt{N}}\big)$ (where $\widetilde{\mathcal{O}}(\cdot)$ hides poly-log factors). Thus, our upper bound in \Cref{thm:main_convergence} with $T = \Theta(\log N)$ matches the lower bound (ignoring poly-log factors). This establishes that $T = \Theta(\log N)$ is indeed the \textbf{optimal choice} in this setting.

\subsection{Extension to the Stochastic Case}
\label{stoc-main}
Here, we consider a more practical version of \Cref{alg:practical}, where for each domain loss $\ell_j(\cdot)$ and the validation loss $\mathcal{L}_\text{V}(\cdot)$, we have access to unbiased stochastic gradients $\widetilde{\nabla} {\ell}_j(\cdot)$ and $\widetilde{\nabla} {\mathcal{L}}_\text{V}(\cdot)$ instead of the actual gradients $\nabla {\ell}_j(\cdot)$ and $\nabla {\mathcal{L}}_\text{V}(\cdot)$. In this case, we have three changes in \Cref{alg:practical}. First, line 7 (inner parameter update) becomes $$\bm{\theta}_{k, t+1} = \bm{\theta}_{k, t} - \eta \widetilde{\nabla}_\theta {\mathcal{L}}_\text{T}(\bm{\theta}_{k, t}, \bm{w}_k),$$ 
where $\widetilde{\nabla}_\theta {\mathcal{L}}_\text{T}(\bm{\theta}_{k, t}, \bm{w}_k) = \sum_{j=1}^m w_k^{(j)} \widetilde{\nabla} {\ell}_j\big(\bm{\theta}_{k,t}\big)$. Second, line 9 changes to
$$\widetilde{\bm{u}}_{k, t+1}^{(j)} = \big(\textbf{\textup{I}} - \eta \bm{H}_{k}\big) \widetilde{\bm{u}}_{k,t}^{(j)}  + \widetilde{\nabla} {\ell}_j\big(\bm{\theta}_{k,t}\big),$$
with $\widetilde{\bm{u}}_{k, 0}^{(j)} = \vec{0}$. Finally, line 12 (outer weight update) becomes
$$w_{k+1}^{(j)} = \frac{w_{k}^{(j)} \exp\big(-\alpha \widetilde{g}^{(j)}_{k}\big)}{\sum_{s=1}^m w_{k}^{(s)} \exp\big(-\alpha \widetilde{g}^{(s)}_{k}\big)},$$
with $\widetilde{g}^{(j)}_k = -\eta \big\langle \widetilde{\nabla} {\mathcal{L}}_\text{V}(\bm{\theta}_{k,T}), \widetilde{\bm{u}}_{k,T}^{(j)}\big \rangle$. Note that $\widetilde{\bm{g}}_{k, T} = \Big[\widetilde{g}^{(1)}_{k,T}, \ldots, \widetilde{g}^{(m)}_{k,T}\Big]^\top$ is our approximate hypergradient here. For full clarity, we write down the resultant algorithm in the stochastic case in \Cref{alg:stoc}. 
\begin{algorithm}[!htb]
	\caption{Data Mixing with Approximate Hessian and Stochastic Gradients}
	\label{alg:stoc}
	\begin{algorithmic}[1]
	    \STATE {\bfseries Input:} Initialization for model parameters $\bm{\theta}_0$, initialization for mixture weights $\bm{w}_0$, parameter update step-size $\eta$, weight update step-size $\alpha$, number of rounds $K$, number of parameter updates per round / \enquote{horizon} $T$, and Hessian approximator $\bm{\mathcal{H}}$ for the weighted training loss. 
        \vspace{0.1 cm}
        \STATE Set $\bm{\theta}_{0, 0} = \bm{\theta}_0$.
        \vspace{0.1 cm}
        \FOR{$k \in \{0,\ldots,K-1\}$}
        \vspace{0.1 cm}
        \STATE Obtain Hessian approximation $\bm{H}_k = \bm{\mathcal{H}}(\bm{w}_k)$. %
        \vspace{0.1 cm}
        \STATE Set $\widetilde{\bm{u}}_{k, 0}^{(j)} = \vec{\bm{0}}$ for $j \in [m]$.
        \vspace{0.1 cm}
        \FOR{$t \in \{0,\ldots,T-1\}$}
        \vspace{0.1 cm}
        \STATE Update $\bm{\theta}_{k, t+1} = \bm{\theta}_{k, t} - \eta \Big(\sum_{j=1}^m w_k^{(j)} \widetilde{\nabla} {\ell}_j\big(\bm{\theta}_{k,t}\big)\Big)$.
        \\
        // \texttt{Inner parameter update}
        \vspace{0.1 cm}
        \FOR{$j \in \{1,\ldots,m\}$}
        \vspace{0.1 cm}
        \STATE Update  $\widetilde{\bm{u}}_{k, t+1}^{(j)} = \big(\textbf{\textup{I}} - \eta \bm{H}_{k}\big) \widetilde{\bm{u}}_{k,t}^{(j)}  + \widetilde{\nabla} {\ell}_j\big(\bm{\theta}_{k,t}\big)$.
        \vspace{0.1 cm}
        \ENDFOR
        \vspace{0.1 cm}
        \ENDFOR
        \vspace{0.1 cm}
        \STATE For $j \in [m]$, update $w_{k+1}^{(j)} = \frac{w_{k}^{(j)} \exp\big(-\alpha \widetilde{g}^{(j)}_{k, T}\big)}{\sum_{s=1}^m w_{k}^{(s)} \exp\big(-\alpha \widetilde{g}^{(s)}_{k, T}\big)}$, where $\widetilde{g}^{(j)}_{k, T} := -\eta  \Big\langle \widetilde{\nabla} {\mathcal{L}}_\text{V}(\bm{\theta}_{k,T}), \widetilde{\bm{u}}_{k,T}^{(j)}\Big \rangle$. 
        \\
        // \texttt{Outer weight update}
        \vspace{0.1 cm}
        \STATE Set $\bm{\theta}_{k+1, 0} = \bm{\theta}_{k, T}$.
        \vspace{0.1 cm}
        \ENDFOR 
        \vspace{0.1 cm}
		\STATE \textbf{Return} $\{\bm{\theta}_{k, 0}\}_{k=0}^{K}$ and $\{\bm{w}_k\}_{k=0}^K$.
	\end{algorithmic}
\end{algorithm}

We make the following standard assumptions on the stochastic gradients.
\begin{assumption}[\textbf{Stochastic gradients}]
\label{asmp-stoc}
Suppose $\zeta$ denotes the source of randomness in the stochastic gradients. The stochastic gradients:
\begin{enumerate}
    \item are unbiased, i.e., $\mathbb{E}_\zeta\big[\widetilde{\nabla} {\ell}_j(\bm{\theta})\big] = \nabla {\ell}_j(\bm{\theta})$ and $\mathbb{E}_\zeta\big[\widetilde{\nabla} {\mathcal{L}}_\text{V}(\bm{\theta})\big] = \nabla {\mathcal{L}}_\text{V}(\bm{\theta}), \text{ } \forall \text{ } \bm{\theta}$.
    \item have bounded variance, i.e., $\mathbb{E}_\zeta\big[\|\widetilde{\nabla} {\ell}_j(\bm{\theta}) - \nabla {\ell}_j(\bm{\theta})\|_2^2\big] \leq \sigma^2  \text{ and } \mathbb{E}_\zeta\big[\|\widetilde{\nabla} {\mathcal{L}}_\text{V}(\bm{\theta}) - \nabla {\mathcal{L}}_\text{V}(\bm{\theta})\|_2^2\big] \leq \sigma^2, \text{ } \forall \text{ } \bm{\theta}$.
\end{enumerate}
\end{assumption}

We are now ready to present our convergence result for the algorithm in the stochastic case (\Cref{alg:stoc}).

\begin{theorem}[\textbf{Stochastic case}]\label{thm:stoc-main}
    Suppose Assumptions \ref{asmp-1}, \ref{asmp-2}, \ref{asmp-3} and \ref{asmp-stoc} hold, $F(\bm{w}) = \mathcal{L}_\textup{V}\big(\bm{\theta}^{*}(\bm{w})\big)$ is convex (in $\bm{w}$), and $F^{*} := \min_{\bm{w} \in \bm{\Delta}^m} F(\bm{w})$.
    Let $N = K T$ denote the total parameter update budget. Suppose our initialization is $\bm{w}_0 = \frac{1}{m} \mathds{1}_m$. Then, $\exists$ $\alpha$ and $\eta$ such that for any $T \geq \Omega(1)$ (w.r.t. $N$), we have the following guarantee for \Cref{alg:stoc}:
    \begin{equation}
        \label{eq:bound-stoc}
        \mathbb{E}\big[F(\overline{\bm{w}}_K)\big] - F^{*} \le \underbrace{\mathcal{O}\Big(\frac{\sqrt{T}}{\sqrt{N}}\Big)}_\textup{(I)} + \underbrace{\mathcal{O}\Big(\frac{\sqrt{\log T}}{\sqrt{T}}\Big)}_\textup{(II)} + \underbrace{\mathcal{O}(\delta)}_\textup{(III)},
    \end{equation}
    where $\overline{\bm{w}}_K := \frac{1}{K} \sum_{k=0}^{K-1} \bm{w}_k$.\footnote{The expectation in \cref{eq:bound-stoc} is w.r.t. the randomness of the algorithm due to use of stochastic gradients.} 
    
    In \cref{eq:bound-stoc}, term \textup{(I)} is the \textbf{optimization error}, term \textup{(II)} is the \textbf{finite horizon bias}, while term \textup{(III)} is the irreducible \textbf{Hessian approximation error}.
\end{theorem}
The detailed version and proof of \Cref{thm:stoc-main} are presented in \Cref{stoc-app}. Just like \Cref{thm:main_convergence}, there are reducible error terms (I) and (II), and an irreducible error term (III), which is the Hessian approximation cost.
\\
\\
\textbf{Reducible error.} Similar to \Cref{thm:main_convergence}, the optimization error (I) increases with $T$ (or as as $K=N/T$ decreases), while the horizon bias (II) decays with $T$ asymptotically. However, in this case, the decay in the horizon bias is much slower (as a function of $T$) than in \Cref{thm:main_convergence}; this is due to the use of stochastic gradients here. As we discuss and show in the detailed version of \Cref{thm:stoc-main} (\Cref{thm:stochastic-case}), the reducible error ((I) + (II)) is minimized by choosing 
\begin{equation}
    T = \Theta\big(\sqrt{N \log N}\big).
\end{equation}
\begin{remark}  
    \label{rmk-stoc}
    In the stochastic case, the reducible error is minimized by setting $T = \Theta\big(\sqrt{N \log N}\big)$ which is \textbf{significantly larger than $\Theta(1)$} as well as the optimal value of $\Theta(\log N)$ in the deterministic case (\Cref{rmk-det}).
\end{remark}

\subsection{Proof Sketch: Bounding the Hypergradient Error}
\label{pf-sketch}
The \textbf{core of our analysis} lies in \textbf{bounding the deviation of the approximate hypergradient} $\bm{g}_{k, T}$ (or its stochastic counterpart $\widetilde{\bm{g}}_{k, T}$) \textbf{from the true hypergradient} $\nabla F(\bm{w}_k)$; the rest of the analysis is based on standard results for mirror descent. Recall that the true hypergradient is given by the Implicit Function Theorem as 
$\nabla F(\bm{w})^{(j)} = - \big \langle \nabla_{\bm{\theta}} \mathcal{L}_\text{V} \big(\bm{\theta}^{\ast}\big),  \nabla_{\bm{\theta}}^2 \mathcal{L}_\text{T} \big(\bm{\theta}^{\ast}, \bm{w}\big)^{-1} \nabla \ell_j\big(\bm{\theta}^{\ast}\big) \big \rangle$, where $\bm{\theta}^{\ast}$ denotes $\bm{\theta}^{\ast}(\bm{w})$ (for brevity). Our analysis decomposes this error into two primary sources: the \textbf{finite-horizon bias} from truncating the inner optimization, and the \textbf{approximation error} from replacing the exact inverse Hessian at the optimum with a truncated Neumann series over the optimization trajectory using the approximate Hessian.

\paragraph{Deterministic setting (\Cref{thm:main_convergence}).}
Let $\bm{H}_k^{\ast} := \nabla_{\bm{\theta}}^2 \mathcal{L}_\text{T} \big(\bm{\theta}^{\ast}_k, \bm{w}_k\big)$. 
Here the analysis boils down to bounding $\|\nabla F(\bm{w}_k) - \bm{g}_{k, T}\|_\infty$, 
and this is done by expressing the inverse Hessian $(\bm{H}_k^{\ast})^{-1}$ in $\nabla F(\bm{w}_k)$ as an infinite Neumann series $\eta \sum_{i=0}^{\infty} \big(\textbf{\textup{I}} - \eta \bm{H}_{k}^{\ast}\big)^i$. Our estimated hypergradient $\bm{g}_{k, T}$ approximates this by using the approximate Hessian $\bm{H}_{k}$ instead of $\bm{H}_{k}^{\ast}$, truncating the sum to $T$ terms, and evaluating gradients along the finite trajectory $\{\bm{\theta}_{k,t}\}_{t=0}^{T}$ rather than at the optimum $\bm{\theta}^{*}_k$. We split the error into two parts: 
\begin{enumerate}
    \item \textbf{Tail error.} This is the error from ignoring terms $i \ge T$ in the Neumann series. Due to strong convexity of the per-domain losses, $\big(\textbf{\textup{I}} - \eta \bm{H}_{k}^{\ast}\big)$ is a contractive operator with spectral radius bounded by $(1 - \eta \mu)$. As a result, this tail error decays exponentially as $(1 - \eta \mu)^T$.
    \item \textbf{Trajectory error.} This is the error accumulating from evaluating gradients at $\{\bm{\theta}_{k,t}\}_{t=0}^T$ instead of $\bm{\theta}^{*}_k$ as well as from the use of approximate Hessian instead of exact Hessian. Since gradient descent on strongly-convex functions converges linearly, $\|\bm{\theta}_{k,t} - \bm{\theta}^{*}_k\|_2$ decays as $\big(1 - \mathcal{O}(\eta \mu)\big)^t$. Using this fact together with the Lipschitzness and smoothness of the per-domain and validation losses, it can be shown that this error is bounded by $\mathcal{O}\big(\eta T \big(1 - \mathcal{O}(\eta \mu)\big)^T\big) + \mathcal{O}(\delta)$, where the last term is due to the Hessian approximation and is irreducible.
\end{enumerate}
Combining these two errors, the total hypergradient error is bounded by $\mathcal{O}\big(\eta T \big(1 - \mathcal{O}(\eta \mu)\big)^T\big) + \mathcal{O}(\delta)$, where $\delta$ is the irreducible error from the Hessian approximator. This near-exponential decay in the hypergradient error drives the optimal choice of $T = \Theta(\log N)$.

\paragraph{Stochastic setting (\Cref{thm:stoc-main}).}
In the stochastic case, the high-level idea is similar to the deterministic case, but the analysis is more involved and tricky due to the use of stochastic gradients. After taking conditional expectations judiciously, the decompositions are similar to the deterministic case. The key difference here compared to the deterministic case is that the $\bm{\theta}_{k,t}$'s do not converge to $\bm{\theta}^{*}_k$ linearly; specifically, here we have $\mathbb{E}\big[\big\|\bm{\theta}_{k,T} - \bm{\theta}^{*}_k\big\|_2\big] \le \mathcal{O}\big(\big(1-\mathcal{O}(\eta\mu)\big)^{T}\big) + \mathcal{O}(\sqrt{\eta} \sigma)$. The noise variance term at the end $\mathcal{O}(\sqrt{\eta} \sigma)$ turns out to be the dominant term in the reducible\footnote{The irreducible $\mathcal{O}(\delta)$ term in the deterministic case (due to Hessian approximation) stays here as well.} part of the expected hypergradient error; the other terms decay as $\mathcal{O}\big(\eta T \big(1 - \mathcal{O}(\eta \mu)\big)^T\big)$ just like the deterministic case. By choosing $\eta = \mathcal{O}(\frac{\log T}{\mu T})$, the reducible part of the expected hypergradient error is minimized and it turns out to be $\mathcal{O}(\frac{\sqrt{\log T}}{\sqrt{T}})$. This decay rate is much slower than the deterministic case (where it was near-exponential asymptotically), and this necessitates a much larger optimal horizon $T = \Theta(\sqrt{N \log N})$.

\paragraph{Analysis without assuming convexity.} In all of our previous results, we assumed $F(\bm{w})$ to be convex. {Even in the case when $F(\bm{w})$ is smooth non-convex, the crux of the analysis lies in bounding the deviation of the approximate hypergradient from the true hypergradient, and the same analysis that we have in the convex case (as discussed in \Cref{pf-sketch}) essentially goes through in this case (assuming the $\ell_j(\bm{\theta})$'s are still strongly convex).  There are \textit{no additional important technical challenges} compared to the convex case; so the exact derivation is left for future work. 
}

\section{Empirical Evaluation}\label{sec:experiment}

\begin{figure*}[t]
\centering

\begin{subfigure}{0.42\textwidth}
\centering
\includegraphics[width=\linewidth]{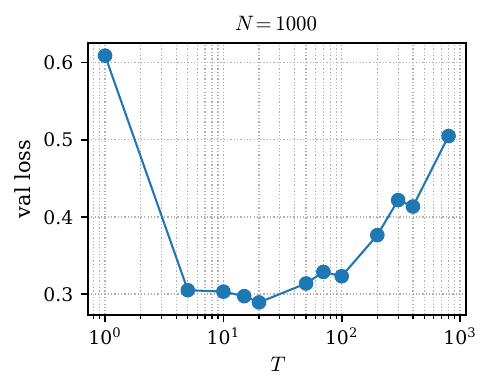}
\caption{Validation loss vs. $T$ for $N=1000$}
\label{fig:val_loss}
\end{subfigure}\hspace{0.7em}%
\begin{subfigure}{0.42\textwidth}
\centering
\includegraphics[width=\linewidth]{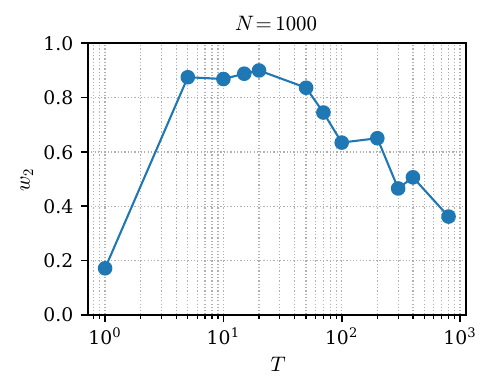}
\caption{$w_2$ vs. $T$ for $N=1000$}
\label{fig:w_T}
\end{subfigure}

\vspace{0.1em}

\begin{subfigure}{0.42\textwidth}
\centering
\includegraphics[width=\linewidth]{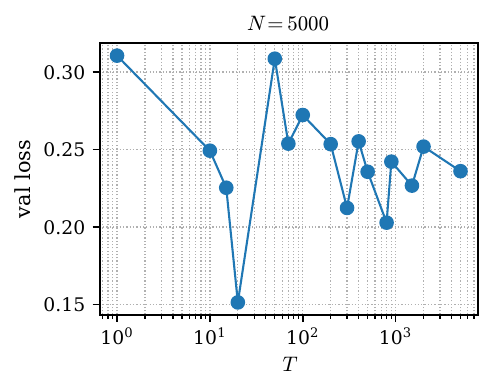}
\caption{Validation loss vs. $T$ for $N=5000$}
\label{fig:val_loss_2_5000_main}
\end{subfigure}\hspace{0.7em}%
\begin{subfigure}{0.42\textwidth}
\centering
\includegraphics[width=\linewidth]{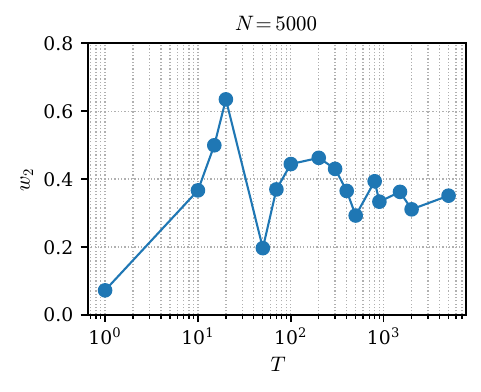}
\caption{$w_2$ vs. $T$ for $N=5000$}
\label{fig:w_2_5000_main}
\end{subfigure}

\caption{%
Validation loss and the weight of the second domain (most aligned with validation data) $w_2$ as a function of the horizon $T$, for $N=\{1000,5000\}$.
Note that the validation loss is lowest and $w_2$ is highest when $T$ is \textit{larger than} $1$ and \textit{sublinear} in $N$.
}
\label{fig:alphaRotatedLossAndW}
\end{figure*}
We design a multi-domain training experiment to illustrate how increasing the horizon $T$ (i.e., updating the mixture weights less frequently) leads to better performance. 
Specifically, we consider the following multi-domain data built out of MNIST.
\\
\\
\textbf{Data Domains:}
We construct three training domains (for a pretraining setting) from standard MNIST data. Domain (i) is the \textbf{vanilla domain} with standard pre-processing, Domain (ii) is the \textbf{rotated domain} with strong randomly chosen {rotation augmentation}, and Domain (iii) is the \textbf{noisy domain} which is standard MNIST corrupted with label noise. Validation data is a fixed rotated version of the MNIST test set, so the rotated domain (Domain (ii)) is the one that best matches the validation distribution.
\\
\\
We train a {CNN} whose architecture we describe in Appendix~\ref{sec:appendix_experiments}. 
\\
\\
\textbf{Data Mixing Algorithm:} We consider {Algorithm~\ref{alg:stoc}} with approximate Hessian $\bm{H}_k = \gamma \textbf{I}$, for all $k$ (see \Cref{app:approx}). More details about the hyperparameters, etc., are deferred to Appendix~\ref{sec:appendix_experiments}.
\\
\\
In \Cref{fig:alphaRotatedLossAndW}, we show the validation loss and the weight of the second domain (which is the ``important'' domain, most aligned with validation data), viz., $w_2$, as a function of the horizon $T$ {for $N=1000$ and $N=5000$}. We discuss the results next. In \Cref{fig:alphaRotatedLossAndW_full} (\Cref{sec:appendix_experiments}), we also plot the validation accuracies. 
\\
\\
\textbf{Optimal horizon is not small.} The performance is poor with a small horizon $T$ (or very frequent weight updates). With small $T$, the validation loss is high %
(Figures \ref{fig:val_loss} and \ref{fig:val_loss_2_5000_main}), and the corresponding values of $w_2$ are small (Figures \ref{fig:w_T} and \ref{fig:w_2_5000_main}). This suggests that the hypergradients with small $T$ are inaccurate and misled by short-horizon bias, underestimating the eventual benefit of allocating high weight to an important domain that may not reduce the loss immediately. 
As $T$ increases to an \textbf{intermediate regime}, %
\textbf{performance improves substantially}: validation loss %
decreases %
and $w_2$ increases, yielding a ``sweet spot'' at moderate horizons. In this regime, the hypergradients are much more precise and reflective of the domain's actual importance, leading to good performance. However, when $T$ becomes too large, %
performance degrades again because of too few weight updates. 

Overall, this experiment corroborates our theoretical result that small horizons are suboptimal due to being prone to getting misled by short-term spurious signals, while intermediate (in particular, sublinear in $N$) horizons do not suffer from this issue. Thus, \textit{less is more}!

\section{Conclusion}
In this work, we showed that the prevalent \enquote{greedy} heuristic of frequent weight updates with
short horizons ($T = \Theta(1)$) is systematically biased, often converging far away from the optimal domain weights. By deriving rigorous scaling laws for the horizon, we established a \enquote{less is more} principle, proving that updating weights less frequently with a horizon growing sublinearly in the total parameter update budget significantly improves convergence. 

We conclude this paper by discussing some limitations and future directions of work. {Our results in this paper are under the assumption that the per-domain losses are strongly convex. While this assumption is limiting, a critical challenge in analyzing the general non-convex case is that for $F(\bm{w})$ to be well-defined in Eq. \eqref{eq:def}, $\bm{\theta}^*(\bm{w})$ needs to be uniquely specified as one of the minimizers of $\mathcal{L}_\text{T}(\bm{\theta}, \bm{w})$, and the distance of the iterates from this $\bm{\theta}^*(\bm{w})$ needs to be bounded in the proofs. It is not clear how this can be done, and so, we leave this analysis for future work. Also, we would like to conduct experiments on larger models and datasets in the future.}

\bibliography{refs}

\newpage
\appendix
\onecolumn

{\begin{center}{\LARGE \bf Appendix}\end{center}}

\section{Quadratic Example in Section~\ref{lin:reg-example}}
\label{lin:reg-example-app}
In this case, it can be verified that the mirror descent update for $w_k$ from \Cref{alg:1-main} is:
\begin{equation}
    w_{k+1} = \frac{w_k}{w_k + (1-w_k) \exp(\alpha \bar{g}_{k, T})}, \text{where } \bar{g}_{k, T} := \frac{\partial \mathcal{L}_\text{V}(\theta_{k,T})}{\partial w_k}.
\end{equation}
It can be verified that the parameter iterates evolve as:
\begin{flalign}
    \label{eq:39}
    \theta_{k+1, 0} = \theta_{k, T} & = \big(1 - \eta\big)^T \theta_{k, 0} + ({1-w_k})\Big(1 - \big(1 - \eta\big)^T\Big).
\end{flalign}
Now:
\begin{flalign}
    \mathcal{L}_\text{V}({\theta}_{k,T}) = \frac{1}{2} \theta_{k,T}^2 = \frac{1}{2} \Bigg(\big(1 - \eta\big)^T \theta_{k, 0} + ({1-w_k})\Big(1 - \big(1 - \eta\big)^T\Big)\Bigg)^2,
\end{flalign}
and
\begin{flalign}
    \nonumber
    \bar{g}_{k, T} = \frac{\partial \mathcal{L}_\text{V}(\theta_{k,T})}{\partial w_k} & = - \underbrace{\Bigg(\big(1 - \eta\big)^T \theta_{k, 0} + ({1-w_k})\Big(1 - \big(1 - \eta\big)^T\Big)\Bigg)}_{= \theta_{k+1,0}} \Big(1 - \big(1 - \eta\big)^T\Big) 
    \\ 
    \label{eq:4-dec26}
    & = -{\theta}_{k+1,0} \Big(1 - \big(1 - \eta\big)^T\Big).
\end{flalign}
Thus, the update for the weight will be:
\begin{equation}    
    \label{eq:42}
    w_{k+1} = \frac{w_k}{w_k + (1-w_k) \exp(\alpha \bar{g}_{k, T})} = \frac{w_k}{w_k + (1-w_k) \exp\Big(-\alpha \theta_{k+1,0} \big(1 - \big(1 - \eta\big)^T\big)\Big)}.
\end{equation}

We will now state the formal version of \Cref{informal-quad-thm} and then prove it.

\begin{theorem}
    Let us fix the total parameter update budget in \Cref{alg:1-main}, i.e., $K T$, to be $N$. Define $\overline{R} := \frac{\eta N}{(1-\eta) \big(1 - (1-\eta)^N\big)}$, where $\eta < 1$ is the parameter update step-size. Suppose we begin with $w_0 = \frac{1}{2}$ and $\theta_{0} = -R$, where $R > \overline{R} > 0$. In that case:
    \begin{itemize}
        \item If $T = \overline{T} = \Big \lceil \frac{(c+1) \log(2R+1)}{\log(1/(1-\eta))} \Big \rceil$ with $c > 0$, then  $w_K \geq \frac{e^{\beta(\alpha/2)}}{1 + e^{\beta(\alpha/2)}}$, where $K = N/\overline{T}$, $\alpha$ is the weight update step-size, and $\beta = \Big(1 - \frac{1}{(2R+1)^c}\Big) \Big(1 - \frac{1}{(2R + 1)^{K(c+1)}}\Big)$. %
        In particular, for large $R$, $\beta \approx 1$ and $w_K \gtrsim \frac{e^{\alpha/2}}{1 + e^{\alpha/2}}$.\footnote{Here, $\gtrsim$ denotes nearly greater than.}
        \item If $T = 1$ (note that $K = N$ in this case), $w_N < \frac{1}{2}$. In particular, as $R \to \infty$, $w_N \to 0$. 
    \end{itemize}
\end{theorem}

\begin{proof}
\textbf{Case 1: $T = \overline{T} = \Big \lceil \frac{(c+1) \log(2R+1)}{\log(1/(1-\eta))} \Big \rceil$ with $c > 0$.} In this case, we have using \cref{eq:39}:
\begin{equation}
    \theta_{1, 0} = - (1 - \eta)^{\overline{T}} R + \frac{1}{2}\Big(1 - (1 - \eta)^{\overline{T}}\Big).
\end{equation}
Since $\overline{T} = \Big \lceil \frac{(c+1) \log(2R+1)}{\log(1/(1-\eta))} \Big \rceil$, we get:
\begin{equation}
    \label{eq:10-dec27}
    \theta_{1, 0} \geq \frac{1}{2}\Bigg(1 - \frac{1}{(2R+1)^c}\Bigg) > 0.
\end{equation}
Since $\theta_{1, 0} > 0$, $\bar{g}_{0, \overline{T}} < 0$ (see \Cref{eq:4-dec26}) and thus, $w_1 > w_0 = \frac{1}{2}$ (this follows from \cref{eq:42}). Next, we have:
\begin{flalign}
    \label{eq:39-dec27}
    \theta_{2, 0} & = \big(1 - \eta\big)^{\overline{T}} \theta_{1, 0} + ({1-w_1})\Big(1 - \big(1 - \eta\big)^{\overline{T}}\Big).
\end{flalign}
Since $\theta_{1, 0} > 0$, we will also have $\theta_{2, 0} > 0$ and as a result $\bar{g}_{1, \overline{T}} < 0 \implies w_2 > w_1 > \frac{1}{2}$. This process will keep repeating and we will have for all $k \geq 2$:
$\theta_{k, 0} > 0$ and $w_k > \ldots > w_1 > \frac{1}{2}$. Let us sharpen these bounds now.

Note that $\theta_{2, 0} \geq \big(1 - \eta\big)^{\overline{T}} \theta_{1, 0}$. Using \cref{eq:39}, we have $\theta_{k, 0} \geq \big(1 - \eta\big)^{\overline{T}} \theta_{k-1, 0}$ for all $k > 2$ as well (this is because $({1-w_k})\big(1 - \big(1 - \eta\big)^{\overline{T}}\big) \geq 0$ always). Unfolding this recursion, 
\begin{equation}
    \label{eq:12-dec27}
    \theta_{k, 0} \geq \big(1 - \eta\big)^{(k-1)\overline{T}} \theta_{1, 0},
\end{equation}
for all $k \geq 2$. 

Let us define $\phi_k := \log \frac{w_k}{1 - w_k}$. As per the mirror descent update for $w_k$ (\cref{eq:42}), we have the following update rule for $\phi_k$:
\begin{equation}  
    \label{eq:13-dec27}
    \phi_k = \phi_{k-1} - \alpha \bar{g}_{k-1, \overline{T}} = \phi_{k-1} + \alpha \theta_{k, 0} \Big(1 - \big(1-\eta\big)^{\overline{T}}\Big),
\end{equation}
where the last step follows from \cref{eq:4-dec26}. Using \cref{eq:12-dec27} in \cref{eq:13-dec27}, we get:
\begin{equation}
    \phi_k \geq \phi_{k-1} + \alpha \Big(1 - \big(1-\eta\big)^{\overline{T}}\Big) \big(1 - \eta\big)^{(k-1){\overline{T}}} \theta_{1, 0}. %
\end{equation}
Applying this recursively, we get:
\begin{equation}
    \phi_k \geq \phi_{1} + \alpha \Big(1 - \big(1-\eta\big)^{\overline{T}}\Big) \sum_{l=1}^{k-1}\big(1 - \eta\big)^{l {\overline{T}}} \theta_{1, 0}.
\end{equation}
From \cref{eq:13-dec27}, we have $\phi_1 = \phi_0 + \alpha \Big(1 - \big(1-\eta\big)^{\overline{T}}\Big) \theta_{1, 0}$. Plugging this above, we get:
\begin{equation}
    \phi_k \geq \phi_{0} + \alpha \Big(1 - \big(1-\eta\big)^{\overline{T}}\Big) \sum_{l=0}^{k-1}\big(1 - \eta\big)^{l {\overline{T}}} \theta_{1, 0}.
\end{equation}
Note that $\phi_0 = 0$ as $w_0 = \frac{1}{2}$. Using $\sum_{l=0}^{k-1}\big(1 - \eta\big)^{l {\overline{T}}} = \frac{1 - (1 - \eta)^{k {\overline{T}}}}{1 - (1 - \eta)^{\overline{T}}}$ and $\phi_0 = 0$ above, we get:
\begin{equation}
    \phi_k \geq \alpha \theta_{1, 0} \Big(1 - (1 - \eta)^{k {\overline{T}}}\Big).
\end{equation}
Recalling that $\overline{T} = \Big \lceil \frac{(c+1) \log(2R+1)}{\log(1/(1-\eta))} \Big \rceil \geq \frac{(c+1) \log(2R+1)}{\log(1/(1-\eta))}$, we have $(1 - \eta)^{{\overline{T}}} \leq \frac{1}{(2R + 1)^{c+1}}$. Using this and \cref{eq:10-dec27} above, we get for $k = K$:
\begin{equation}
    \phi_K \geq \frac{\alpha}{2}\Bigg(1 - \frac{1}{(2R+1)^c}\Bigg) \Bigg(1 - \frac{1}{(2R + 1)^{K(c+1)}}\Bigg). 
\end{equation}
Let $\beta := \Big(1 - \frac{1}{(2R+1)^c}\Big) \Big(1 - \frac{1}{(2R + 1)^{K(c+1)}}\Big)$. This gives us:
\begin{equation}
    w_K \geq \frac{e^{\beta(\alpha/2)}}{1 + e^{\beta(\alpha/2)}}.
\end{equation}
\\
\\
\textbf{Case 2: $T=1$}. We perform a total of $N$ parameter and weight updates. In this case, let us denote the parameter iterates simply by $\{\theta_k\}_{k \geq 0}$. Plugging in $T=1$ in \Cref{eq:39}, here we have:
\begin{equation}
    \label{eq:12-dec26}
    \theta_{k} =  (1 - \eta) \theta_{k-1} + \eta (1-w_{k-1}).
\end{equation}
Since $1 - w_{k-1} \leq 1$ (as $w_{k-1} \geq 0$), we have that:
\begin{equation}
    \theta_{k} \leq (1 - \eta) \theta_{k-1} + \eta \leq (1 - \eta)^2 \theta_{k-2} + \eta \Big(1 + (1-\eta)\Big) \leq \ldots \leq (1 - \eta)^k \theta_{0} + \eta \sum_{l=0}^{k-1} (1-\eta)^l.
\end{equation}
Using the fact that $\sum_{l=0}^{k-1} (1-\eta)^l = \frac{1 - (1-\eta)^k}{\eta}$ above and recalling that $\theta_0 = -R$, we get:
\begin{equation}
    \label{eq:14-dec26}
    \theta_k \leq 1 - (R+1) (1 - \eta)^k.
\end{equation}
{As done in Case 1, let us define $\phi_k := \log \frac{w_k}{1 - w_k}$. As per the mirror descent update for $w_k$ (\cref{eq:42}), we have the following update rule for $\phi_k$:
\begin{equation}   
    \label{eq:15-dec26}
    \phi_k = \phi_{k-1} - \alpha \bar{g}_{k-1, 1} = \phi_{k-1} + \eta \alpha \theta_k,
\end{equation}
where the last step follows from \cref{eq:4-dec26} with $T=1$.}
Unfolding the recursion in \cref{eq:15-dec26}, we get:
\begin{equation}
    \label{eq:19-dec27}
    \phi_k = \phi_0 + \eta \alpha \sum_{l=1}^k \theta_l.
\end{equation}
Note that $\phi_0 = 0$ as $w_0 = \frac{1}{2}$. Using this and applying \cref{eq:14-dec26} in \cref{eq:19-dec27}, we get for $k = N$:
\begin{equation}
    \phi_N \leq \eta \alpha \Big(N - (R+1) \sum_{l=1}^N (1-\eta)^l \Big) = \eta \alpha \Big(N - \frac{(R+1) (1-\eta)\big(1 - (1-\eta)^N\big)}{\eta} \Big).
\end{equation}
Note that if $R > \frac{\eta N}{(1-\eta) \big(1 - (1-\eta)^N\big)}$, then $\phi_N < 0 \implies w_N < \frac{1}{2}$. 
{Moreover, for $R \to \infty$, we have $\phi_N \to -\infty \implies w_N \to 0$. 
}
\end{proof}

\section{Proof of Proposition~\ref{prop:recursion}}
\label{pf-prop:recursion}
For the reader's convenience, we first restate \Cref{prop:recursion} and then prove it.
\begin{proposition}
\label{prop:recursion-restated}
For $t \ge 0$, we have:
\begin{equation*}
    \frac{\partial \bm{\theta}_{k, t}}{\partial w_k^{(j)}} = \Big(\textbf{\textup{I}} - \eta \nabla^2_{\bm{\theta}} \mathcal{L}_\text{T}\big(\bm{\theta}_{k,t-1}, \bm{w}_k\big)\Big)\frac{\partial \bm{\theta}_{k, t-1}}{\partial w_k^{(j)}} - \eta \nabla \ell_j\big(\bm{\theta}_{k, t-1}\big),
\end{equation*}
with $\frac{\partial \bm{\theta}_{k, 0}}{\partial w_k^{(j)}} = \vec{0}$. 
\end{proposition}

\begin{proof}
Note that in any round $k$, the update rule for gradient descent on the parameters at step $t \geq 1$ is:
\begin{equation}
    \label{eq:21-jan2}
    \bm{\theta}_{k,t} = \bm{\theta}_{k,t-1} - \eta \nabla_{\bm{\theta}} \mathcal{L}_\text{T}\big(\bm{\theta}_{k, t-1}, \bm{w}_k\big) = \bm{\theta}_{k,t-1} - \eta \sum_{i=1}^m w_k^{(i)} \nabla \ell_i\big(\bm{\theta}_{k,t-1}\big). 
\end{equation}
As per \cref{eq:21-jan2}, note that:
\begin{align}
    \frac{\partial \bm{\theta}_{k, t}}{\partial w_k^{(j)}} & = \frac{\partial \bm{\theta}_{k, t-1}}{\partial w_k^{(j)}} - \eta \sum_{i=1}^m \frac{\partial \big(w_k^{(i)} \nabla \ell_i(\bm{\theta}_{k,t-1})\big)}{\partial w_k^{(j)}}
    \\
    & = \frac{\partial \bm{\theta}_{k, t-1}}{\partial w_k^{(j)}} - \eta \sum_{i=1}^m w_k^{(i)} \frac{\partial \nabla \ell_i(\bm{\theta}_{k,t-1})}{\partial w_k^{(j)}} - \eta \nabla \ell_j(\bm{\theta}_{k, t-1})
    \\
    \label{eq:5-july18}
    & = \frac{\partial \bm{\theta}_{k, t-1}}{\partial w_k^{(j)}} - \eta \sum_{i=1}^m w_k^{(i)} \nabla^2 \ell_i(\bm{\theta}_{k, t-1})\frac{\partial \bm{\theta}_{k,t-1}}{\partial w_k^{(j)}} - \eta \nabla \ell_j(\bm{\theta}_{k, t-1}).
\end{align}
In \cref{eq:5-july18}, we have used the chain rule. Recalling that $\mathcal{L}_\text{T}(\bm{\theta}, \bm{w}) := \sum_{i=1}^m w^{(i)} \ell_i(\bm{\theta})$, we can rewrite \cref{eq:5-july18} as:
\begin{equation}
    \frac{\partial \bm{\theta}_{k, t}}{\partial w_k^{(j)}} = \Big(\textbf{I} - \eta \nabla^2_{\bm{\theta}} \mathcal{L}_\text{T}(\bm{\theta}_{k,t-1}, \bm{w}_k)\Big)\frac{\partial \bm{\theta}_{k,t-1}}{\partial w_k^{(j)}} - \eta \nabla \ell_j(\bm{\theta}_{k,t-1}).
\end{equation}
This finishes the proof.
\end{proof}

\section{Detailed Version and Proof of Theorem~\ref{thm:main_convergence}}
\label{app:pf-main_convergence}
We will first write the statement for the detailed version of \Cref{thm:main_convergence} and then prove it.

\begin{theorem}[\textbf{Detailed version of \Cref{thm:main_convergence}}]%
\label{thm2}
Suppose Assumptions \ref{asmp-1}, \ref{asmp-2}, and \ref{asmp-3} hold. Let $F(\bm{w}) := \mathcal{L}_\textup{V}\big(\bm{\theta}^{*}(\bm{w})\big)$ be convex (in $\bm{w} \in \bm{\Delta}^m$) and $F^{*} = \min_{\bm{w} \in \bm{\Delta}^m} F(\bm{w})$. Let $\max_{i, j}\big\|\textup{arg min}_{\bm{\theta}} \text{ } \ell_i(\bm{\theta}) - \textup{arg min}_{\bm{\theta}} \text{ } \ell_j(\bm{\theta})\big\|_2 \leq D$ and $R := 2 \max\Big(\big\|\bm{\theta}_{0, 0} - \bm{\theta}^{\ast}_0\big\|_2, \Big(\frac{2L}{\mu} + 1\Big) D\Big)$. Suppose we begin from $\bm{w}_0 = \frac{1}{m} \mathds{1}_m$ and choose $\alpha = \frac{\widehat{\mu} \sqrt{\log m}}{\sqrt{K} G G_\text{V}}$. Then for any {$\eta < \min\Big(\frac{1}{\widehat{L}}, \frac{\mu}{L^2}\Big)$} and $T \geq \big\lceil\frac{\log 4}{\eta \mu}\big\rceil$, we have the following guarantee for \Cref{alg:practical}:
\begin{align}
    \nonumber
    F\Bigg(\frac{1}{K} \sum_{k=0}^{K-1} \bm{w}_k\Bigg) - F^{\ast} & \leq \underbrace{\frac{3 G G_\text{V} \sqrt{\log m}}{2 \widehat{\mu} \sqrt{K}} + 2 \Big(1 - \frac{\eta \mu}{2}\Big)^{T-1} \Bigg(\eta T \big(L G_\text{V} + L_\text{V} G \big) R + \frac{G G_\text{V}}{\mu}\Bigg)}_\text{reducible error $\mathcal{E}$} 
    \\
    \label{eq:41-jan19}
    & \quad + \underbrace{\frac{2 \delta G G_\text{V}}{\min\big(\mu^2, \widehat{\mu}^2\big)}}_\text{irreducible error}.
\end{align}
If we fix the total parameter update budget to $N$, i.e., $K T = N$, then the reducible error as a function of $T$ is
\begin{equation}
    \mathcal{E}(T) = \frac{3 G G_\text{V} \sqrt{\log m}\sqrt{T}}{2 \widehat{\mu} \sqrt{N}} + 2 \Big(1 - \frac{\eta \mu}{2}\Big)^{T-1} \Bigg(\eta T \big(L G_\text{V} + L_\text{V} G \big) R + \frac{G G_\text{V}}{\mu}\Bigg),
\end{equation}
and this is minimized by choosing $T = \Theta(\log N)$. In particular, if we choose $T = \big\lceil\frac{2 \log N}{\eta \mu}\rceil$, the reducible error is:
\begin{equation}
    \mathcal{E} = \mathcal{O}\Bigg(\Bigg(\frac{G G_\text{V} \sqrt{\log m}}{\sqrt{\eta \mu} \widehat{\mu}}\Bigg) \frac{\sqrt{\log N}}{\sqrt{N}}\Bigg).
\end{equation}
The last term in \cref{eq:41-jan19}, i.e., $\frac{2 \delta G G_\text{V}}{\min(\mu^2, \widehat{\mu}^2)}$, is independent of $K$ and $T$, and cannot be reduced (by adjusting $K$ or $T$).
\end{theorem}

\begin{proof}
    Note that $\bm{g}_{k, T} = \Big[{g}^{(1)}_{k,T}, \ldots, {g}^{(m)}_{k,T}\Big]^\top$ is our algorithm's approximated hypergradient. Let $\bm{w}^{*} \in \text{arg min}_{\bm{w} \in \bm{\Delta}^m} F(\bm{w})$. Using the convexity of $F$, we have:
    \begin{flalign}
        \nonumber
        F(\bm{w}_k) - F(\bm{w}^{*}) \leq \langle \nabla F(\bm{w}_k), \bm{w}_k - \bm{w}^{*}\rangle = \langle \bm{g}_{k,T}, \bm{w}_k - \bm{w}^{*}\rangle + \langle \nabla F(\bm{w}_k) - \bm{g}_{k,T}, \bm{w}_k - \bm{w}^{*}\rangle.
    \end{flalign}
    Thus:
    \begin{equation}
        \label{eq:27-jan2}
        \frac{1}{K} \sum_{k=0}^{K-1} \Big(F(\bm{w}_k) - F(\bm{w}^{*})\Big) \leq \underbrace{\frac{1}{K} \sum_{k=0}^{K-1}  \big \langle \bm{g}_{k,T}, \bm{w}_k - \bm{w}^{*} \big \rangle}_\text{(I)} + \underbrace{\frac{1}{K} \sum_{k=0}^{K-1} \Big \langle \nabla F(\bm{w}_k) - \bm{g}_{k,T}, \bm{w}_k - \bm{w}^{*}\Big \rangle}_\text{(II)}.
    \end{equation}
    We will bound (I) first. To that end, note that our mirror descent (MD) update rule for the weights is obtained by solving:
    $$\bm{w}_{k+1} = \text{arg min}_{\bm{w} \in \bm{\Delta}^m} \langle\bm{g}_{k,T}, \bm{w} \rangle + \frac{1}{\alpha}D_\text{KL}(\bm{w} || \bm{w}_k).$$
    Note that for each $j \in [m]$:
    \begin{flalign}
        \nonumber
        \big|{g}^{(j)}_{k,T}\big| & = \eta \Bigg| \sum_{i=0}^{T-1} \Bigg\langle \nabla \mathcal{L}_\text{V}(\bm{\theta}_{k,T}), \big(\textbf{\textup{I}} - \eta \bm{H}_{k}\big)^{T-1-i} \nabla \ell_j\big(\bm{\theta}_{k,i}\big)\Bigg \rangle\Bigg|
        \\
        \nonumber
        & \leq \eta \sum_{i=0}^{T-1} \Bigg|\Bigg\langle \nabla \mathcal{L}_\text{V}(\bm{\theta}_{k,T}), \big(\textbf{\textup{I}} - \eta \bm{H}_{k}\big)^{T-1-i} \nabla \ell_j\big(\bm{\theta}_{k,i}\big)\Bigg \rangle\Bigg|
        \\
        \nonumber
        & \leq \eta \sum_{i=0}^{T-1} \big\|\textbf{\textup{I}} - \eta \bm{H}_{k}\big\|_\text{op}^{T-1-i} \Big\|\nabla \mathcal{L}_\text{V}(\bm{\theta}_{k,T})\Big\|_2 \Big\|\nabla \ell_j\big(\bm{\theta}_{k,i}\big)\Big\|_2
        \\
        \label{eq:29-dec30}
        & \leq \eta \sum_{i=0}^{T-1} (1 - \eta \widehat{\mu})^{T-1-i} G_\text{V} G,
    \end{flalign}
    where we have used $\bm{H}_k$ is PSD with $\widehat{\mu} \textbf{\textup{I}} \preceq \bm{H}_k \preceq \widehat{L} \textbf{\textup{I}}$ and $\eta \leq \frac{1}{\widehat{L}}$ due to which $\big(\textbf{\textup{I}} - \eta \bm{H}_{k}\big)$ is PSD with $\big\|\textbf{\textup{I}} - \eta \bm{H}_{k}\big\|_\text{op} \leq (1 - \eta \widehat{\mu})$, $\mathcal{L}_\text{V}$ is $G_\text{V}$-Lipschitz, and $\ell_j$ is $G$-Lipschitz. Simplifying \cref{eq:29-dec30}, we get:
    \begin{equation}
        \label{eq:30-dec30}
        \|\bm{g}_{k, T}\|_\infty \leq \max_{j \in [m]} \big|{g}^{(j)}_{k,T}\big| \leq \frac{G G_\text{V}}{\widehat{\mu}}. %
    \end{equation}
    {Recall that we start from $\bm{w}_0 = \frac{1}{m} \mathds{1}_m$.} Then using \Cref{lem1} and setting %
    $G_\textup{max} = \frac{G G_\text{V}}{\widehat{\mu}}$ from \cref{eq:30-dec30}, we have:
    \begin{equation}
        \label{eq:30-jan2}
        \text{(I)} = \frac{1}{K} \sum_{k=0}^{K-1}  \big \langle \bm{g}_{k,T}, \bm{w}_k - \bm{w}^{*} \big \rangle \leq \frac{\log m}{K \alpha} + \frac{\alpha G_\textup{max}^2}{2} = \frac{\log m}{K \alpha} + \frac{\alpha G^2 G_\text{V}^2}{2 \widehat{\mu}^2}. %
    \end{equation}
    We will bound (II) now. To that end, note that:
    \begin{equation}
        \langle \nabla F(\bm{w}_k) - \bm{g}_{k,T}, \bm{w}_k - \bm{w}^{*}\rangle \leq \|\nabla F(\bm{w}_k) - \bm{g}_{k,T}\|_\infty \|\bm{w}_k - \bm{w}^{*}\|_1 \leq 2 \|\nabla F(\bm{w}_k) - \bm{g}_{k,T}\|_\infty,
    \end{equation}
    where the last step follows because $\bm{w}_k$ and $\bm{w}^{*}$ lie on the simplex. Thus:
    \begin{flalign}
        \nonumber
        \text{(II)} = \frac{1}{K} \sum_{k=0}^{K-1} \Big \langle \nabla F(\bm{w}_k) - \bm{g}_{k,T}, \bm{w}_k - \bm{w}^{*}\Big \rangle & \leq \frac{2}{K} \sum_{k=0}^{K-1} \Big\|\nabla F(\bm{w}_k) - \bm{g}_{k,T}\Big\|_\infty 
        \\
        \label{eq:33-jan2}
        & = \frac{2}{K} \sum_{k=0}^{K-1} \max_{j \in [m]} \Big|\nabla F(\bm{w}_k)^{(j)} - {g}_{k,T}^{(j)}\Big|.
    \end{flalign}
    Let us denote $\bm{\theta}^{\ast}(\bm{w}_k)$ by $\bm{\theta}^{\ast}_k$ for brevity. Using \Cref{lem2}, we have:
    \begin{equation}
        \label{eq:34-jan1}
        \nabla F(\bm{w}_k)^{(j)} = - \Bigg \langle \nabla \mathcal{L}_\text{V} \big(\bm{\theta}^{\ast}_k\big),  \nabla_{\bm{\theta}}^2 \mathcal{L}_\text{T} \big(\bm{\theta}^{\ast}_k, \bm{w}_k\big)^{-1} \nabla \ell_j\big(\bm{\theta}^{\ast}_k\big) \Bigg \rangle.
    \end{equation}
    Again, for brevity let $\bm{H}_k^{\ast} := \nabla_{\bm{\theta}}^2 \mathcal{L}_\text{T} \big(\bm{\theta}^{\ast}_k, \bm{w}_k\big)$. 
    We can express the Hessian inverse $(\bm{H}_k^{\ast})^{-1}$ using the Neumann series (see for example, \cite{lorraine2020optimizing, agarwal2017second}):
    \begin{equation}
        (\bm{H}_k^{\ast})^{-1} = \eta \sum_{i=0}^\infty \big(\textbf{\textup{I}} - \eta \bm{H}_{k}^{\ast}\big)^i,
    \end{equation}
    {with $\eta < \frac{1}{L}$ (note that this is satisfied because we choose $\eta \leq \frac{\mu}{L^2}$), where recall that $\bm{H}_k^{\ast} \preceq L \textbf{I}$.} Plugging this into \cref{eq:34-jan1} gives us:
    \begin{equation}
        \label{eq:42-jan8}
        \nabla F(\bm{w}_k)^{(j)} = - \eta \sum_{i=0}^\infty \Bigg \langle \nabla \mathcal{L}_\text{V} \big(\bm{\theta}^{\ast}_k\big),  \big(\textbf{\textup{I}} - \eta \bm{H}_{k}^{\ast}\big)^i \nabla \ell_j\big(\bm{\theta}^{\ast}_k\big) \Bigg \rangle.
    \end{equation}
    From \cref{eq:26-jan1}, recall that:
    \begin{equation}
    \label{eq:43-jan8}
    {g}^{(j)}_{k,T} = -\eta \sum_{i=0}^{T-1} \Bigg\langle \nabla \mathcal{L}_\text{V}(\bm{\theta}_{k,T}), \big(\textbf{\textup{I}} - \eta \bm{H}_{k}\big)^{i} \nabla \ell_j\big(\bm{\theta}_{k,T-1-i}\big)\Bigg \rangle.
    \end{equation}
So we have:
\begin{flalign}
    \nonumber
    & \Big|\nabla F(\bm{w}_k)^{(j)} - {g}^{(j)}_{k,T} \Big| \leq 
    \\
    \nonumber
    & \underbrace{\eta \sum_{i=0}^{T-1} \Bigg|\Bigg\{\Bigg\langle \nabla \mathcal{L}_\text{V}(\bm{\theta}_{k,T}), \big(\textbf{\textup{I}} - \eta \bm{H}_{k}\big)^{i} \nabla \ell_j\big(\bm{\theta}_{k,T-1-i}\big)\Bigg \rangle - \Bigg \langle \nabla \mathcal{L}_\text{V} \big(\bm{\theta}^{\ast}_k\big),  \big(\textbf{\textup{I}} - \eta \bm{H}_{k}^{\ast}\big)^i \nabla \ell_j\big(\bm{\theta}^{\ast}_k\big) \Bigg \rangle\Bigg\} \Bigg|}_\text{(A)}
    \\
    \label{eq:38-jan2}
    & + \underbrace{\eta \sum_{i=T}^\infty \Bigg|\Bigg \langle \nabla \mathcal{L}_\text{V} \big(\bm{\theta}^{\ast}_k\big),  \big(\textbf{\textup{I}} - \eta \bm{H}_{k}^{\ast}\big)^i \nabla \ell_j\big(\bm{\theta}^{\ast}_k\big) \Bigg \rangle\Bigg|}_\text{(B)}
\end{flalign}
Using the fact that {$\mu \textbf{I} \preceq \bm{H}_k^{\ast} \preceq L \textbf{I}$,  $\big(\textbf{\textup{I}} - \eta \bm{H}_{k}^{\ast}\big)$ is PSD because $\eta \leq \frac{\mu}{L^2} \leq \frac{1}{L}$ and so $\big\|\textbf{\textup{I}} - \eta \bm{H}_{k}^{\ast}\big\|_\text{op} \leq (1 - \eta \mu)$, $\mathcal{L}_\text{V}$ is $G_\text{V}$-Lipschitz, and $\ell_j$ is $G$-Lipschitz}, we have:
\begin{equation}
    \label{eq:39-jan2}
    \text{(B)} \leq \eta \sum_{t=T}^\infty \Big\|\textbf{\textup{I}} - \eta \bm{H}_{k}^{\ast}\Big\|_\text{op}^{i} \Big\|\nabla \mathcal{L}_\text{V}(\bm{\theta}_{k}^{\ast})\Big\|_2 \Big\|\nabla \ell_j\big(\bm{\theta}_{k}^{\ast}\big)\Big\|_2 \leq 
    \eta \sum_{t=T}^\infty (1 - \eta \mu)^{i} G_\text{V} G = \frac{G_\text{V} G}{\mu} \big(1 - \eta \mu\big)^{T}.
\end{equation}
Using \Cref{lem3} %
and using the fact that $\sum_{r=0}^\infty r x^{r-1} = \frac{1}{(1-x)^2}$ for any $x \in (0,1)$, we have:
\begin{flalign}
    \nonumber
    \text{(A)} & \leq \eta \sum_{i=0}^{T-1} \Bigg\{\eta \delta G_\text{V} G i \Big(1 - \eta \min\big(\mu, \widehat{\mu}\big)\Big)^{i-1} + \Big(L G_\text{V} + L_\text{V} G \Big) \Big(1 - \frac{\eta \mu}{2}\Big)^{T-1} \big\|\bm{\theta}_{k, 0} - \bm{\theta}^{\ast}_k\big\|_2\Bigg\}
    \\
    \label{eq:40-jan2-0}
    & \leq \eta^2 \delta G_\text{V} G \underbrace{\sum_{i=0}^{\infty} i \Big(1 - \eta \min\big(\mu, \widehat{\mu}\big)\Big)^{i-1}}_{= \frac{1}{\eta^2 \min(\mu^2, \widehat{\mu}^2)}} + \eta \Big(L G_\text{V} + L_\text{V} G \Big) T \Big(1 - \frac{\eta \mu}{2}\Big)^{T-1} \big\|\bm{\theta}_{k, 0} - \bm{\theta}^{\ast}_k\big\|_2
    \\
    \label{eq:40-jan2}
    & = \frac{\delta G_\text{V} G}{\min\big(\mu^2, \widehat{\mu}^2\big)} + \eta \Big(L G_\text{V} + L_\text{V} G \Big) T \Big(1 - \frac{\eta \mu}{2}\Big)^{T-1} \big\|\bm{\theta}_{k, 0} - \bm{\theta}^{\ast}_k\big\|_2.
\end{flalign}
Note that we can go from \cref{eq:40-jan2-0} to \cref{eq:40-jan2} because $\eta < \min\Big(\frac{1}{\widehat{L}}, \frac{\mu}{L^2}\Big)$ due to which $\eta \min\big(\mu, \widehat{\mu}\big) < 1$.
\\
\\
Putting \cref{eq:39-jan2} and \cref{eq:40-jan2} into \cref{eq:38-jan2}, we get:
\small
\begin{flalign}
    \nonumber
    \Big|\nabla F(\bm{w}_k)^{(j)} - {g}^{(j)}_{k,T} \Big| & \leq \frac{\delta G_\text{V} G}{\min\big(\mu^2, \widehat{\mu}^2\big)} + \eta \Big(L G_\text{V} + L_\text{V} G \Big) T \Big(1 - \frac{\eta \mu}{2}\Big)^{T-1} \big\|\bm{\theta}_{k, 0} - \bm{\theta}^{\ast}_k\big\|_2 +  \frac{G_\text{V} G}{\mu} \big(1 - \eta \mu\big)^{T}
    \\
    \label{eq:41-jan2-iii}
    & \leq \frac{\delta G_\text{V} G}{\min\big(\mu^2, \widehat{\mu}^2\big)} + \Big(1 - \frac{\eta \mu}{2}\Big)^{T-1} \Bigg(\eta T \big(L G_\text{V} + L_\text{V} G \big) \big\|\bm{\theta}_{k, 0} - \bm{\theta}^{\ast}_k\big\|_2 + \frac{G_\text{V} G}{\mu}\Bigg).
\end{flalign}
\normalsize
Next, using \Cref{lem-bounded-dist}, we have $\big\|\bm{\theta}_{k, 0} - \bm{\theta}^{\ast}_k\big\|_2 \leq R = 2 \max\big(\big\|\bm{\theta}_{0, 0} - \bm{\theta}^{\ast}_0\big\|_2, B\big)$, for all $k \geq 0$, when $T \geq \big \lceil\frac{\log 4}{\eta \mu} \big \rceil$. Using this in \cref{eq:41-jan2-iii}, we get:
\begin{equation}
    \label{eq:41-jan2}
    \Big|\nabla F(\bm{w}_k)^{(j)} - {g}^{(j)}_{k,T} \Big| \leq \frac{\delta G_\text{V} G}{\min\big(\mu^2, \widehat{\mu}^2\big)} + \Big(1 - \frac{\eta \mu}{2}\Big)^{T-1} \Bigg(\eta T \big(L G_\text{V} + L_\text{V} G \big) R + \frac{G_\text{V} G}{\mu}\Bigg).
\end{equation}
Note that \cref{eq:41-jan2} holds for each and every $k$ and $j$. Using the above in \cref{eq:33-jan2}, we get:
\begin{equation}
    \label{eq:42-jan2}
    \text{(II)} \leq \frac{2 \delta G_\text{V} G}{\min\big(\mu^2, \widehat{\mu}^2\big)} + 2 \Big(1 - \frac{\eta \mu}{2}\Big)^{T-1} \Bigg(\eta T \big(L G_\text{V} + L_\text{V} G \big) R + \frac{G_\text{V} G}{\mu}\Bigg).
\end{equation}
Finally, using \cref{eq:30-jan2} and \cref{eq:42-jan2} in \cref{eq:27-jan2}, we get:
\small
\begin{equation}
    \label{eq:43-jan2}
    \frac{1}{K} \sum_{k=0}^{K-1} \Big(F(\bm{w}_k) - F(\bm{w}^{*})\Big) \leq
    \frac{\log m}{K \alpha} + \frac{\alpha G^2 G_\text{V}^2}{2 \widehat{\mu}^2}%
    + 2 \Big(1 - \frac{\eta \mu}{2}\Big)^{T-1} \Bigg(\eta T \big(L G_\text{V} + L_\text{V} G \big) R + \frac{G G_\text{V}}{\mu}\Bigg) 
    + {\frac{2 \delta G G_\text{V}}{\min\big(\mu^2, \widehat{\mu}^2\big)}}.
\end{equation}
\normalsize
Moreover, letting $\overline{\bm{w}}_K := \frac{1}{K} \sum_{k=0}^{K-1} \bm{w}_k$, we have using Jensen's inequality,  $F(\overline{\bm{w}}_K) - F(\bm{w}^{\ast}) \leq \frac{1}{K} \sum_{k=0}^{K-1} \big(F(\bm{w}_k) - F(\bm{w}^{*})\big)$. So from \cref{eq:43-jan2}, we also get:
\begin{flalign}
    \nonumber
    & F(\overline{\bm{w}}_K) - F^{*} \leq
    \frac{\log m}{K \alpha} + \frac{\alpha G^2 G_\text{V}^2}{2 \widehat{\mu}^2}%
    + 2 \Big(1 - \frac{\eta \mu}{2}\Big)^{T-1} \Bigg(\eta T \big(L G_\text{V} + L_\text{V} G \big) R + \frac{G G_\text{V}}{\mu}\Bigg) 
    + {\frac{2 \delta G G_\text{V}}{\min\big(\mu^2, \widehat{\mu}^2\big)}},
\end{flalign}
with $F^{*} = F(\bm{w}^{\ast})$.
If we choose $\alpha = \frac{\widehat{\mu} \sqrt{\log m}}{\sqrt{K} G G_\text{V}}$, then we will get the optimal $\frac{1}{\sqrt{K}}$ dependence on $K$. In particular, we get:
\begin{flalign}
    \nonumber
    & F(\overline{\bm{w}}_K) - F^{*} \leq 
    \underbrace{\frac{3 G G_\text{V} \sqrt{\log m}}{2 \widehat{\mu} \sqrt{K}} + 2 \Big(1 - \frac{\eta \mu}{2}\Big)^{T-1} \Bigg(\eta T \big(L G_\text{V} + L_\text{V} G \big) R + \frac{G G_\text{V}}{\mu}\Bigg)}_\text{reducible error, $\mathcal{E}$} 
    + \underbrace{\frac{2 \delta G G_\text{V}}{\min\big(\mu^2, \widehat{\mu}^2\big)}}_\text{irreducible error}.
\end{flalign}
Let us fix the total parameter update budget to $N$, i.e., $K T = N$. Under this constraint, we will determine the optimal $T$ that will minimize the RHS. In particular, note that the reducible error is
\begin{equation}
    \mathcal{E}(T) = \frac{3 G G_\text{V} \sqrt{\log m}\sqrt{T}}{2 \widehat{\mu} \sqrt{N}} + 2 \Big(1 - \frac{\eta \mu}{2}\Big)^{T-1} \Bigg(\eta T \big(L G_\text{V} + L_\text{V} G \big) R + \frac{G G_\text{V}}{\mu}\Bigg).
\end{equation}
Note that $\mathcal{E}(T)$ behaves like:
\begin{equation}
    h(T) = \frac{a\sqrt{T}}{\sqrt{N}} + b (T + b_2) e^{-cT},
\end{equation}
for some appropriate positive constants $a, b, b_2$, and $c$. Noting that $h'(T) = \frac{a}{2\sqrt{N T}} + b e^{-c T}\big(1 - c(T + b_2)\big)$, it can be verified that the optimal value of $T$, say $T^{\ast}$, that minimizes $h(\cdot)$ satisfies:
\begin{equation}
    T^{\ast} = \frac{1}{c} \log\Bigg(\frac{2b}{a}\sqrt{N T^{\ast}} \big(c (T^{\ast} + b_2) - 1\big)\Bigg).
\end{equation}
Since $T^{\ast} \leq N$ and $T^{\ast} \geq \Omega(1)$, we must have $T^{\ast} = \Theta(\log N)$. Thus, $\mathcal{E}(T)$ is also minimized by choosing $T = \Theta(\log N)$. This finishes the proof.
\end{proof}

\begin{lemma}
    \label{lem1}
    Suppose the conditions of \Cref{thm2} hold and $\|\bm{g}_{k, T}\|_\infty \leq G_\textup{max}$ for all $k$. Then:
    \begin{equation*}
        \sum_{k=0}^{K-1} \big \langle \bm{g}_{k,T}, \bm{w}_k - \bm{w}^{*} \big \rangle \leq \frac{\log m}{\alpha} + \frac{K \alpha G_\textup{max}^2}{2}.
    \end{equation*}
\end{lemma}
\begin{proof}
    This proof is almost identical to the proof of Theorem 4.2 in \cite{bubeck2015convex}.
    
    Let $\Phi(\cdot)$ denote the negative entropy function and let $\bm{z}_{k+1}$ be such that:
    $$\nabla \Phi(\bm{z}_{k+1}) = \nabla \Phi(\bm{w}_{k}) - \alpha \bm{g}_{k,T}.$$ 
    Note that 
    $$\bm{w}_{k+1} = \Pi_{\bm{\Delta}^m}(\bm{z}_{k+1}),$$ 
    where $\Pi_{\bm{\Delta}^m}$ is the projection operator onto the $m$-dimensional simplex. Also, note that the Bregman divergence associated with the negative entropy function is the KL divergence.
    \\
    \\
    Following the proof of Theorem 4.2 in \cite{bubeck2015convex}, we have:
    \begin{flalign}
        \nonumber
        \big \langle \bm{g}_{k,T}, \bm{w}_k - \bm{w}^{*} \big \rangle & = \frac{1}{\alpha} \big \langle \nabla \Phi(\bm{w}_{k}) - \nabla \Phi(\bm{z}_{k+1}), \bm{w}_k - \bm{w}^{*} \big \rangle
        \\
        \nonumber
        & = \frac{1}{\alpha} \Big(D_\text{KL}\big(\bm{w}^{*} || \bm{w}_k\big) + D_\text{KL}\big(\bm{w}_k || \bm{z}_{k+1}\big) - D_\text{KL}\big(\bm{w}^{*} || \bm{z}_{k+1}\big)\Big)
        \\
        & \leq \frac{1}{\alpha} \Big(D_\text{KL}\big(\bm{w}^{*} || \bm{w}_k\big) + D_\text{KL}\big(\bm{w}_k || \bm{z}_{k+1}\big) - D_\text{KL}\big(\bm{w}^{*} || \bm{w}_{k+1}\big) - D_\text{KL}\big(\bm{w}_{k+1} || \bm{z}_{k+1} \big) \Big).
    \end{flalign}
    Summing up, we get:
    \begin{equation}
        \label{eq:31-dec30}
        \sum_{k=0}^{K-1} \big \langle \bm{g}_{k,T}, \bm{w}_k - \bm{w}^{*} \big \rangle \leq \frac{1}{\alpha} D_\text{KL}\big(\bm{w}^{*} || \bm{w}_0\big) + \frac{1}{\alpha} \sum_{k=0}^{K-1} \Big(D_\text{KL}\big(\bm{w}_k || \bm{z}_{k+1}\big) -  D_\text{KL}\big(\bm{w}_{k+1} || \bm{z}_{k+1} \big)\Big).
    \end{equation}
    Next, using the same steps as in the proof of Theorem 4.2 of \cite{bubeck2015convex}, the fact that $\Phi(\cdot)$ is 1-strongly convex w.r.t. the $\ell_1$ norm on $\bm{\Delta}^m$, and $\|\bm{g}_{k, T}\|_\infty \leq G_\text{max}$, we have:
    \begin{equation}
        D_\text{KL}\big(\bm{w}_k || \bm{z}_{k+1}\big) -  D_\text{KL}\big(\bm{w}_{k+1} || \bm{z}_{k+1} \big) \leq \frac{\alpha^2 G_\text{max}^2}{2}.
    \end{equation}
    Plugging this into \cref{eq:31-dec30}, we get:
    \begin{equation}
        \label{eq:33-dec30}
        \sum_{k=0}^{K-1} \big \langle \bm{g}_{k,T}, \bm{w}_k - \bm{w}^{*} \big \rangle \leq \frac{1}{\alpha} D_\text{KL}\big(\bm{w}^{*} || \bm{w}_0\big) + \frac{K \alpha G_\text{max}^2}{2}.
    \end{equation}
    It can be verified that if $\bm{w}_0 = \frac{1}{m} \mathds{1}_m$, then $D_\text{KL}\big(\bm{w}^{*} || \bm{w}_0\big) \leq \log m$. Using this in \cref{eq:33-dec30} finishes the proof.
\end{proof}

\begin{lemma}
    \label{lem2}
    We have:
    \begin{equation*}
        \nabla F(\bm{w})^{(j)} = - \Bigg \langle \nabla_{\bm{\theta}} \mathcal{L}_\text{V} \big(\bm{\theta}^{\ast}(\bm{w})\big),  \nabla_{\bm{\theta}}^2 \mathcal{L}_\text{T} \big(\bm{\theta}^{\ast}(\bm{w}), \bm{w}\big)^{-1} \nabla \ell_j\big(\bm{\theta}^{\ast}(\bm{w})\big) \Bigg \rangle.
    \end{equation*}
\end{lemma}

\begin{proof}
    $\nabla F(\bm{w})$ is the hyper-gradient that can be computed using the Implicit Function Theorem (IFT) (see for example, \cite{bengio2000gradient}). It can be derived as follows.
    \begin{flalign}
        \label{eq:39-jan1}
        \nabla F(\bm{w})^{(j)} = \frac{\partial \mathcal{L}_\text{V}(\bm{\theta}^{\ast}(\bm{w}))}{\partial w^{(j)}} = \Bigg \langle \nabla_{\bm{\theta}} \mathcal{L}_\text{V} \big(\bm{\theta}^{\ast}(\bm{w})\big), \frac{\partial \bm{\theta}^{\ast}(\bm{w})}{\partial w^{(j)}} \Bigg \rangle.
    \end{flalign}
    For convenience, let us denote $\bm{\theta}^{\ast}(\bm{w})$ by just $\bm{\theta}^{\ast}$. Since $\bm{\theta}^{\ast}$ is the minimizer of $\mathcal{L}_\text{T}(\cdot, \bm{w})$, we have:
    \begin{equation}
        \sum_{p=1}^m w^{(p)} \nabla \ell_p\big(\bm{\theta}^{\ast}\big) = \vec{\bm{0}}.
    \end{equation}
    Differentiating the above w.r.t. $w^{(j)}$, we get:
    \begin{equation}
        \nabla \ell_j\big(\bm{\theta}^{\ast}\big) + \underbrace{\Bigg(\sum_{p=1}^m w^{(p)} \nabla^2 \ell_p\big(\bm{\theta}^{\ast}\big)\Bigg)}_{= \nabla_{\bm{\theta}}^2 \mathcal{L}_\text{T} (\bm{\theta}^{\ast}, \bm{w})} \frac{\partial \bm{\theta}^{\ast}}{\partial w^{(j)}} = \vec{\bm{0}} \implies \frac{\partial \bm{\theta}^{\ast}}{\partial w^{(j)}} = - \nabla_{\bm{\theta}}^2 \mathcal{L}_\text{T} (\bm{\theta}^{\ast}, \bm{w})^{-1} \nabla \ell_j\big(\bm{\theta}^{\ast}\big).
    \end{equation}
    Plugging this into \cref{eq:39-jan1} gives us the desired result.
\end{proof}

\begin{lemma}
    \label{lem3}
    Suppose the conditions of \Cref{thm2} hold. Then:
    \begin{flalign*}
        & \Bigg|\Bigg\langle \nabla \mathcal{L}_\text{V}(\bm{\theta}_{k,T}), \big(\textbf{\textup{I}} - \eta \bm{H}_{k}\big)^{i} \nabla \ell_j\big(\bm{\theta}_{k,T-1-i}\big)\Bigg \rangle - \Bigg \langle \nabla \mathcal{L}_\text{V} \big(\bm{\theta}^{\ast}_k\big),  \big(\textbf{\textup{I}} - \eta \bm{H}_{k}^{\ast}\big)^i \nabla \ell_j\big(\bm{\theta}^{\ast}_k\big) \Bigg \rangle\Bigg|
        \\
        & \leq %
        \eta \delta G_\text{V} G i \Big(1 - \eta \min\big(\mu, \widehat{\mu}\big)\Big)^{i-1} + \Big(L G_\text{V} + L_\text{V} G \Big) \Big(1 - \frac{\eta \mu}{2}\Big)^{T-1} \big\|\bm{\theta}_{k, 0} - \bm{\theta}^{\ast}_k\big\|_2.
    \end{flalign*}
\end{lemma}
\begin{proof}
    Note that:
    \begin{flalign}
        & \Bigg|\Bigg\langle \nabla \mathcal{L}_\text{V}(\bm{\theta}_{k,T}), \big(\textbf{\textup{I}} - \eta \bm{H}_{k}\big)^{i} \nabla \ell_j\big(\bm{\theta}_{k,T-1-i}\big)\Bigg \rangle - \Bigg \langle \nabla \mathcal{L}_\text{V} \big(\bm{\theta}^{\ast}_k\big),  \big(\textbf{\textup{I}} - \eta \bm{H}_{k}^{\ast}\big)^i \nabla \ell_j\big(\bm{\theta}^{\ast}_k\big) \Bigg \rangle\Bigg|
        \\
        & = \Bigg|\Bigg\langle \nabla \mathcal{L}_\text{V}(\bm{\theta}_{k,T}), \big(\textbf{\textup{I}} - \eta \bm{H}_{k}\big)^{i} \nabla \ell_j\big(\bm{\theta}_{k,T-1-i}\big)\Bigg \rangle - \Bigg \langle \nabla \mathcal{L}_\text{V} \big(\bm{\theta}_{k,T}\big),  \big(\textbf{\textup{I}} - \eta \bm{H}_{k}^{\ast}\big)^i \nabla \ell_j\big(\bm{\theta}^{\ast}_k\big) \Bigg \rangle
        \\
        \nonumber
        & \quad \quad \quad - \Bigg \langle \Big(\nabla \mathcal{L}_\text{V} \big(\bm{\theta}^{\ast}_k\big) - \nabla \mathcal{L}_\text{V} \big(\bm{\theta}_{k,T}\big)\Big),  \big(\textbf{\textup{I}} - \eta \bm{H}_{k}^{\ast}\big)^i \nabla \ell_j\big(\bm{\theta}^{\ast}_k\big) \Bigg \rangle
        \Bigg|
        \\ 
        & \leq {\Bigg|\Bigg\langle \nabla \mathcal{L}_\text{V}(\bm{\theta}_{k,T}), \big(\textbf{\textup{I}} - \eta \bm{H}_{k}\big)^{i} \nabla \ell_j\big(\bm{\theta}_{k,T-1-i}\big) - \big(\textbf{\textup{I}} - \eta \bm{H}_{k}^{\ast}\big)^i \nabla \ell_j\big(\bm{\theta}^{\ast}_k\big) \Bigg \rangle \Bigg|}
        \\ 
        \nonumber
        & \quad \quad \quad + {\Bigg| \Bigg \langle \Big(\nabla \mathcal{L}_\text{V} \big(\bm{\theta}^{\ast}_k\big) - \nabla \mathcal{L}_\text{V} \big(\bm{\theta}_{k,T}\big)\Big),  \big(\textbf{\textup{I}} - \eta \bm{H}_{k}^{\ast}\big)^i \nabla \ell_j\big(\bm{\theta}^{\ast}_k\big) \Bigg \rangle
        \Bigg|}
        \\
        \label{eq:50-jan2}
        & \leq \underbrace{\Bigg|\Bigg\langle \nabla \mathcal{L}_\text{V}(\bm{\theta}_{k,T}), \Big(\big(\textbf{\textup{I}} - \eta \bm{H}_{k}\big)^{i} - \big(\textbf{\textup{I}} - \eta \bm{H}_{k}^{\ast}\big)^i \Big) \nabla \ell_j\big(\bm{\theta}_{k,T-1-i}\big) \Bigg \rangle \Bigg|}_{(1)}
        \\
        \nonumber
        & \quad \quad \quad + \underbrace{\Bigg|\Bigg\langle \nabla \mathcal{L}_\text{V}(\bm{\theta}_{k,T}), \big(\textbf{\textup{I}} - \eta \bm{H}_{k}^{\ast}\big)^i \Big(\nabla \ell_j\big(\bm{\theta}^{\ast}_k\big) - \nabla \ell_j\big(\bm{\theta}_{k,T-1-i}\big)\Big) \Bigg \rangle \Bigg|}_{(2)}
        \\
        \nonumber
        & \quad \quad \quad + \underbrace{\Bigg| \Bigg \langle \Big(\nabla \mathcal{L}_\text{V} \big(\bm{\theta}^{\ast}_k\big) - \nabla \mathcal{L}_\text{V} \big(\bm{\theta}_{k,T}\big)\Big),  \big(\textbf{\textup{I}} - \eta \bm{H}_{k}^{\ast}\big)^i \nabla \ell_j\big(\bm{\theta}^{\ast}_k\big) \Bigg \rangle
        \Bigg|}_{(3)}.
    \end{flalign}
    {Using the fact that $\mathcal{L}_\text{V}$ is $G_\text{V}$-Lipschitz, $\big\|\textbf{\textup{I}} - \eta \bm{H}_{k}^{\ast}\big\|_\text{op} \leq (1 - \eta \mu)$ (see the explanation after \cref{eq:38-jan2}), $\ell_j$ is $L$-smooth}, we have:
    \begin{equation}
        \label{eq:50-0-jan1}
        \text{(2)} \leq G_\text{V} (1 - \eta \mu)^i L \big\|\bm{\theta}_{k,T-1-i} - \bm{\theta}^{\ast}_k\big\|_2.
    \end{equation}
    {Using standard guarantees of gradient descent on strongly convex and smooth functions (recall $\mathcal{L}_\text{T}(\bm{\theta}, \cdot)$ is $\mu$-strongly-convex and $L$-smooth)}, it can be shown that:
    \begin{equation}
        \label{eq:52-jan1}
        \big\|\bm{\theta}_{k,T-1-i} - \bm{\theta}^{\ast}_k\big\|_2^2 \leq \Big(1 - \frac{\eta \mu}{2}\Big)^{2(T-1-i)}  \big\|\bm{\theta}_{k, 0} - \bm{\theta}^{\ast}_k\big\|_2^2, 
    \end{equation}
    {when $\eta \leq \frac{\mu}{L^2}$.} Using this in \cref{eq:50-0-jan1} gives us:
    \begin{equation}
        \label{eq:53-jan2}
        \text{(2)} \leq L G_\text{V} (1 - \eta \mu)^i \Big(1 - \frac{\eta \mu}{2}\Big)^{T-1-i} \big\|\bm{\theta}_{k, 0} - \bm{\theta}^{\ast}_k\big\|_2 \leq L G_\text{V} \Big(1 - \frac{\eta \mu}{2}\Big)^{T-1} \big\|\bm{\theta}_{k, 0} - \bm{\theta}^{\ast}_k\big\|_2.
    \end{equation}
    {Similarly, using the fact that $\mathcal{L}_\text{V}$ is $L_\text{V}$-smooth, $\big\|\textbf{\textup{I}} - \eta \bm{H}_{k}^{\ast}\big\|_\text{op} \leq (1 - \eta \mu)$, $\ell_j$ is $G$-Lipschitz}, we have:
    \begin{equation}
        \label{eq:50-jan1}
        \text{(3)} \leq L_\text{V} \big\|\bm{\theta}_{k,T} - \bm{\theta}^{\ast}_k\big\|_2 (1 - \eta \mu)^i G.
    \end{equation}
    {Also, just like \cref{eq:52-jan1}, we have}:
    \begin{equation}
        \label{eq:80-jan15}
        \big\|\bm{\theta}_{k,T} - \bm{\theta}^{\ast}_k\big\|_2^2 \leq \Big(1 - \frac{\eta \mu}{2}\Big)^{2T}  \big\|\bm{\theta}_{k, 0} - \bm{\theta}^{\ast}_k\big\|_2^2, 
    \end{equation}
    {when $\eta \leq \frac{\mu}{L^2}$.} Using this in \cref{eq:50-jan1} gives us:
    \begin{equation}
        \label{eq:56-jan2}
        \text{(3)} \leq L_\text{V} G \Big(1 - \frac{\eta \mu}{2}\Big)^{T} (1 - \eta \mu)^i \big\|\bm{\theta}_{k, 0} - \bm{\theta}^{\ast}_k\big\|_2 \leq L_\text{V} G \Big(1 - \frac{\eta \mu}{2}\Big)^{T+i} \big\|\bm{\theta}_{k, 0} - \bm{\theta}^{\ast}_k\big\|_2.
    \end{equation}
    As for (1), using the fact that $\mathcal{L}_\text{V}$ is $G_\text{V}$-Lipschitz and $\ell_j$ is $G$-Lipschitz, we have:
    \begin{equation}
        \label{eq:57-jan2}
        \text{(1)} \leq G_\text{V} G \Big\|\big(\textbf{\textup{I}} - \eta \bm{H}_{k}\big)^{i} - \big(\textbf{\textup{I}} - \eta \bm{H}_{k}^{\ast}\big)^i\Big\|_\text{op}.
    \end{equation}
    Note that for any two matrices square matrices $\bm{P}$ and $\bm{Q}$, we have:
    \begin{flalign}
        & \bm{P}^i - \bm{Q}^i = \sum_{l=0}^{i-1} \bm{P}^{i-1-l} (\bm{P} - \bm{Q}) \bm{Q}^{l} \implies \|\bm{P}^i - \bm{Q}^i\|_\text{op} \leq \sum_{l=0}^{i-1} \|\bm{P}\|_\text{op}^{i-1-l} \|\bm{P} - \bm{Q}\|_\text{op} \|\bm{Q}\|_\text{op}^{l}.
    \end{flalign}
    We will use this result for $\bm{P} = \textbf{\textup{I}} - \eta \bm{H}_{k}$ and $\bm{Q} = \textbf{\textup{I}} - \eta \bm{H}_{k}^{\ast}$. In this case, note that {$\|\bm{P} - \bm{Q}\|_\text{op} \leq \eta \delta$ since $\|\bm{H}_{k} - \bm{H}_{k}^{\ast}\|_\text{op} \leq \delta$, $\|\bm{P}\|_\text{op} \leq 1 - \eta \widehat{\mu}$ (see the explanation after \cref{eq:29-dec30}), and $\|\bm{Q}\|_\text{op} \leq 1 - \eta {\mu}$ (see the explanation after \cref{eq:38-jan2})}. This gives us:
    \begin{equation}
        \label{eq:74-jan8}
        \Big\|\big(\textbf{\textup{I}} - \eta \bm{H}_{k}\big)^{i} - \big(\textbf{\textup{I}} - \eta \bm{H}_{k}^{\ast}\big)^i\Big\|_\text{op} \leq \eta \delta \sum_{l=0}^{i-1} (1 - \eta \widehat{\mu})^{i-1-l} (1 - \eta {\mu})^{l} \leq  \eta \delta i \Big(1 - \eta \min\big(\mu, \widehat{\mu}\big)\Big)^{i-1}.
    \end{equation}
    Note that $\eta < \min\Big(\frac{1}{\widehat{L}}, \frac{\mu}{L^2}\Big)$ and so, $\eta \min\big(\mu, \widehat{\mu}\big) < 1$. 
    Plugging \cref{eq:74-jan8} into \cref{eq:57-jan2}, we get:
    \begin{equation}
        \label{eq:60-jan2}
        \text{(1)} \leq  \eta \delta G_\text{V} G i \Big(1 - \eta \min\big(\mu, \widehat{\mu}\big)\Big)^{i-1}. %
    \end{equation}
    Using equations (\ref{eq:60-jan2}), (\ref{eq:53-jan2}), and (\ref{eq:56-jan2}) in \cref{eq:50-jan2}:
    \begin{flalign*}
        & \Bigg|\Bigg\langle \nabla \mathcal{L}_\text{V}(\bm{\theta}_{k,T}), \big(\textbf{\textup{I}} - \eta \bm{H}_{k}\big)^{i} \nabla \ell_j\big(\bm{\theta}_{k,T-1-i}\big)\Bigg \rangle - \Bigg \langle \nabla \mathcal{L}_\text{V} \big(\bm{\theta}^{\ast}_k\big),  \big(\textbf{\textup{I}} - \eta \bm{H}_{k}^{\ast}\big)^i \nabla \ell_j\big(\bm{\theta}^{\ast}_k\big) \Bigg \rangle\Bigg|
        \\
        & \leq %
        \eta \delta G_\text{V} G i \Big(1 - \eta \min\big(\mu, \widehat{\mu}\big)\Big)^{i-1} + \Bigg(L G_\text{V} \Big(1 - \frac{\eta \mu}{2}\Big)^{T-1} 
        + L_\text{V} G \Big(1 - \frac{\eta \mu}{2}\Big)^{T+i}\Bigg) \big\|\bm{\theta}_{k, 0} - \bm{\theta}^{\ast}_k\big\|_2
        \\
        & \leq %
        \eta \delta G_\text{V} G i \Big(1 - \eta \min\big(\mu, \widehat{\mu}\big)\Big)^{i-1} + \Big(L G_\text{V} + L_\text{V} G \Big) \Big(1 - \frac{\eta \mu}{2}\Big)^{T-1} \big\|\bm{\theta}_{k, 0} - \bm{\theta}^{\ast}_k\big\|_2.
    \end{flalign*}
    This finishes the proof.
\end{proof}

\begin{lemma}
    \label{lem-bounded-dist}
    Suppose the conditions of \Cref{thm2} hold (in particular, $T \geq \big\lceil\frac{\log 4}{\eta \mu}\big\rceil$) and {let $R = 2 \max\Big(\big\|\bm{\theta}_{0, 0} - \bm{\theta}^{\ast}_0\big\|_2, \Big(\frac{2L}{\mu} + 1\Big) D \Big)$ be as defined in \Cref{thm2}}. Then for all $k \geq 0$, we have:
    \begin{equation*}
        \big\|\bm{\theta}_{k, 0} - \bm{\theta}^{\ast}_k\big\|_2 \leq R.
    \end{equation*}
\end{lemma}

\begin{proof}
    We will prove this by induction. For this proof, we need the following important result from \Cref{max:dist}. 
    \begin{equation}
        \label{eq:87-jan25}
        \sup_{\bm{w}, \bm{w}'} \Big\|\textup{arg min}_{\bm{\theta}} \mathcal{L}_\textup{T}(\bm{\theta}, \bm{w}) - \textup{arg min}_{\bm{\theta}} \mathcal{L}_\textup{T}(\bm{\theta}, \bm{w}')\Big\|_2 \leq B :=  \Bigg(\frac{2L}{\mu} + 1\Bigg) D,
    \end{equation}
    where $\max_{i, j}\big\|\textup{arg min}_{\bm{\theta}} \ell_i(\bm{\theta}) - \textup{arg min}_{\bm{\theta}} \ell_j(\bm{\theta})\big\|_2 \leq D$. With this notation, note that 
    $$R := 2 \max\big(\big\|\bm{\theta}_{0, 0} - \bm{\theta}^{\ast}_0\big\|_2, B\big).$$
    Let us first consider the base case of $k=1$. 
    Note that $\big\|\bm{\theta}_{0, 0} - \bm{\theta}^{\ast}_0\big\|_2 \leq \frac{R}{2}$ (as per the definition of $R$). Now:
    \begin{flalign}
        \nonumber
        \big\|\bm{\theta}_{1, 0} - \bm{\theta}^{\ast}_1\big\|_2 &= \big\|\bm{\theta}_{1, 0} - \bm{\theta}^{\ast}_0 + \bm{\theta}^{\ast}_0 - \bm{\theta}^{\ast}_1\big\|_2
        \\
        \nonumber
        & \leq \big\|\bm{\theta}_{1, 0} - \bm{\theta}^{\ast}_0\big\|_2 + \underbrace{\big\|\bm{\theta}^{\ast}_0 - \bm{\theta}^{\ast}_1\big\|_2}_{\leq B} \quad \quad \quad \text{(note that $\big\|\bm{\theta}^{\ast}_0 - \bm{\theta}^{\ast}_1\big\|_2 \leq B$ using \cref{eq:87-jan25})}
        \\
        \nonumber
        & =  \big\|\bm{\theta}_{0, T} - \bm{\theta}^{\ast}_0\big\|_2 + B \quad \quad \quad \quad \quad \quad \quad \text{(recall that $\bm{\theta}_{k+1, 0} = \bm{\theta}_{k, T}$ for all $k \geq 0$)}
        \\
        \label{eq:85-jan15}
        & \leq \big\|\bm{\theta}_{0, 0} - \bm{\theta}^{\ast}_0\big\|_2 \Big(1 - \frac{\eta \mu}{2}\Big)^{T} + B,
    \end{flalign}
    where the last step follows from \cref{eq:80-jan15}. Recalling that $\big\|\bm{\theta}_{0, 0} - \bm{\theta}^{\ast}_0\big\|_2 \leq \frac{R}{2}$ and $B \leq \frac{R}{2}$ above, we get:
    \begin{equation}
        \big\|\bm{\theta}_{1, 0} - \bm{\theta}^{\ast}_1\big\|_2 \leq \frac{R}{2}\Big(\Big(1 - \frac{\eta \mu}{2}\Big)^{T} + 1\Big) \leq R.
    \end{equation}
    So the base case is true. Now assume our claim is true up to some $k > 1$, i.e., $\big\|\bm{\theta}_{k, 0} - \bm{\theta}^{\ast}_k\big\|_2 \leq R$. We will show that the claim is also true for $k+1$. Following similar steps as the ones done to get to \cref{eq:85-jan15}, we obtain:
    \begin{flalign}
        \nonumber
        \big\|\bm{\theta}_{k+1, 0} - \bm{\theta}^{\ast}_{k+1}\big\|_2 
        & 
        \leq \big\|\bm{\theta}_{k+1, 0} - \bm{\theta}^{\ast}_{k}\big\|_2 + \underbrace{\big\|\bm{\theta}^{\ast}_{k} - \bm{\theta}^{\ast}_{k+1}\big\|_2}_{\leq B \leq \frac{R}{2}} 
        \\
        \nonumber
        & \leq \big\|\bm{\theta}_{k, T} - \bm{\theta}^{\ast}_{k}\big\|_2 + \frac{R}{2}
        \\
        \nonumber
        & \leq \underbrace{\big\|\bm{\theta}_{k, 0} - \bm{\theta}^{\ast}_k\big\|_2}_{\leq R} \Big(1 - \frac{\eta \mu}{2}\Big)^{T} + \frac{R}{2}
        \\
        & \leq R \Bigg(e^{-\frac{\eta \mu T}{2}} + \frac{1}{2}\Bigg),
    \end{flalign}
    where in the last step we have used the fact that $1 - x \leq e^{-x}$ for all $x \geq 0$. Note that for $T \geq \big\lceil\frac{\log 4}{\eta \mu}\big\rceil$, we have $e^{-\frac{\eta \mu T}{2}} \leq \frac{1}{2}$ and thus, $\big\|\bm{\theta}_{k+1, 0} - \bm{\theta}^{\ast}_{k+1}\big\|_2 \leq R$. This completes the induction step and finishes the proof.
\end{proof}

\begin{lemma}
\label{max:dist}
Suppose the conditions of \Cref{thm2} hold and recall that $\max_{i, j}\big\|\textup{arg min}_{\bm{\theta}} \ell_i(\bm{\theta}) - \textup{arg min}_{\bm{\theta}} \ell_j(\bm{\theta})\big\|_2 \leq D$. Let $\bm{\theta}^{*}(\bm{w}) := \textup{arg min}_{\bm{\theta}} \mathcal{L}_\textup{T}(\bm{\theta}, \bm{w})$. Then:
\begin{equation*}
    \sup_{\bm{w}, \bm{w}'} \big\|\bm{\theta}^{*}(\bm{w}) - \bm{\theta}^{*}(\bm{w}')\big\|_2 \leq  \Bigg(\frac{2L}{\mu} + 1\Bigg) D.
\end{equation*}
\end{lemma}

\begin{proof}
    A similar result has been shown in \cite{bonnans2013perturbation}. For the sake of completeness and clarity, we derive the result we need for our setting of interest.
 
    For brevity, let $\bm{\theta}^{*}$ denote $\bm{\theta}^{*}(\bm{w}) := \text{arg min}_{\bm{\theta}} \mathcal{L}_\text{T}(\bm{\theta}, \bm{w})$. By the first-order optimality condition (and recalling $\mathcal{L}_\text{T}(\bm{\theta}, \bm{w}) = \sum_{j=1}^m w_j \ell_j(\bm{\theta})$), we have:
    \begin{equation}
        \label{eq:90-jan25}
        \sum_{j=1}^m w_j \nabla \ell_j(\bm{\theta}^{*}) = \vec{\bm{0}}.
    \end{equation}    
    Suppose $\bm{\theta}^{*}_{(i)} := \text{arg min}_{\bm{\theta}} \ell_i(\bm{\theta})$ and $\max_{i, j} \big\|\bm{\theta}^{*}_{(i)} - \bm{\theta}^{*}_{(j)}\big\|_2 \leq D$. 
    
    From \cref{eq:90-jan25}, we have for any $i \in [m]$:
    \begin{flalign}
        \sum_{j=1}^m w_j \big \langle \nabla \ell_j(\bm{\theta}^{*}), \bm{\theta}^{*} - \bm{\theta}^{*}_{(i)} \big \rangle = 0
        \implies \sum_{j=1}^m w_j \big \langle \nabla \ell_j(\bm{\theta}^{*}), \bm{\theta}^{*} - \bm{\theta}^{*}_{(j)} + \bm{\theta}^{*}_{(j)} - \bm{\theta}^{*}_{(i)} \big \rangle = 0.
    \end{flalign}
    Rearranging terms above, we get:
    \begin{equation}
        \underbrace{\sum_{j=1}^m w_j \big \langle \nabla \ell_j(\bm{\theta}^{*}), \bm{\theta}^{*} - \bm{\theta}^{*}_{(j)}\big \rangle}_\text{(1)} = \underbrace{\sum_{j=1}^m w_j \big \langle \nabla \ell_j(\bm{\theta}^{*}), \bm{\theta}^{*}_{(i)} - \bm{\theta}^{*}_{(j)} \big \rangle}_\text{(2)}.
    \end{equation}
    Recalling the fact that $\ell_j$'s are $\mu$-strongly-convex and using standard properties of strongly-convex functions, we have that:
    \begin{equation*}
        \big \langle \nabla \ell_j(\bm{\theta}^{*}), \bm{\theta}^{*} - \bm{\theta}^{*}_{(j)}\big \rangle \geq \mu \big\|\bm{\theta}^{*} - \bm{\theta}^{*}_{(j)}\big\|_2^2, \text{ } \forall j \in [m].
    \end{equation*}
    Using this, we have that:
    \begin{equation}
        \label{eq:93-jan25}
        \text{(1)} \geq \mu \sum_{j=1}^m w_j \big\|\bm{\theta}^{*} - \bm{\theta}^{*}_{(j)}\big\|_2^2.
    \end{equation}
    Next, using the fact that $\nabla \ell_j(\bm{\theta}^{*}_{(j)}) = \vec{\bm{0}}$, the $\ell_j$'s are $L$-smooth, and $\max_{i, j} \big\|\bm{\theta}^{*}_{(i)} - \bm{\theta}^{*}_{(j)}\big\|_2 \leq D$, we have using the Cauchy-Schwarz inequality:
    \begin{equation*}
        \big \langle \nabla \ell_j(\bm{\theta}^{*}), \bm{\theta}^{*}_{(i)} - \bm{\theta}^{*}_{(j)} \big \rangle \leq  \big\|\nabla \ell_j(\bm{\theta}^{*}) - \nabla \ell_j(\bm{\theta}^{*}_{(j)})\big\|_2 \big\|\bm{\theta}^{*}_{(i)} - \bm{\theta}^{*}_{(j)}\big\|_2 \leq L \big\| \bm{\theta}^{*} - \bm{\theta}^{*}_{(j)} \big\|_2 D, \text{ } \forall j \in [m].
    \end{equation*}
    Using this, we have:
    \begin{equation}
        \label{eq:94-0-jan25}
        \text{(2)} \leq L D \sum_{j=1}^m w_j \big\| \bm{\theta}^{*} - \bm{\theta}^{*}_{(j)} \big\|_2.
    \end{equation}
    Further, using the Cauchy-Schwarz inequality:
    \small
    \begin{equation*}
        \sum_{j=1}^m w_j \big\| \bm{\theta}^{*} - \bm{\theta}^{*}_{(j)} \big\|_2 = \sum_{j=1}^m w_j^{1/2}  w_j^{1/2} \big\| \bm{\theta}^{*} - \bm{\theta}^{*}_{(j)} \big\|_2 \leq \sqrt{\underbrace{\Big(\sum_{j=1}^m w_j\Big)}_{=1}\Big(\sum_{j=1}^m w_j \big\| \bm{\theta}^{*} - \bm{\theta}^{*}_{(j)} \big\|_2^2\Big)} = \sqrt{\sum_{j=1}^m w_j \big\| \bm{\theta}^{*} - \bm{\theta}^{*}_{(j)} \big\|_2^2}
    \end{equation*}
    \normalsize
    Using this in \cref{eq:94-0-jan25}, we get:
    \begin{equation}
        \label{eq:94-jan25}
        \text{(2)} \leq L D \sqrt{\sum_{j=1}^m w_j \big\| \bm{\theta}^{*} - \bm{\theta}^{*}_{(j)} \big\|_2^2}.
    \end{equation}
    Now from Equations \eqref{eq:93-jan25} and \eqref{eq:94-jan25} and recalling (1) $=$ (2), we get after a bit of simplification:
    \begin{equation}
        \label{eq:95-jan25}
        \sum_{j=1}^m w_j \big\|\bm{\theta}^{*} - \bm{\theta}^{*}_{(j)}\big\|_2^2 \leq \Big(\frac{L^2}{\mu^2}\Big)D^2.
    \end{equation}
    {
    Now since each $w_j \in [0, 1]$, there must exist at least one ${j} \in [m]$ such that $\big\|\bm{\theta}^{*} - \bm{\theta}^{*}_{(j)}\big\|_2^2 \leq \big(\frac{L^2}{\mu^2}\big)D^2$, otherwise \cref{eq:95-jan25} will be violated. Define $\mathcal{S}(\bm{w}) := \Big\{j \in [m] \text{ } \Big| \text{ } \big\|\bm{\theta}^{*} - \bm{\theta}^{*}_{(j)}\big\|_2 \leq \big(\frac{L}{\mu}\big)D\Big\}$. As per the previous discussion, $|\mathcal{S}(\bm{w})| \geq 1$ for all $\bm{w} \in \bm{\Delta}^m$. 

    Now let us consider two $\bm{w}$ and $\bm{w}'$ $\in \bm{\Delta}^m$. Let $j \in \mathcal{S}(\bm{w})$ and $j' \in \mathcal{S}(\bm{w}')$. Then:
    \begin{flalign}
        \nonumber
        \big\|\bm{\theta}^{*}(\bm{w}) - \bm{\theta}^{*}(\bm{w}')\big\|_2 & = \big\|\bm{\theta}^{*}(\bm{w}) - \bm{\theta}^{*}_{(j)} + \bm{\theta}^{*}_{(j)} - \bm{\theta}^{*}(\bm{w}')\big\|_2
        \\
        \nonumber
        & \leq \underbrace{\big\|\bm{\theta}^{*}(\bm{w}) - \bm{\theta}^{*}_{(j)}\big\|_2}_{\leq {{L D}/{\mu}}} + \big\|\bm{\theta}^{*}_{(j)} - \bm{\theta}^{*}(\bm{w}')\big\|_2
        \\
        \nonumber
        & \leq \Bigg(\frac{L}{\mu}\Bigg)D +  \big\|\bm{\theta}^{*}_{(j)} -  \bm{\theta}^{*}_{(j')} + \bm{\theta}^{*}_{(j')} - \bm{\theta}^{*}(\bm{w}')\big\|_2
        \\
        \nonumber
        & \leq \Bigg(\frac{L}{\mu}\Bigg)D +  \underbrace{\big\|\bm{\theta}^{*}_{(j)} -  \bm{\theta}^{*}_{(j')}\big\|_2}_{\leq D} + \underbrace{\big\|\bm{\theta}^{*}_{(j')} - \bm{\theta}^{*}(\bm{w}')\big\|_2}_{\leq {{L D}/{\mu}}}
        \\
        & \leq \Bigg(\frac{2L}{\mu} + 1\Bigg) D.
    \end{flalign}
    This finishes the proof.
    }
\end{proof}

\section{Detailed Version and Proof of Theorem~\ref{thm:stoc-main}}
\label{stoc-app}
We will now write the statement for the detailed version of \Cref{thm:stoc-main} and then prove it.

\begin{theorem}[\textbf{Detailed version of  \Cref{thm:stoc-main}}]
    \label{thm:stochastic-case}
    {Suppose Assumptions \ref{asmp-1}, \ref{asmp-2}, \ref{asmp-3}, and \ref{asmp-stoc} hold. Let $F(\bm{w}) := \mathcal{L}_\textup{V}\big(\bm{\theta}^{*}(\bm{w})\big)$ be convex (in $\bm{w} \in \bm{\Delta}^m$) 
    and $F^{*} = \min_{\bm{w} \in \bm{\Delta}^m} F(\bm{w})$. Let $\widetilde{G} := \sqrt{G^2 + \sigma^2}$, $\widetilde{G}_\text{V} := \sqrt{{G}_\text{V}^2 + \sigma^2}$, $\max_{i, j}\big\|\textup{arg min}_{\bm{\theta}} \text{ } \ell_i(\bm{\theta}) - \textup{arg min}_{\bm{\theta}} \text{ } \ell_j(\bm{\theta})\big\|_2 \leq D$, and $\overline{R} := 3 \max\Big(\big\|\bm{\theta}_{0, 0} - \bm{\theta}^{\ast}_0\big\|_2, \Big(\frac{2L}{\mu} + 1\Big) D\Big)$. Suppose we begin from $\bm{w}_0 = \frac{1}{m} \mathds{1}_m$, and choose $\alpha = \frac{\widehat{\mu} \sqrt{\log m}}{\sqrt{K} \widetilde{G} \widetilde{G}_\text{V}}$ and $\eta = \frac{4 \log T}{\mu T}$. Then for any $T \geq \max\Big(2, \Big\lceil \mathcal{O}\Big(\max\Big(\frac{\widehat{L}}{\mu} \log \frac{\widehat{L}}{\mu}, \frac{L^2}{\mu^2} \log \frac{L^2}{\mu^2}, \frac{\sigma^2}{\mu^2 (\overline{R})^2}\Big)\Big)\Big\rceil\Big)$, we have the following guarantee for \Cref{alg:stoc}:
    \begin{flalign}
        \mathbb{E}\Bigg[F\Bigg(\frac{1}{K} \sum_{k=0}^{K-1} \bm{w}_k\Bigg)\Bigg] - F^{*} \leq \mathcal{O}\Bigg(\underbrace{\frac{\widetilde{G} \widetilde{G}_\text{V} \sqrt{\log m}}{\widehat{\mu} \sqrt{K}} + \sigma \Bigg(\frac{L G_\text{V}}{\mu^2} + \frac{L_\text{V} \widetilde{G}}{{\mu} \widehat{\mu}}\Bigg) \frac{\sqrt{\log T}}{\sqrt{T}}}_\text{reducible error, $\mathcal{E}$} + \underbrace{\frac{\delta G_\text{V} G}{\min(\mu^2, \widehat{\mu}^2)}}_\text{irreducible error}\Bigg).
    \end{flalign}
    If we fix the total parameter update budget to $N$, i.e., $K T = N$, then the reducible error as a function of $T$ is
    \begin{equation}
        \mathcal{E}(T) = {\mathcal{O}}\Bigg(\Bigg({\frac{\widetilde{G} \widetilde{G}_\text{V} \sqrt{\log m}}{\widehat{\mu} \sqrt{N}}\Bigg) \sqrt{T} + \sigma \Bigg(\frac{L G_\text{V}}{\mu^2} + \frac{L_\text{V} \widetilde{G}}{{\mu} \widehat{\mu}}\Bigg) \frac{\sqrt{\log T}}{\sqrt{T}}}\Bigg).
    \end{equation}
    This error is minimized by choosing $T = \Theta\big(\sqrt{N \log N}\big)$. In particular, if we choose $T = \big \lceil \sqrt{N \log N} \big \rceil$, the reducible error is:
    \begin{equation}
        \mathcal{E} = {\mathcal{O}}\Bigg(\Bigg({\frac{\widetilde{G} \widetilde{G}_\text{V} \sqrt{\log m}}{\widehat{\mu}} + \frac{\sigma L G_\text{V}}{\mu^2} + \frac{\sigma L_\text{V} \widetilde{G}}{{\mu} \widehat{\mu}}\Bigg) \frac{(\log N)^{1/4}}{N^{1/4}}}\Bigg).
    \end{equation}
    }
\end{theorem}

\begin{proof}
We begin this proof by imposing the condition $\eta \leq \min\Big(\frac{1}{\widehat{L}}, \frac{\mu}{L^2}\Big)$. Also, note that $\widetilde{\bm{g}}_{k, T} = \Big[\widetilde{g}^{(1)}_{k,T}, \ldots, \widetilde{g}^{(m)}_{k,T}\Big]^\top$ is our algorithm's approximated hypergradient. 
\\
\\
Let $\bm{w}^{*} \in \text{arg min}_{\bm{w} \in \bm{\Delta}^m} F(\bm{w})$. Similar to \cref{eq:27-jan2}, we have:
\begin{multline}
    \label{eq:27-jan2-ii}
    \frac{1}{K} \sum_{k=0}^{K-1} \Big(\mathbb{E}\Big[F(\bm{w}_k)\Big] - F(\bm{w}^{*})\Big) \leq \underbrace{\frac{1}{K} \sum_{k=0}^{K-1}  \mathbb{E}\Big[\Big \langle \widetilde{\bm{g}}_{k,T}, \bm{w}_k - \bm{w}^{*} \Big \rangle\Big]}_\text{(I)} 
    \\
    + \underbrace{\frac{1}{K} \sum_{k=0}^{K-1} \mathbb{E}\Big[\Big \langle \nabla F(\bm{w}_k) - \widetilde{\bm{g}}_{k,T}, \bm{w}_k - \bm{w}^{*}\Big \rangle\Big]}_\text{(II)}.
\end{multline}
Note that for each $j \in [m]$:
\begin{flalign}
    \nonumber
    \mathbb{E}\Big[\big|\widetilde{g}^{(j)}_{k,T}\big|\Big] & = \eta \mathbb{E}\Bigg[\Bigg| \sum_{i=0}^{T-1} \Bigg\langle \widetilde{\nabla} {\mathcal{L}}_\text{V}(\bm{\theta}_{k,T}), \big(\textbf{\textup{I}} - \eta \bm{H}_{k}\big)^{T-1-i} \widetilde{\nabla} {\ell}_j\big(\bm{\theta}_{k,i}\big)\Bigg \rangle\Bigg|\Bigg]
    \\
    \nonumber
    & \leq \eta \sum_{i=0}^{T-1} \mathbb{E}\Bigg[\Bigg|\Bigg\langle \widetilde{\nabla} {\mathcal{L}}_\text{V}(\bm{\theta}_{k,T}), \big(\textbf{\textup{I}} - \eta \bm{H}_{k}\big)^{T-1-i} \widetilde{\nabla} {\ell}_j\big(\bm{\theta}_{k,i}\big)\Bigg \rangle\Bigg|\Bigg]
    \\
    \nonumber
    & \leq \eta \sum_{i=0}^{T-1} \Big\|\big(\textbf{\textup{I}} - \eta \bm{H}_{k}\big)\Big\|_\text{op}^{T-1-i} \mathbb{E}\Bigg[\Big\|\widetilde{\nabla} {\mathcal{L}}_\text{V}(\bm{\theta}_{k,T})\Big\|_2 \Big\|\widetilde{\nabla} {\ell}_j\big(\bm{\theta}_{k,i}\big)\Big\|_2\Bigg]
    \\
    \label{eq:76-jan8-ii}
    & \leq \eta \sum_{i=0}^{T-1} \Big\|\big(\textbf{\textup{I}} - \eta \bm{H}_{k}\big)\Big\|_\text{op}^{T-1-i} \sqrt{\mathbb{E}\Bigg[\Big\|\widetilde{\nabla} {\mathcal{L}}_\text{V}(\bm{\theta}_{k,T})\Big\|_2^2\Bigg] \mathbb{E}\Bigg[\Big\|\widetilde{\nabla} {\ell}_j\big(\bm{\theta}_{k,i}\big)\Big\|_2^2\Bigg]},
    \\
    \label{eq:29-dec30-ii}
    & \leq \eta \sum_{i=0}^{T-1} (1 - \eta \widehat{\mu})^{T-1-i} \widetilde{G}_\text{V} \widetilde{G},
\end{flalign}
where \cref{eq:76-jan8-ii} follows from H\"{o}lder's inequality, while \cref{eq:29-dec30-ii} follows by using the fact that {$\widehat{\mu} \textbf{\textup{I}} \preceq \bm{H}_k \preceq \widehat{L} \textbf{\textup{I}}$ and $\big(\textbf{\textup{I}} - \eta \bm{H}_{k}\big)$ is PSD as $\eta \leq \frac{1}{\widehat{L}}$ with $\big\|\textbf{\textup{I}} - \eta \bm{H}_{k}\big\|_\text{op} \leq (1 - \eta \widehat{\mu})$, and $\widetilde{\nabla} {\mathcal{L}}_\text{V}(\cdot)$ and $\widetilde{\nabla} {\ell}_j(\cdot)$ have second moments bounded by $\widetilde{G}_\text{V}$ and $\widetilde{G}$, respectively (\Cref{fact-stoc}).\footnote{Note that in \cref{eq:29-dec30-ii} we have taken expectation w.r.t. the randomness in the stochastic gradients first, while conditioning on the randomness in $\bm{\theta}_{k,i}$ and $\bm{\theta}_{k,T}$. This is what allows us to use the fact that the stochastic gradients have bounded second moments.}} Simplifying \cref{eq:29-dec30-ii}, we get:
\begin{equation}
    \label{eq:30-dec30-ii}
    \mathbb{E}\Big[\|\widetilde{\bm{g}}_{k, T}\|_\infty\Big] \leq \mathbb{E}\Big[\max_{j \in [m]} \big|\widetilde{g}^{(j)}_{k,T}\big|\Big] \leq
    \eta \sum_{i=0}^{\infty} (1 - \eta \widehat{\mu})^{i} \widetilde{G}_\text{V} \widetilde{G} = 
    \frac{\widetilde{G} \widetilde{G}_\text{V}}{\widehat{\mu}}.
\end{equation}
It is easy to extend \Cref{lem1} (by first taking expectation w.r.t. the randomness in the current round and then taking expectation w.r.t. the previous rounds) to get the following result when $\mathbb{E}\Big[\|\widetilde{\bm{g}}_{k, T}\|_\infty\Big] \leq \widetilde{G}_\text{max}$:
\begin{equation*}
    \sum_{k=0}^{K-1} \mathbb{E}\Big[\big \langle \widetilde{\bm{g}}_{k,T}, \bm{w}_k - \bm{w}^{*} \big \rangle\Big] \leq \frac{\log m}{\alpha} + \frac{K \alpha \widetilde{G}_\textup{max}^2}{2}.
\end{equation*}
Plugging in $\widetilde{G}_\text{max} = \frac{\widetilde{G} \widetilde{G}_\text{V}}{\widehat{\mu}}$ from \cref{eq:30-dec30-ii} above, we get (similar to \cref{eq:30-jan2}):
\begin{equation}
    \label{eq:30-jan2-ii}
    \text{(I)} \leq \frac{\log m}{K \alpha} + \frac{\alpha \widetilde{G}^2 \widetilde{G}_\text{V}^2}{2 \widehat{\mu}^2}.
\end{equation}
As for (II), letting $\bm{v}_k = \bm{w}_k - \bm{w}^{*}$, we have:
\begin{equation}
    \label{eq:81-jan8}
    \mathbb{E}\Big[\Big \langle \nabla F(\bm{w}_k) - \widetilde{\bm{g}}_{k,T}, \bm{v}_k\Big \rangle\Big] = \mathbb{E}_{\{0,1,\ldots,k-1\}}
    \Bigg[\mathbb{E}_k\Bigg[\sum_{j=1}^m \Big(\nabla F(\bm{w}_k)^{(j)} - \widetilde{g}^{(j)}_{k,T}\Big) {v}_k^{(j)}\Bigg]\Bigg],
\end{equation}
where $\mathbb{E}_l[\cdot]$ denotes the expectation w.r.t. the randomness in the stochastic gradients of the $l^\text{th}$ round while conditioning on the randomness in the previous rounds (so note that $\bm{v}_l$ will be fixed in the conditional expectation). 
Recall that from \cref{eq:42-jan8}, we have:
\begin{equation}
    \nabla F(\bm{w}_k)^{(j)} = - \eta \sum_{i=0}^\infty \Bigg \langle \nabla \mathcal{L}_\text{V} \big(\bm{\theta}^{\ast}_k\big),  \big(\textbf{\textup{I}} - \eta \bm{H}_{k}^{\ast}\big)^i \nabla \ell_j\big(\bm{\theta}^{\ast}_k\big) \Bigg \rangle.
\end{equation}
Note that $\bm{\theta}^{\ast}_k$ only depends on $\bm{w}_k$ and is therefore fixed while taking the conditional expectation $\mathbb{E}_k[\cdot]$. 
\\
\\
Also, similar to \cref{eq:43-jan8}, we have:
\begin{equation}
\widetilde{g}^{(j)}_{k,T} = -\eta \sum_{i=0}^{T-1} \Bigg\langle \widetilde{\nabla} {\mathcal{L}}_\text{V}(\bm{\theta}_{k,T}), \big(\textbf{\textup{I}} - \eta \bm{H}_{k}\big)^{i} \widetilde{\nabla} {\ell}_j\big(\bm{\theta}_{k,T-1-i}\big)\Bigg \rangle.
\end{equation}
Plugging these two equations into \cref{eq:81-jan8}, we get:
\small
\begin{flalign}
    \nonumber
    & \mathbb{E}_k\Big[\Big \langle \nabla F(\bm{w}_k) - \widetilde{\bm{g}}_{k,T}, \bm{v}_k\Big \rangle\Big]
    \\
    \nonumber
    & = \eta \mathbb{E}_k\Bigg[\sum_{j=1}^m \sum_{i=0}^{T-1} \Bigg(\Bigg\langle \widetilde{\nabla} {\mathcal{L}}_\text{V}(\bm{\theta}_{k,T}), \big(\textbf{\textup{I}} - \eta \bm{H}_{k}\big)^{i} \widetilde{\nabla} {\ell}_j\big(\bm{\theta}_{k,T-1-i}\big)\Bigg \rangle - \Bigg \langle \nabla \mathcal{L}_\text{V} \big(\bm{\theta}^{\ast}_k\big),  \big(\textbf{\textup{I}} - \eta \bm{H}_{k}^{\ast}\big)^i \nabla \ell_j\big(\bm{\theta}^{\ast}_k\big) \Bigg \rangle\Bigg){v}_k^{(j)}\Bigg]
    \\
    & \quad \quad - \eta \mathbb{E}_k\Bigg[\sum_{j=1}^m \sum_{i=T}^\infty \Bigg \langle \nabla \mathcal{L}_\text{V} \big(\bm{\theta}^{\ast}_k\big),  \big(\textbf{\textup{I}} - \eta \bm{H}_{k}^{\ast}\big)^i \nabla \ell_j\big(\bm{\theta}^{\ast}_k\big) \Bigg \rangle {v}_k^{(j)} \Bigg].
\end{flalign}
\normalsize
Next, taking expectation w.r.t. the randomness in the stochastic gradient $\widetilde{\nabla} {\mathcal{L}}_\text{V}(\bm{\theta}_{k,T})$, which is independent of the randomness in $\{\bm{\theta}_{k,i}\}_{i=0}^{T}$ and the stochastic gradients $\{\widetilde{\nabla} {\ell}_j\big(\bm{\theta}_{k,i}\big)\}_{i=0}^{T-1}$, we get:
\small
\begin{flalign}
    \nonumber
    & \mathbb{E}_k\Big[\Big \langle \nabla F(\bm{w}_k) - \widetilde{\bm{g}}_{k,T}, \bm{v}_k\Big \rangle\Big]
    \\
    \nonumber
    & = {\eta \mathbb{E}_k\Bigg[\sum_{j=1}^m \sum_{i=0}^{T-1} \Bigg(\Bigg\langle \nabla {\mathcal{L}}_\text{V}(\bm{\theta}_{k,T}), \big(\textbf{\textup{I}} - \eta \bm{H}_{k}\big)^{i} \widetilde{\nabla} {\ell}_j\big(\bm{\theta}_{k,T-1-i}\big)\Bigg \rangle - \Bigg \langle \nabla \mathcal{L}_\text{V} \big(\bm{\theta}^{\ast}_k\big),  \big(\textbf{\textup{I}} - \eta \bm{H}_{k}^{\ast}\big)^i \nabla \ell_j\big(\bm{\theta}^{\ast}_k\big) \Bigg \rangle\Bigg){v}_k^{(j)}\Bigg]}
    \\
    \nonumber
    & \quad \quad {-\eta \mathbb{E}_k\Bigg[\sum_{j=1}^m \sum_{i=T}^\infty \Bigg \langle \nabla \mathcal{L}_\text{V} \big(\bm{\theta}^{\ast}_k\big),  \big(\textbf{\textup{I}} - \eta \bm{H}_{k}^{\ast}\big)^i \nabla \ell_j\big(\bm{\theta}^{\ast}_k\big) \Bigg \rangle {v}_k^{(j)} \Bigg]}
    \\
    \nonumber
    & = \underbrace{\eta \mathbb{E}_k\Bigg[\sum_{j=1}^m \sum_{i=0}^{T-1} \Bigg(\Bigg\langle \nabla {\mathcal{L}}_\text{V}(\bm{\theta}_{k,T}), \big(\textbf{\textup{I}} - \eta \bm{H}_{k}\big)^{i} \widetilde{\nabla} {\ell}_j\big(\bm{\theta}_{k,T-1-i}\big)\Bigg \rangle - \Bigg\langle \nabla {\mathcal{L}}_\text{V}(\bm{\theta}_{k}^{\ast}), \big(\textbf{\textup{I}} - \eta \bm{H}_{k}\big)^{i} \widetilde{\nabla} {\ell}_j\big(\bm{\theta}_{k,T-1-i}\big)\Bigg \rangle \Bigg){v}_k^{(j)}\Bigg]}_\text{(A)}
    \\
    \nonumber
    & \quad \quad + \underbrace{\eta \mathbb{E}_k\Bigg[\sum_{j=1}^m \sum_{i=0}^{T-1} \Bigg(\Bigg\langle \nabla {\mathcal{L}}_\text{V}(\bm{\theta}_{k}^{\ast}), \big(\textbf{\textup{I}} - \eta \bm{H}_{k}\big)^{i} \widetilde{\nabla} {\ell}_j\big(\bm{\theta}_{k,T-1-i}\big)\Bigg \rangle - \Bigg \langle \nabla \mathcal{L}_\text{V} \big(\bm{\theta}^{\ast}_k\big),  \big(\textbf{\textup{I}} - \eta \bm{H}_{k}^{\ast}\big)^i \nabla \ell_j\big(\bm{\theta}^{\ast}_k\big) \Bigg \rangle\Bigg){v}_k^{(j)}\Bigg]}_\text{(B)}
    \\
    \label{eq:85-jan10}
    & \quad \quad \underbrace{-\eta \mathbb{E}_k\Bigg[\sum_{j=1}^m \sum_{i=T}^\infty \Bigg \langle \nabla \mathcal{L}_\text{V} \big(\bm{\theta}^{\ast}_k\big),  \big(\textbf{\textup{I}} - \eta \bm{H}_{k}^{\ast}\big)^i \nabla \ell_j\big(\bm{\theta}^{\ast}_k\big) \Bigg \rangle {v}_k^{(j)} \Bigg]}_\text{(C)}.
\end{flalign}
\normalsize
Again, taking expectation w.r.t. the randomness in the stochastic gradient $\widetilde{\nabla} {\mathcal{L}}_\text{V}(\bm{\theta}_{k,T-1-i})$, which is conditionally independent of the randomness in $\bm{\theta}_{k,T-1-i}$, we get:
\begin{equation}
    \text{(B)} = \eta \mathbb{E}_k\Bigg[\sum_{j=1}^m \sum_{i=0}^{T-1} \Bigg\langle \nabla {\mathcal{L}}_\text{V}(\bm{\theta}_{k}^{\ast}), \big(\textbf{\textup{I}} - \eta \bm{H}_{k}\big)^{i} \nabla {\ell}_j\big(\bm{\theta}_{k,T-1-i}\big) - \big(\textbf{\textup{I}} - \eta \bm{H}_{k}^{\ast}\big)^i \nabla \ell_j\big(\bm{\theta}^{\ast}_k\big) \Bigg \rangle{v}_k^{(j)}\Bigg].
\end{equation}
Plugging this into \cref{eq:85-jan10}, we get:
\begin{flalign}
    \nonumber
    & \mathbb{E}_k\Big[\Big \langle \nabla F(\bm{w}_k) - \widetilde{\bm{g}}_{k,T}, \bm{v}_k\Big \rangle\Big]
    \\
    \nonumber 
    & = \underbrace{\eta \mathbb{E}_k\Bigg[\sum_{j=1}^m \sum_{i=0}^{T-1} \Bigg\langle \nabla {\mathcal{L}}_\text{V}(\bm{\theta}_{k,T}) - \nabla {\mathcal{L}}_\text{V}(\bm{\theta}_{k}^{\ast}), \big(\textbf{\textup{I}} - \eta \bm{H}_{k}\big)^{i} \widetilde{\nabla} {\ell}_j\big(\bm{\theta}_{k,T-1-i}\big)\Bigg \rangle {v}_k^{(j)}\Bigg]}_\text{(A)}
    \\
    \nonumber
    & \quad \quad + \underbrace{\eta \mathbb{E}_k\Bigg[\sum_{j=1}^m \sum_{i=0}^{T-1} \Bigg\langle \nabla {\mathcal{L}}_\text{V}(\bm{\theta}_{k}^{\ast}), \big(\textbf{\textup{I}} - \eta \bm{H}_{k}\big)^{i} \nabla {\ell}_j\big(\bm{\theta}_{k,T-1-i}\big) - \big(\textbf{\textup{I}} - \eta \bm{H}_{k}^{\ast}\big)^i \nabla \ell_j\big(\bm{\theta}^{\ast}_k\big) \Bigg \rangle {v}_k^{(j)}\Bigg]}_\text{(B)}
    \\
    \label{eq:86-jan10}
    & \quad \quad \underbrace{-\eta \mathbb{E}_k\Bigg[\sum_{j=1}^m \sum_{i=T}^\infty \Bigg \langle \nabla \mathcal{L}_\text{V} \big(\bm{\theta}^{\ast}_k\big),  \big(\textbf{\textup{I}} - \eta \bm{H}_{k}^{\ast}\big)^i \nabla \ell_j\big(\bm{\theta}^{\ast}_k\big) \Bigg \rangle {v}_k^{(j)} \Bigg]}_\text{(C)}.
\end{flalign}
Let us first bound (C). Note that:
\begin{equation}
    \text{(C)} \leq \eta \sum_{i=T}^\infty \mathbb{E}_k\Bigg[\max_{j \in [m]} \Bigg| \Bigg \langle \nabla \mathcal{L}_\text{V} \big(\bm{\theta}^{\ast}_k\big),  \big(\textbf{\textup{I}} - \eta \bm{H}_{k}^{\ast}\big)^i \nabla \ell_j\big(\bm{\theta}^{\ast}_k\big) \Bigg \rangle \Bigg| \underbrace{\sum_{j \in [m]} |{v}_k^{(j)}|}_{= \|\bm{v}_k\|_1} \Bigg]
\end{equation}
Recalling that $\bm{v}_k = \bm{w}_k - \bm{w}^{*}$ and $\bm{w}_k$ \& $\bm{w}^{*}$ lie on the simplex, note that $\|\bm{v}_k\|_1 \leq 2$. Using this above, we get:
\begin{equation}
    \label{eq:90-jan10}
    \text{(C)} \leq 2 \eta \sum_{i=T}^\infty \mathbb{E}_k\Bigg[\max_{j \in [m]} \Bigg| \Bigg \langle \nabla \mathcal{L}_\text{V} \big(\bm{\theta}^{\ast}_k\big),  \big(\textbf{\textup{I}} - \eta \bm{H}_{k}^{\ast}\big)^i \nabla \ell_j\big(\bm{\theta}^{\ast}_k\big) \Bigg \rangle \Bigg|\Bigg].
\end{equation}
Next, following the steps of \cref{eq:39-jan2} to bound the RHS of the above equation, we get:
\begin{equation}
    \label{eq:90-C}
    \text{(C)} \leq \frac{2 G_\text{V} G}{\mu} \big(1 - \eta \mu\big)^{T}.
\end{equation}
Similarly, using the fact that $\|\bm{v}_k\|_1 \leq 2$, we have:
\begin{flalign}
    \text{(A)} \leq 2 \eta \sum_{i=0}^{T-1} \max_{j \in [m]} \mathbb{E}_k\Bigg[\Bigg|\Bigg\langle \nabla {\mathcal{L}}_\text{V}(\bm{\theta}_{k,T}) - \nabla {\mathcal{L}}_\text{V}(\bm{\theta}_{k}^{\ast}), \big(\textbf{\textup{I}} - \eta \bm{H}_{k}\big)^{i} \widetilde{\nabla} {\ell}_j\big(\bm{\theta}_{k,T-1-i}\big)\Bigg \rangle\Bigg|\Bigg],
\end{flalign}
and 
\begin{equation}
    \text{(B)} \leq 2 \eta \sum_{i=0}^{T-1} \max_{j \in [m]}\mathbb{E}_k\Bigg[\Bigg|\Bigg\langle \nabla {\mathcal{L}}_\text{V}(\bm{\theta}_{k}^{\ast}), \big(\textbf{\textup{I}} - \eta \bm{H}_{k}\big)^{i} \nabla {\ell}_j\big(\bm{\theta}_{k,T-1-i}\big) - \big(\textbf{\textup{I}} - \eta \bm{H}_{k}^{\ast}\big)^i \nabla \ell_j\big(\bm{\theta}^{\ast}_k\big) \Bigg \rangle\Bigg|\Bigg].
\end{equation}
Using the second bound of \Cref{lem3-ii}, we have:
\begin{flalign}
    \nonumber
    \text{(A)} & \leq 2 \eta \Bigg(L_\text{V} \widetilde{G} \Big(1 - \frac{\eta \mu}{2}\Big)^{T}  \big\|\bm{\theta}_{k, 0} - \bm{\theta}^{\ast}_k\big\|_2 + \frac{\sqrt{\eta} \sigma L_\text{V} \widetilde{G}}{\sqrt{\mu}} \Bigg) {\sum_{i=0}^{T-1} \big(1 - \eta \widehat{\mu}\big)^i}
    \\
    \nonumber
    & \leq 2 \eta \Bigg(L_\text{V} \widetilde{G} \Big(1 - \frac{\eta \mu}{2}\Big)^{T}  \big\|\bm{\theta}_{k, 0} - \bm{\theta}^{\ast}_k\big\|_2 + \frac{\sqrt{\eta} \sigma L_\text{V} \widetilde{G}}{\sqrt{\mu}} \Bigg) \underbrace{\sum_{i=0}^{\infty} \big(1 - \eta \widehat{\mu}\big)^i}_{= \frac{1}{\eta \widehat{\mu}}}
    \\
    \label{eq:94-jan10-A}
    & \leq \frac{2 L_\text{V} \widetilde{G}}{\widehat{\mu}} \Big(1 - \frac{\eta \mu}{2}\Big)^{T}  \big\|\bm{\theta}_{k, 0} - \bm{\theta}^{\ast}_k\big\|_2 + \frac{2 \sqrt{\eta} \sigma L_\text{V} \widetilde{G}}{\sqrt{\mu} \widehat{\mu}}.
\end{flalign}
Next, using the first bound of \Cref{lem3-ii}, we have:
\small
\begin{flalign}
    \nonumber
    \text{(B)} & \leq 2 \eta^2 \delta G_\text{V} G \sum_{i=0}^{T-1} i \Big(1 - \eta \min\big(\mu, \widehat{\mu}\big)\Big)^{i-1} + 2 \eta L G_\text{V} T \Big(1 - \frac{\eta \mu}{2}\Big)^{T-1} \big\|\bm{\theta}_{k, 0} - \bm{\theta}^{\ast}_k\big\|_2 + \frac{2 \sqrt{\eta^3} \sigma L G_\text{V}}{\sqrt{\mu}} \sum_{i=0}^{T-1} \big(1 - \eta \mu\big)^i
    \\
    \nonumber
    & \leq 2 \eta^2 \delta G_\text{V} G \underbrace{\sum_{i=0}^{\infty} i \Big(1 - \eta \min\big(\mu, \widehat{\mu}\big)\Big)^{i-1}}_{= \frac{1}{\eta^2 \min(\mu^2, \widehat{\mu}^2)}} + 2 \eta L G_\text{V} T \Big(1 - \frac{\eta \mu}{2}\Big)^{T-1} \big\|\bm{\theta}_{k, 0} - \bm{\theta}^{\ast}_k\big\|_2 + \frac{2 \sqrt{\eta^3} \sigma L G_\text{V}}{\sqrt{\mu}} \underbrace{\sum_{i=0}^{\infty} \big(1 - \eta \mu\big)^i}_{= \frac{1}{\eta \mu}}
    \\
    \label{eq:94-jan10-B}
    & = \frac{2 \delta G_\text{V} G}{\min(\mu^2, \widehat{\mu}^2)} + 2 \eta L G_\text{V} T \Big(1 - \frac{\eta \mu}{2}\Big)^{T-1} \big\|\bm{\theta}_{k, 0} - \bm{\theta}^{\ast}_k\big\|_2 + \frac{2 \sqrt{\eta} \sigma L G_\text{V}}{\sqrt{\mu^3}}.
\end{flalign}
\normalsize
Note that to get to the last equation we have used the fact that $\eta \min\big(\mu, \widehat{\mu}\big) < 1$ (as explained after \cref{eq:40-jan2}) and $\sum_{r=0}^\infty r x^{r-1} = \frac{1}{(1-x)^2}$ for any $x \in (0,1)$. 
\\
\\
Now plugging in equations (\ref{eq:94-jan10-A}), (\ref{eq:94-jan10-B}) and (\ref{eq:90-C}) into \cref{eq:86-jan10} while recalling $\bm{v}_k = \bm{w}_k - \bm{w}^{*}$, we get:
\small
\begin{flalign}
    \nonumber
    & \mathbb{E}_k\Bigg[\Big \langle \nabla F(\bm{w}_k) - \widetilde{\bm{g}}_{k,T}, \bm{w}_k - \bm{w}^{*}\Big \rangle\Bigg] \leq
    \\
    & 2 \Bigg(L G_\text{V} \eta T + \frac{L_\text{V} \widetilde{G}}{\widehat{\mu}}\Bigg) \Big(1 - \frac{\eta \mu}{2}\Big)^{T-1} \big\|\bm{\theta}_{k, 0} - \bm{\theta}^{\ast}_k\big\|_2 + \frac{2 G_\text{V} G}{\mu} \big(1 - \eta \mu\big)^{T} + 2 \sqrt{\eta} \sigma \Bigg(\frac{L G_\text{V}}{\sqrt{\mu^3}} + \frac{L_\text{V} \widetilde{G}}{\sqrt{\mu} \widehat{\mu}}\Bigg) + \frac{2 \delta G_\text{V} G}{\min(\mu^2, \widehat{\mu}^2)}.
\end{flalign}
\normalsize
Thus:
\begin{flalign}
    \nonumber
    & \mathbb{E}\Bigg[\Big \langle \nabla F(\bm{w}_k) - \widetilde{\bm{g}}_{k,T}, \bm{w}_k - \bm{w}^{*}\Big \rangle\Bigg] = \mathbb{E}_{\{0,\ldots,k-1\}}\Bigg[\mathbb{E}_k\Bigg[\Big \langle \nabla F(\bm{w}_k) - \widetilde{\bm{g}}_{k,T}, \bm{w}_k - \bm{w}^{*}\Big \rangle\Bigg]\Bigg] \leq
    \\
    \nonumber
    & \quad \Bigg\{2 \Bigg(L G_\text{V} \eta T + \frac{L_\text{V} \widetilde{G}}{\widehat{\mu}}\Bigg) \Big(1 - \frac{\eta \mu}{2}\Big)^{T-1} \mathbb{E}_{\{0,\ldots,k-1\}}\Big[\big\|\bm{\theta}_{k, 0} - \bm{\theta}^{\ast}_k\big\|_2\Big] + \frac{2 G_\text{V} G}{\mu} \big(1 - \eta \mu\big)^{T} 
    \\
    & \quad \quad +  2 \sqrt{\eta} \sigma \Bigg(\frac{L G_\text{V}}{\sqrt{\mu^3}} + \frac{L_\text{V} \widetilde{G}}{\sqrt{\mu} \widehat{\mu}}\Bigg) 
    + \frac{2 \delta G_\text{V} G}{\min(\mu^2, \widehat{\mu}^2)}\Bigg\}.
\end{flalign}
Using \Cref{lem-bounded-dist-stoc}, we have $\mathbb{E}_{\{0,\ldots,k-1\}}\Big[\big\|\bm{\theta}_{k, 0} - \bm{\theta}^{\ast}_k\big\|_2\Big] \leq \overline{R} = 3 \max\big(\big\|\bm{\theta}_{0, 0} - \bm{\theta}^{\ast}_0\big\|_2, B\big)$ for all $k \geq 0$, when $\frac{\log 9}{\mu T} \leq \eta \leq \frac{\mu (\overline{R})^2}{9 \sigma^2}$. $\Big($Note that now our overall condition on $\eta$ becomes $\frac{\log 9}{\mu T} \leq \eta \leq \min\Big(\frac{1}{\widehat{L}}, \frac{\mu}{L^2}, \frac{\mu (\overline{R})^2}{9 \sigma^2}\Big)$.$\Big)$ Using this in the above equation, we get:
\begin{flalign}
    \nonumber
    & \mathbb{E}\Bigg[\Big \langle \nabla F(\bm{w}_k) - \widetilde{\bm{g}}_{k,T}, \bm{w}_k - \bm{w}^{*}\Big \rangle\Bigg] \leq
    \\
    \label{eq:96-jan11-ii}
    & 2 \overline{R} \Bigg(L G_\text{V} \eta T + \frac{L_\text{V} \widetilde{G}}{\widehat{\mu}}\Bigg) \Big(1 - \frac{\eta \mu}{2}\Big)^{T-1} + \frac{2 G_\text{V} G}{\mu} \big(1 - \eta \mu\big)^{T} + 2 \sqrt{\eta} \sigma \Bigg(\frac{L G_\text{V}}{\sqrt{\mu^3}} + \frac{L_\text{V} \widetilde{G}}{\sqrt{\mu} \widehat{\mu}}\Bigg) + \frac{2 \delta G_\text{V} G}{\min(\mu^2, \widehat{\mu}^2)}.
\end{flalign}
Thus:
\begin{flalign}
    \nonumber
    & \text{(II)} = \frac{1}{K} \sum_{k=0}^{K-1} \mathbb{E}\Big[\Big \langle \nabla F(\bm{w}_k) - \widetilde{\bm{g}}_{k,T}, \bm{w}_k - \bm{w}^{*}\Big \rangle\Big] \leq 
    \\
    \label{eq:96-jan11}
    & 2 \overline{R} \Bigg(L G_\text{V} \eta T + \frac{L_\text{V} \widetilde{G}}{\widehat{\mu}}\Bigg) \Big(1 - \frac{\eta \mu}{2}\Big)^{T-1} + \frac{2 G_\text{V} G}{\mu} \big(1 - \eta \mu\big)^{T} + 2 \sqrt{\eta} \sigma \Bigg(\frac{L G_\text{V}}{\sqrt{\mu^3}} + \frac{L_\text{V} \widetilde{G}}{\sqrt{\mu} \widehat{\mu}}\Bigg) + \frac{2 \delta G_\text{V} G}{\min(\mu^2, \widehat{\mu}^2)}.
\end{flalign}
Using \cref{eq:96-jan11} and \cref{eq:30-jan2-ii} in \cref{eq:27-jan2-ii} while using the fact that $1 - x \leq e^{-x}$ for all $x \geq 0$:
\small
\begin{multline}
    \label{eq:98-jan11}
    \frac{1}{K} \sum_{k=0}^{K-1} \Big(\mathbb{E}\Big[F(\bm{w}_k)\Big] - F(\bm{w}^{*})\Big) \leq \\
    \mathcal{O}\Bigg(\frac{\log m}{K \alpha} + \frac{\alpha \widetilde{G}^2 \widetilde{G}_\text{V}^2}{\widehat{\mu}^2} + \overline{R} \Bigg(L G_\text{V} \eta T + \frac{L_\text{V} \widetilde{G}}{\widehat{\mu}}\Bigg) e^{-\frac{\eta T \mu}{2}}  + \frac{G_\text{V} G}{\mu} e^{- \eta T \mu} + \sqrt{\eta} \sigma \Bigg(\frac{L G_\text{V}}{\sqrt{\mu^3}} + \frac{L_\text{V} \widetilde{G}}{\sqrt{\mu} \widehat{\mu}}\Bigg) + \frac{\delta G_\text{V} G}{\min(\mu^2, \widehat{\mu}^2)}\Bigg).
\end{multline}
\normalsize
It can be verified that the RHS of \cref{eq:98-jan11} is minimized (order-wise w.r.t. $T$ and $K$) by setting $\alpha = \frac{\widehat{\mu} \sqrt{\log m}}{\sqrt{K} \widetilde{G} \widetilde{G}_\text{V}}$ and $\eta = \frac{4 \log T}{\mu T}$. {Here we restrict $T$ such that it satisfies the constraint $\frac{4 \log T}{\mu T} \leq \min\Big(\frac{1}{\widehat{L}}, \frac{\mu}{L^2}, \frac{\mu (\overline{R})^2}{9 \sigma^2}\Big)$; this is satisfied for $T \geq \Big\lceil \mathcal{O}\Big(\max\Big(\frac{\widehat{L}}{\mu} \log \frac{\widehat{L}}{\mu}, \frac{L^2}{\mu^2} \log \frac{L^2}{\mu^2}, \frac{\sigma^2}{\mu^2 (\overline{R})^2}\Big)\Big)\Big\rceil$. The other constraint of $\frac{4 \log T}{\mu T} \geq \frac{\log 9}{\mu T}$ due to \Cref{lem-bounded-dist-stoc} is satisfied with $T \geq 2$}. With these choices, we get:
\begin{equation}
    \label{eq:122-jan17}
    \frac{1}{K} \sum_{k=0}^{K-1} \Big(\mathbb{E}\Big[F(\bm{w}_k)\Big] - F(\bm{w}^{*})\Big) \leq \mathcal{O}\Bigg({\frac{\widetilde{G} \widetilde{G}_\text{V} \sqrt{\log m}}{\widehat{\mu} \sqrt{K}} + \sigma \Bigg(\frac{L G_\text{V}}{\mu^2} + \frac{L_\text{V} \widetilde{G}}{{\mu} \widehat{\mu}}\Bigg) \frac{\sqrt{\log T}}{\sqrt{T}}} + {\frac{\delta G_\text{V} G}{\min(\mu^2, \widehat{\mu}^2)}}\Bigg).
\end{equation}
Letting $\overline{\bm{w}}_K := \frac{1}{K} \sum_{k=0}^{K-1} \bm{w}_k$, we have using Jensen's inequality,  $F(\overline{\bm{w}}_K) - F(\bm{w}^{\ast}) \leq \frac{1}{K} \sum_{k=0}^{K-1} \big(F(\bm{w}_k) - F(\bm{w}^{*})\big)$. Using this in \cref{eq:122-jan17} and letting $F^{*} = F(\bm{w}^{*})$, we get:
\begin{equation}
    \mathbb{E}\Big[F(\overline{\bm{w}}_K)\Big] - F^{*} \leq \mathcal{O}\Bigg(\underbrace{\frac{\widetilde{G} \widetilde{G}_\text{V} \sqrt{\log m}}{\widehat{\mu} \sqrt{K}} + \sigma \Bigg(\frac{L G_\text{V}}{\mu^2} + \frac{L_\text{V} \widetilde{G}}{{\mu} \widehat{\mu}}\Bigg) \frac{\sqrt{\log T}}{\sqrt{T}}}_\text{reducible error, $\mathcal{E}$} + \underbrace{\frac{\delta G_\text{V} G}{\min(\mu^2, \widehat{\mu}^2)}}_\text{irreducible error}\Bigg).
\end{equation}
Let us fix the total parameter update budget to $N$, i.e., $K T = N$. Under this constraint, we will determine the optimal $T$ that will minimize the RHS. In particular, note that the reducible error is:
\begin{equation}
    \mathcal{E}(T) = {\mathcal{O}}\Bigg(\Bigg({\frac{\widetilde{G} \widetilde{G}_\text{V} \sqrt{\log m}}{\widehat{\mu} \sqrt{N}}\Bigg) \sqrt{T} + \sigma \Bigg(\frac{L G_\text{V}}{\mu^2} + \frac{L_\text{V} \widetilde{G}}{{\mu} \widehat{\mu}}\Bigg) \frac{\sqrt{\log T}}{\sqrt{T}}}\Bigg).
\end{equation}
It can be verified that $\mathcal{E}(T)$ is minimized by choosing $T = \Theta(\sqrt{N \log N})$. In particular, choosing $T = \sqrt{N \log N}$, the reducible error is:
\begin{equation}
    \mathcal{E} = {\mathcal{O}}\Bigg(\Bigg({\frac{\widetilde{G} \widetilde{G}_\text{V} \sqrt{\log m}}{\widehat{\mu}} + \frac{\sigma L G_\text{V}}{\mu^2} + \frac{\sigma L_\text{V} \widetilde{G}}{{\mu} \widehat{\mu}}\Bigg) \frac{(\log N)^{1/4}}{N^{1/4}}}\Bigg).
\end{equation}
\end{proof}

\begin{fact}[\textbf{Bounded second moment}]
    \label{fact-stoc}
    Suppose \Cref{asmp-1}, \Cref{asmp-2}, and \Cref{asmp-stoc} hold. Then the second moments of $\widetilde{\nabla} {\ell}_j(\cdot)$ and $\widetilde{\nabla} {\mathcal{L}}_\text{V}(\cdot)$ are bounded by $\widetilde{G} = \sqrt{G^2 + \sigma^2}$ and $\widetilde{G}_\text{V} = \sqrt{G_\text{V}^2 + \sigma^2}$, respectively.
\end{fact}
\begin{proof}
    Note that under \Cref{asmp-stoc} and \Cref{asmp-1}, we have:
    $$\mathbb{E}_\zeta\big[\|\widetilde{\nabla} {\ell}_j(\bm{\theta})\|_2^2\big] = \mathbb{E}_\zeta\big[\|\widetilde{\nabla} {\ell}_j(\bm{\theta}) - \nabla {\ell}_j(\bm{\theta})\|_2^2\big] + \|\nabla {\ell}_j(\bm{\theta})\|_2^2 \leq \sigma^2 + G^2 = \widetilde{G}^2.$$
    Similarly, under \Cref{asmp-stoc} and \Cref{asmp-2}, we have:
    $$\mathbb{E}_\zeta\big[\|\widetilde{\nabla} {\mathcal{L}}_\text{V}(\bm{\theta})\|_2^2\big] = \mathbb{E}_\zeta\big[\|\widetilde{\nabla} {\mathcal{L}}_\text{V}(\bm{\theta}) - \nabla \mathcal{L}_\text{V}(\bm{\theta})\|_2^2\big] + \|\nabla \mathcal{L}_\text{V}(\bm{\theta})\|_2^2 \leq \sigma^2 + G_\text{V}^2 = \widetilde{G}_\text{V}^2.$$
    Thus, the second moments of $\widetilde{\nabla} {\ell}_j(\cdot)$ and $\widetilde{\nabla} {\mathcal{L}}_\text{V}(\cdot)$ are bounded by $\widetilde{G}$ and $\widetilde{G}_\text{V}$, respectively.
\end{proof}

\begin{lemma}
    \label{lem3-ii}
    Suppose the conditions of \Cref{thm:stochastic-case} hold and $\eta \leq \min\Big(\frac{1}{\widehat{L}}, \frac{\mu}{L^2}\Big)$. Then for each $j \in [m]$, we have:
    \begin{flalign*}
        & \mathbb{E}_k\Bigg[\Bigg|\Bigg\langle \nabla {\mathcal{L}}_\text{V}(\bm{\theta}_{k}^{\ast}), \big(\textbf{\textup{I}} - \eta \bm{H}_{k}\big)^{i} \nabla {\ell}_j\big(\bm{\theta}_{k,T-1-i}\big) - \big(\textbf{\textup{I}} - \eta \bm{H}_{k}^{\ast}\big)^i \nabla \ell_j\big(\bm{\theta}^{\ast}_k\big) \Bigg \rangle\Bigg|\Bigg]
        \\
        & \leq \eta \delta G_\text{V} G i \Big(1 - \eta \min\big(\mu, \widehat{\mu}\big)\Big)^{i-1} + L G_\text{V} \Big(1 - \frac{\eta \mu}{2}\Big)^{T-1} \big\|\bm{\theta}_{k, 0} - \bm{\theta}^{\ast}_k\big\|_2 +  \frac{\sqrt{\eta} \sigma L G_\text{V}}{\sqrt{\mu}} \big(1 - \eta \mu\big)^i,
    \end{flalign*}
    and
    \begin{flalign*}
        & \mathbb{E}_k\Bigg[\Bigg|\Bigg\langle \nabla {\mathcal{L}}_\text{V}(\bm{\theta}_{k,T}) - \nabla {\mathcal{L}}_\text{V}(\bm{\theta}_{k}^{\ast}), \big(\textbf{\textup{I}} - \eta \bm{H}_{k}\big)^{i} \widetilde{\nabla} {\ell}_j\big(\bm{\theta}_{k,T-1-i}\big)\Bigg \rangle\Bigg|\Bigg] \\ 
        & \leq
        L_\text{V} \widetilde{G} \big(1 - \eta \widehat{\mu}\big)^i \Big(1 - \frac{\eta \mu}{2}\Big)^{T}  \big\|\bm{\theta}_{k, 0} - \bm{\theta}^{\ast}_k\big\|_2 + \frac{\sqrt{\eta} \sigma L_\text{V} \widetilde{G}}{\sqrt{\mu}} \big(1 - \eta \widehat{\mu}\big)^i.
    \end{flalign*}
\end{lemma}

\begin{proof}
Using the fact that $\mathcal{L}_\text{V}$ is $G_\text{V}$-Lipschitz, $\ell_j$ is $G$-Lipschitz and $L$-smooth, $(\textbf{\textup{I}} - \eta \bm{H}_{k}^{\ast})$ is PSD and $\big\|\textbf{\textup{I}} - \eta \bm{H}_{k}^{\ast}\big\|_\text{op} \leq 1 - \eta \mu$ (see the explanation after \cref{eq:38-jan2}), we have:
\small
\begin{flalign}
    \nonumber
    & \mathbb{E}_k\Bigg[\Bigg|\Bigg\langle \nabla {\mathcal{L}}_\text{V}(\bm{\theta}_{k}^{\ast}), \big(\textbf{\textup{I}} - \eta \bm{H}_{k}\big)^{i} \nabla {\ell}_j\big(\bm{\theta}_{k,T-1-i}\big) - \big(\textbf{\textup{I}} - \eta \bm{H}_{k}^{\ast}\big)^i \nabla \ell_j\big(\bm{\theta}^{\ast}_k\big) \Bigg \rangle\Bigg|\Bigg]
    \\
    \nonumber
    & \leq G_\text{V} \mathbb{E}_k\Bigg[\Big\|\big(\textbf{\textup{I}} - \eta \bm{H}_{k}\big)^{i} \nabla {\ell}_j\big(\bm{\theta}_{k,T-1-i}\big) - \big(\textbf{\textup{I}} - \eta \bm{H}_{k}^{\ast}\big)^i \nabla \ell_j\big(\bm{\theta}^{\ast}_k\big)\Big\|_2\Bigg]
    \\
    \nonumber
    & \leq G_\text{V} \mathbb{E}_k\Bigg[\Big\|\Big(\big(\textbf{\textup{I}} - \eta \bm{H}_{k}\big)^{i} - \big(\textbf{\textup{I}} - \eta \bm{H}_{k}^{\ast}\big)^i \Big) \nabla {\ell}_j\big(\bm{\theta}_{k,T-1-i}\big)\Big\|_2 + \Big\|\big(\textbf{\textup{I}} - \eta \bm{H}_{k}^{\ast}\big)^i \Big(\nabla {\ell}_j\big(\bm{\theta}_{k,T-1-i}\big) - \nabla \ell_j\big(\bm{\theta}^{\ast}_k\big)\Big)\Big\|_2\Bigg]
    \\
    \nonumber
    & \leq G_\text{V} \Bigg(G \Big\|\Big(\big(\textbf{\textup{I}} - \eta \bm{H}_{k}\big)^{i} - \big(\textbf{\textup{I}} - \eta \bm{H}_{k}^{\ast}\big)^i \Big)\Big\|_\text{op} + L \Big\|\big(\textbf{\textup{I}} - \eta \bm{H}_{k}^{\ast}\big)^i\Big\|_\text{op} \mathbb{E}_k\Big[\big\|\bm{\theta}_{k,T-1-i} - \bm{\theta}^{\ast}_k\big\|_2\Big]\Bigg)
    \\
    \nonumber
    & = G_\text{V} \Bigg(G \Big\|\Big(\big(\textbf{\textup{I}} - \eta \bm{H}_{k}\big)^{i} - \big(\textbf{\textup{I}} - \eta \bm{H}_{k}^{\ast}\big)^i \Big)\Big\|_\text{op} + L \Big\|\textbf{\textup{I}} - \eta \bm{H}_{k}^{\ast}\Big\|_\text{op}^i \mathbb{E}_k\Big[\big\|\bm{\theta}_{k,T-1-i} - \bm{\theta}^{\ast}_k\big\|_2\Big]\Bigg)
    \\
    \label{eq:93-jan10}
    & \leq G_\text{V} \Bigg(G \Big\|\Big(\big(\textbf{\textup{I}} - \eta \bm{H}_{k}\big)^{i} - \big(\textbf{\textup{I}} - \eta \bm{H}_{k}^{\ast}\big)^i \Big)\Big\|_\text{op} + L \big(1 - \eta \mu\big)^i \mathbb{E}_k\Big[\big\|\bm{\theta}_{k,T-1-i} - \bm{\theta}^{\ast}_k\big\|_2\Big]\Bigg).
\end{flalign}
\normalsize
Using standard guarantees of stochastic gradient descent on strongly-convex and smooth functions, we have:
\begin{equation}
    \label{eq:94-jan10}
    \mathbb{E}_k\Big[\big\|\bm{\theta}_{k,T-1-i} - \bm{\theta}^{\ast}_k\big\|_2^2\Big] \leq \Big(1 - \frac{\eta \mu}{2}\Big)^{2(T-1-i)}  \big\|\bm{\theta}_{k, 0} - \bm{\theta}^{\ast}_k\big\|_2^2 + \frac{\eta \sigma^2}{\mu},
\end{equation}
when $\eta \leq \frac{\mu}{L^2}$. Therefore,
\begin{equation}
    \label{eq:95-jan10}
    \mathbb{E}_k\Big[\big\|\bm{\theta}_{k,T-1-i} - \bm{\theta}^{\ast}_k\big\|_2\Big] \leq \sqrt{\mathbb{E}_k\Big[\big\|\bm{\theta}_{k,T-1-i} - \bm{\theta}^{\ast}_k\big\|_2^2\Big]} \leq \Big(1 - \frac{\eta \mu}{2}\Big)^{T-1-i}  \big\|\bm{\theta}_{k, 0} - \bm{\theta}^{\ast}_k\big\|_2 + \frac{\sqrt{\eta} \sigma}{\sqrt{\mu}}.
\end{equation}
Also, in \cref{eq:74-jan8} in the proof of \Cref{lem3}, we obtained: $$\Big\|\big(\textbf{\textup{I}} - \eta \bm{H}_{k}\big)^{i} - \big(\textbf{\textup{I}} - \eta \bm{H}_{k}^{\ast}\big)^i\Big\|_\text{op} \leq \eta \delta i \Big(1 - \eta \min\big(\mu, \widehat{\mu}\big)\Big)^{i-1}.$$
Using this and \cref{eq:95-jan10} in  \cref{eq:93-jan10}, we get:
\small
\begin{flalign}
    \nonumber
    & \mathbb{E}_k\Bigg[\Bigg|\Bigg\langle \nabla {\mathcal{L}}_\text{V}(\bm{\theta}_{k}^{\ast}), \big(\textbf{\textup{I}} - \eta \bm{H}_{k}\big)^{i} \nabla {\ell}_j\big(\bm{\theta}_{k,T-1-i}\big) - \big(\textbf{\textup{I}} - \eta \bm{H}_{k}^{\ast}\big)^i \nabla \ell_j\big(\bm{\theta}^{\ast}_k\big) \Bigg \rangle\Bigg|\Bigg]
    \\
    \nonumber
    & \leq G_\text{V} \Bigg(G \eta \delta i \Big(1 - \eta \min\big(\mu, \widehat{\mu}\big)\Big)^{i-1} + L \big(1 - \eta \mu\big)^i \Big(1 - \frac{\eta \mu}{2}\Big)^{T-1-i}  \big\|\bm{\theta}_{k, 0} - \bm{\theta}^{\ast}_k\big\|_2 + L \big(1 - \eta \mu\big)^i \Big(\frac{\sqrt{\eta} \sigma}{\sqrt{\mu}}\Big)\Bigg)
    \\
    & \leq \eta \delta G_\text{V} G i \Big(1 - \eta \min\big(\mu, \widehat{\mu}\big)\Big)^{i-1} + L G_\text{V} \Big(1 - \frac{\eta \mu}{2}\Big)^{T-1} \big\|\bm{\theta}_{k, 0} - \bm{\theta}^{\ast}_k\big\|_2 +  \frac{\sqrt{\eta} \sigma L G_\text{V}}{\sqrt{\mu}} \big(1 - \eta \mu\big)^i.
\end{flalign}
\normalsize
This finishes the proof of the first bound.
\\
\\
For the second bound, we begin by using H\"{o}lder's inequality to get:
\begin{flalign}
    \nonumber
    & \mathbb{E}_k\Bigg[\Bigg|\Bigg\langle \nabla {\mathcal{L}}_\text{V}(\bm{\theta}_{k,T}) - \nabla {\mathcal{L}}_\text{V}(\bm{\theta}_{k}^{\ast}), \big(\textbf{\textup{I}} - \eta \bm{H}_{k}\big)^{i} \widetilde{\nabla} {\ell}_j\big(\bm{\theta}_{k,T-1-i}\big)\Bigg \rangle\Bigg|\Bigg]
    \\
    \nonumber
    & \leq \mathbb{E}_k\Bigg[\Big\|\nabla {\mathcal{L}}_\text{V}(\bm{\theta}_{k,T}) - \nabla {\mathcal{L}}_\text{V}(\bm{\theta}_{k}^{\ast})\Big\|_2 \Big\| \big(\textbf{\textup{I}} - \eta \bm{H}_{k}\big)^{i} \widetilde{\nabla} {\ell}_j\big(\bm{\theta}_{k,T-1-i}\big)\Big\|_2\Bigg]
    \\
    \label{eq:97-jan10}
    & \leq \sqrt{\mathbb{E}_k\Bigg[\Big\|\nabla {\mathcal{L}}_\text{V}(\bm{\theta}_{k,T}) - \nabla {\mathcal{L}}_\text{V}(\bm{\theta}_{k}^{\ast})\Big\|_2^2\Bigg]} \sqrt{\mathbb{E}_k\Bigg[\Big\| \big(\textbf{\textup{I}} - \eta \bm{H}_{k}\big)^{i} \widetilde{\nabla} {\ell}_j\big(\bm{\theta}_{k,T-1-i}\big)\Big\|_2^2\Bigg]}
\end{flalign}
Next, using the fact that $\mathcal{L}_\text{V}$ is $L_\text{V}$-smooth, $(\textbf{\textup{I}} - \eta \bm{H}_{k})$ is PSD and $\big\|\textbf{\textup{I}} - \eta \bm{H}_{k}\big\|_\text{op} \leq 1 - \eta \widehat{\mu}$ (see the explanation after \cref{eq:29-dec30}), and the second moment of $\widetilde{\nabla} {\ell}_j(\cdot)$ is bounded by $\widetilde{G}$ (\Cref{fact-stoc}) in \cref{eq:97-jan10}, we get:
\begin{flalign}
    \nonumber
    & \mathbb{E}_k\Bigg[\Bigg|\Bigg\langle \nabla {\mathcal{L}}_\text{V}(\bm{\theta}_{k,T}) - \nabla {\mathcal{L}}_\text{V}(\bm{\theta}_{k}^{\ast}), \big(\textbf{\textup{I}} - \eta \bm{H}_{k}\big)^{i} \widetilde{\nabla} {\ell}_j\big(\bm{\theta}_{k,T-1-i}\big)\Bigg \rangle\Bigg|\Bigg]
    \\
    \nonumber
    & \leq L_\text{V} \sqrt{\mathbb{E}_k\Big[\big\|\bm{\theta}_{k,T} - \bm{\theta}_{k}^{\ast}\big\|_2^2\Big]} \sqrt{\Big\|\big(\textbf{\textup{I}} - \eta \bm{H}_{k}\big)^{2i}\Big\|_\text{op}\mathbb{E}_k\Big[\Big\|\widetilde{\nabla} {\ell}_j\big(\bm{\theta}_{k,T-1-i}\big)\Big\|_2^2\Big]}
    \\
    \label{eq:98-jan10}
    & \leq L_\text{V} \widetilde{G} \big(1 - \eta \widehat{\mu}\big)^i \sqrt{\mathbb{E}_k\Big[\big\|\bm{\theta}_{k,T} - \bm{\theta}_{k}^{\ast}\big\|_2^2\Big]}.
\end{flalign}
Just like \cref{eq:94-jan10}, we have:
\begin{equation}
    \label{eq:99-jan10}
    \mathbb{E}_k\Big[\big\|\bm{\theta}_{k,T} - \bm{\theta}^{\ast}_k\big\|_2^2\Big] \leq \Big(1 - \frac{\eta \mu}{2}\Big)^{2T}  \big\|\bm{\theta}_{k, 0} - \bm{\theta}^{\ast}_k\big\|_2^2 + \frac{\eta \sigma^2}{\mu}.
\end{equation}
Using this in \cref{eq:98-jan10} and simplifying a bit, we get:
\begin{flalign}
    \nonumber
    & \mathbb{E}_k\Bigg[\Bigg|\Bigg\langle \nabla {\mathcal{L}}_\text{V}(\bm{\theta}_{k,T}) - \nabla {\mathcal{L}}_\text{V}(\bm{\theta}_{k}^{\ast}), \big(\textbf{\textup{I}} - \eta \bm{H}_{k}\big)^{i} \widetilde{\nabla} {\ell}_j\big(\bm{\theta}_{k,T-1-i}\big)\Bigg \rangle\Bigg|\Bigg] 
    \\
    & \leq L_\text{V} \widetilde{G} \big(1 - \eta \widehat{\mu}\big)^i \Big(1 - \frac{\eta \mu}{2}\Big)^{T}  \big\|\bm{\theta}_{k, 0} - \bm{\theta}^{\ast}_k\big\|_2 + \frac{\sqrt{\eta} \sigma L_\text{V} \widetilde{G}}{\sqrt{\mu}} \big(1 - \eta \widehat{\mu}\big)^i.
\end{flalign}
This finishes the proof of the second bound.
\end{proof}

\begin{lemma}
    \label{lem-bounded-dist-stoc}
    Suppose the conditions of \Cref{thm:stochastic-case} hold and {let $\overline{R} = 3 \max\Big(\big\|\bm{\theta}_{0, 0} - \bm{\theta}^{\ast}_0\big\|_2, \Big(\frac{2L}{\mu} + 1\Big) D\Big)$ be as defined in \Cref{thm:stochastic-case}}. In addition, let $\eta$ be chosen so that $\frac{\log 9}{\mu T} \leq \eta \leq \frac{\mu (\overline{R})^2}{9 \sigma^2}$. Then for all $k \geq 0$, we have:
    \begin{equation*}
        \mathbb{E}_{\{0,\ldots,k-1\}}\Big[\big\|\bm{\theta}_{k, 0} - \bm{\theta}^{\ast}_k\big\|_2\Big] \leq \overline{R}.
    \end{equation*}
\end{lemma}

\begin{proof}
The proof idea is similar to \Cref{lem-bounded-dist} and we will prove this result by induction. Before starting the proof, recall the following important result from \Cref{max:dist}.
\begin{equation}
    \label{eq:136-jan25}
    \sup_{\bm{w}, \bm{w}'} \Big\|\textup{arg min}_{\bm{\theta}} \mathcal{L}_\textup{T}(\bm{\theta}, \bm{w}) - \textup{arg min}_{\bm{\theta}} \mathcal{L}_\textup{T}(\bm{\theta}, \bm{w}')\Big\|_2 \leq B :=  \Bigg(\frac{2L}{\mu} + 1\Bigg) D,
\end{equation}
where $\max_{i, j}\big\|\textup{arg min}_{\bm{\theta}} \ell_i(\bm{\theta}) - \textup{arg min}_{\bm{\theta}} \ell_j(\bm{\theta})\big\|_2 \leq D$. With this notation, note that 
$$\overline{R} := 3 \max\big(\big\|\bm{\theta}_{0, 0} - \bm{\theta}^{\ast}_0\big\|_2, B\big).$$
Let us first consider the base case of $k=1$. Note that $\big\|\bm{\theta}_{0, 0} - \bm{\theta}^{\ast}_0\big\|_2 \leq \frac{\overline{R}}{3}$ (as per the definition of $\overline{R}$). Now:
\begin{flalign}
    \nonumber
    \mathbb{E}_{0}\Big[\big\|\bm{\theta}_{1, 0} - \bm{\theta}^{\ast}_1\big\|_2\Big] &= \mathbb{E}_{0}\Big[\big\|\bm{\theta}_{1, 0} - \bm{\theta}^{\ast}_0 + \bm{\theta}^{\ast}_0 - \bm{\theta}^{\ast}_1\big\|_2\Big]
    \\
    \nonumber
    & \leq \mathbb{E}_{0}\Big[\big\|\bm{\theta}_{1, 0} - \bm{\theta}^{\ast}_0\big\|_2\Big] + \mathbb{E}_{0}\Big[\underbrace{\big\|\bm{\theta}^{\ast}_0 - \bm{\theta}^{\ast}_1\big\|_2}_{\leq B}\Big] \quad \text{(note that $\big\|\bm{\theta}^{\ast}_0 - \bm{\theta}^{\ast}_1\big\|_2 \leq B$ using eq. \eqref{eq:136-jan25})}
    \\
    \nonumber
    & =  \mathbb{E}_{0}\Big[\big\|\bm{\theta}_{0, T} - \bm{\theta}^{\ast}_0\big\|_2\Big] + B \quad \quad \quad \quad \quad \quad \quad \quad \text{(recall that $\bm{\theta}_{k+1, 0} = \bm{\theta}_{k, T}$ for all $k \geq 0$)}
    \\
    \nonumber
    & \leq  \sqrt{\mathbb{E}_{0}\Big[\big\|\bm{\theta}_{0, T} - \bm{\theta}^{\ast}_0\big\|_2^2\Big]} + B
    \\
    \label{eq:85-jan15-iii}
    & \leq \sqrt{\big\|\bm{\theta}_{0, 0} - \bm{\theta}^{\ast}_0\big\|_2^2 \Big(1 - \frac{\eta \mu}{2}\Big)^{2T} + \frac{\eta \sigma^2}{\mu}} + B,
\end{flalign}
where the last step follows from \cref{eq:99-jan10}. Thus:
\begin{equation}
    \label{eq:123-jan17}
    \mathbb{E}_{0}\Big[\big\|\bm{\theta}_{1, 0} - \bm{\theta}^{\ast}_1\big\|_2\Big] \leq \big\|\bm{\theta}_{0, 0} - \bm{\theta}^{\ast}_0\big\|_2 \Big(1 - \frac{\eta \mu}{2}\Big)^{T} + \frac{\sqrt{\eta} \sigma}{\sqrt{\mu}} + B.
\end{equation}
Note that $\big\|\bm{\theta}_{0, 0} - \bm{\theta}^{\ast}_0\big\|_2 \leq \frac{\overline{R}}{3}$ and $B \leq \frac{\overline{R}}{3}$. Also when $\eta \leq \frac{\mu (\overline{R})^2}{9 \sigma^2}$, we have $\frac{\sqrt{\eta} \sigma}{\sqrt{\mu}} \leq \frac{\overline{R}}{3}$. Using all of this above, we get:
\begin{equation}
    \mathbb{E}_{0}\Big[\big\|\bm{\theta}_{1, 0} - \bm{\theta}^{\ast}_1\big\|_2\Big] \leq {\overline{R}}.
\end{equation}
So the base case is true. Now assume our claim is true up to some $k > 1$, i.e., $\mathbb{E}_{\{0,\ldots,k-1\}}\Big[\big\|\bm{\theta}_{k, 0} - \bm{\theta}^{\ast}_k\big\|_2\Big] \leq \overline{R}$. We will show that the claim is also true for $k+1$. Following similar steps as the ones to get to \cref{eq:85-jan15-iii}, we obtain:
\begin{flalign}
    \nonumber
    \mathbb{E}_{\{0,\ldots,k\}}\Big[\big\|\bm{\theta}_{k+1, 0} - \bm{\theta}^{\ast}_{k+1}\big\|_2\Big] 
    & 
    \leq \mathbb{E}_{\{0,\ldots,k\}}\Big[\big\|\bm{\theta}_{k+1, 0} - \bm{\theta}^{\ast}_{k}\big\|_2\Big] + \mathbb{E}_{\{0,\ldots,k\}}\Big[\underbrace{\big\|\bm{\theta}^{\ast}_{k} - \bm{\theta}^{\ast}_{k+1}\big\|_2}_{\leq B \leq \frac{\overline{R}}{3}}\Big]
    \\
    \nonumber
    & \leq \mathbb{E}_{\{0,\ldots,k\}}\Big[\big\|\bm{\theta}_{k, T} - \bm{\theta}^{\ast}_{k}\big\|_2\Big] + \frac{\overline{R}}{3}
    \\
    \nonumber
    & \leq \mathbb{E}_{\{0,\ldots,k\}}\Bigg[{\big\|\bm{\theta}_{k, 0} - \bm{\theta}^{\ast}_k\big\|_2} \Big(1 - \frac{\eta \mu}{2}\Big)^{T} + \frac{\sqrt{\eta} \sigma}{\sqrt{\mu}}\Bigg] + \frac{\overline{R}}{3} \quad \text{(similar to eq. \eqref{eq:123-jan17})}
    \\
    \nonumber
    & \leq \mathbb{E}_{\{0,\ldots,k\}}\Big[{\big\|\bm{\theta}_{k, 0} - \bm{\theta}^{\ast}_k\big\|_2}\Big] \Big(1 - \frac{\eta \mu}{2}\Big)^{T} + \underbrace{\frac{\sqrt{\eta} \sigma}{\sqrt{\mu}}}_{\leq \frac{\overline{R}}{3}} +  \frac{\overline{R}}{3}
    \\
    \nonumber
    & = \underbrace{\mathbb{E}_{\{0,\ldots,k-1\}}\Big[{\big\|\bm{\theta}_{k, 0} - \bm{\theta}^{\ast}_k\big\|_2}\Big]}_{\leq \overline{R}} \Big(1 - \frac{\eta \mu}{2}\Big)^{T} + \frac{2 \overline{R}}{3},
\end{flalign}
where the last step follows because $\big\|\bm{\theta}_{k, 0} - \bm{\theta}^{\ast}_k\big\|_2$ does not depend on the randomness in the $k^\text{th}$ round. Next, using the fact that $\big(1 - \frac{\eta \mu}{2}\big)^{T} \leq e^{-\frac{\eta \mu T}{2}}$ above, we get:
\begin{equation}
    \mathbb{E}_{\{0,\ldots,k\}}\Big[\big\|\bm{\theta}_{k+1, 0} - \bm{\theta}^{\ast}_{k+1}\big\|_2\Big] \leq \overline{R} \Bigg(e^{-\frac{\eta \mu T}{2}} + \frac{2}{3}\Bigg).
\end{equation}
Note that for $\eta \geq \frac{\log 9}{\mu T}$, we have $e^{-\frac{\eta \mu T}{2}} \leq \frac{1}{3}$ and thus, $\mathbb{E}_{\{0,\ldots,k\}}\Big[\big\|\bm{\theta}_{k+1, 0} - \bm{\theta}^{\ast}_{k+1}\big\|_2\Big] \leq \overline{R}$. This completes the induction step and finishes the proof.
\end{proof}

\clearpage

\section{A Practically-Usable Hessian Approximation}
\label{app:approx}
Here we discuss an approximate Hessian $\bm{H}_k$ that can be used in practice in Algorithms \ref{alg:practical} and \ref{alg:stoc} (as well as related algorithms). We propose to use the isotropic approximation $\bm{H}_k = \gamma_k \textbf{I}$, where $\gamma_k > 0$ is appropriately chosen. Note that with this choice, the update rule of $\bm{u}_{k,t}^{(j)}$ becomes:
\begin{equation}
    \bm{u}_{k, t+1}^{(j)} = \big(1 - \eta \gamma_k\big) \bm{u}_{k,t}^{(j)}  + \nabla \ell_j\big(\bm{\theta}_{k,t}\big).
\end{equation}
Observe that this is essentially a \textit{momentum-like update} for each domain. Notably, this gives us an approximation for the hypergradient using a simple running average that can be efficiently computed without any matrix-vector products. 

The simplest version of the above scheme is using a constant $\gamma_k$, i.e., setting $\gamma_k = \gamma$, for all $k$.

\section{Remaining Details About Experiments in Section~\ref{sec:experiment}} \label{sec:appendix_experiments}
\textbf{Model Architecture.} We train a lightweight convolutional network (CNN) which consists of two $3\times 3$ convolution layers with ReLU activations and max pooling (channel width $c$ followed by $2c$), followed by a two-layer MLP head ($128$ hidden units) and a $10$-way linear classifier. 
\\
\\
We list the tuned hyperparameters for \Cref{alg:practical} in Table~\ref{tab:sweep-config}. 

\begin{table}[!htb]
\caption{Hyperparameters for \Cref{alg:practical}.}
\vspace{0.1 cm}
\centering
\small
\setlength{\tabcolsep}{8pt}
\renewcommand{\arraystretch}{1.15}
\begin{tabular}{ll}
\hline
\textbf{Category} & \textbf{Value} \\
\hline
\hline
Inner optimizer &
SGD step size $\eta = 0.05$ \\
Outer (weights) update &
Mirror descent step size $\alpha = 0.5$ \\
Hessian approximation parameter &
$\gamma = 0.01$ (recall $\bm{H}_k = \gamma \textbf{I}$) \\
\hline
Batch size &
$\texttt{batch\_size}=256$ \\
Train subset size &
$\texttt{train\_subset}=12000$ examples per domain \\
Validation subset size &
$\texttt{val\_subset}=3000$ examples \\
Channel width &
$\texttt{channel\_width}=32$ (small CNN backbone) \\
\hline
\end{tabular}
\label{tab:sweep-config}
\end{table}

\noindent We plot a magnified version of the plots in \Cref{fig:alphaRotatedLossAndW} as well as the corresponding validation accuracies in \Cref{fig:alphaRotatedLossAndW_full}.

\begin{figure*}[t]
\centering

\begin{subfigure}[t]{0.49\textwidth}
\centering
\includegraphics[width=\linewidth]{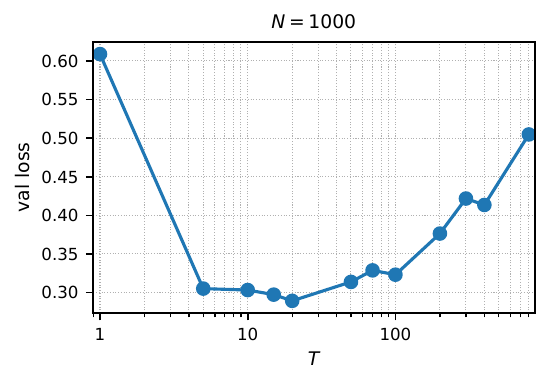}
\caption{Validation loss vs. $T$}
\label{fig:val_loss_1000_full}
\end{subfigure}\hfill
\begin{subfigure}[t]{0.49\textwidth}
\centering
\includegraphics[width=\linewidth]{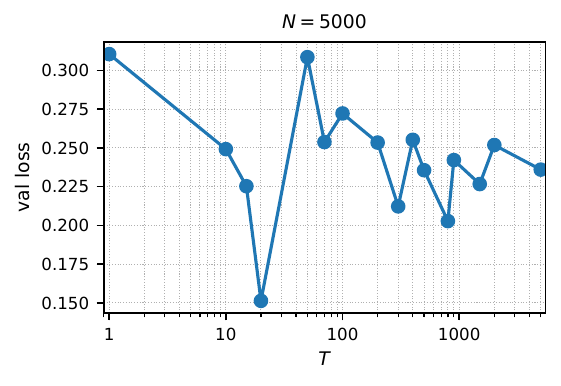}
\caption{Validation loss vs. $T$}
\label{fig:val_loss_5000_full}
\end{subfigure}

\vspace{0.6em}

\begin{subfigure}[t]{0.49\textwidth}
\centering
\includegraphics[width=\linewidth]{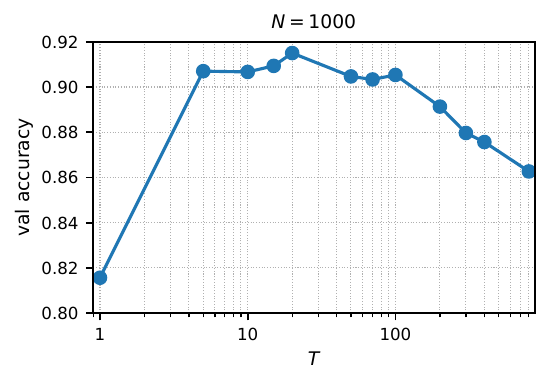}
\caption{Validation accuracy vs. $T$}
\label{fig:val_acc_1000_full}
\end{subfigure}\hfill
\begin{subfigure}[t]{0.49\textwidth}
\centering
\includegraphics[width=\linewidth]{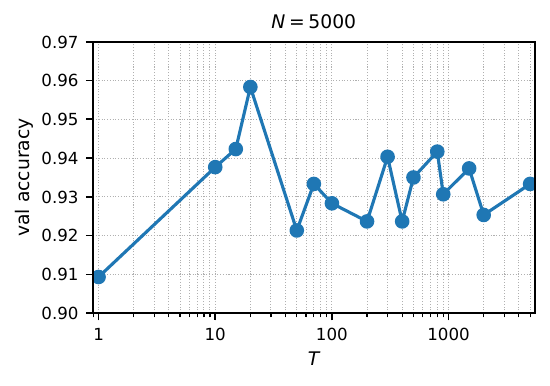}
\caption{Validation accuracy vs. $T$}
\label{fig:val_acc_5000_full}
\end{subfigure}

\vspace{0.6em}

\begin{subfigure}[t]{0.49\textwidth}
\centering
\includegraphics[width=\linewidth]{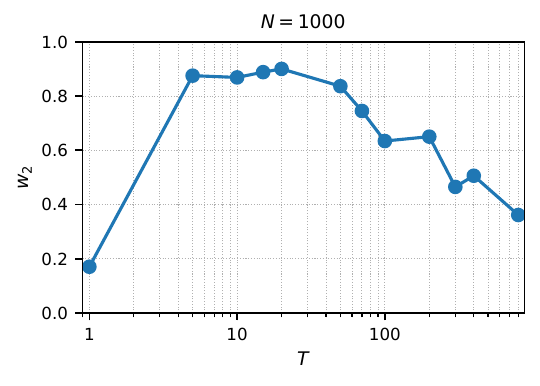}
\caption{$w_2$ vs. $T$}
\label{fig:w_2_1000}
\end{subfigure}\hfill
\begin{subfigure}[t]{0.49\textwidth}
\centering
\includegraphics[width=\linewidth]{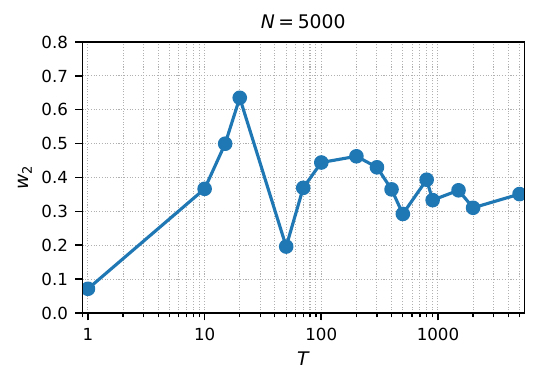}
\caption{$w_2$ vs. $T$}
\label{fig:w_2_5000}
\end{subfigure}

\caption{%
Validation loss, validation accuracy, and the weight of the second domain (most aligned with validation data) as a function of the horizon $T$, for $N = 1000$ and $5000$.
}
\label{fig:alphaRotatedLossAndW_full}
\end{figure*}
\clearpage

\end{document}